\PassOptionsToPackage{T2A}{fontenc}
\documentclass[onecolumn]{cohere}

\usepackage{placeins}
\usepackage{amsmath}
\usepackage{footmisc}
\usepackage{multirow}
\usepackage{tikz}
\usepackage{xspace}

\definecolor{tabblue}{RGB}{31,119,180}
\definecolor{thomblue1}{RGB}{136,153,244}
\definecolor{thomblue2}{RGB}{83,109,222}
\newcommand{\iconcolor}[1]{\textcolor{thomblue2}{#1}}

\usepackage{tcolorbox}
\usepackage{longtable}
\usepackage{booktabs}
\usepackage{makecell}
\usepackage{comment}
\usepackage{adjustbox}
\usepackage{marvosym}
\usepackage{xcolor} 
\usepackage{float}
\usepackage{fontawesome5}

\newtheorem{theorem}{Theorem}
\newtheorem{remark}{Remark}
\newtheorem{proof}{Proof}

\definecolor{ayadsymbol}{RGB}{76, 110, 230} 
\definecolor{secondsymbol}{RGB}{140, 165, 240}

\DeclareSymbolFont{extraup}{U}{zavm}{m}{n}
\DeclareMathSymbol{\vardiamond}{\mathalpha}{extraup}{87}

\renewcommand{\fasymbol}{\textcolor{thomblue1}{$\vardiamond$}}

\title{The Leaderboard Illusion}

\multiauthors
\author{
    name={Shivalika Singh\faa},
    affiliation={1},
}
\author{
    name={Yiyang Nan},
    affiliation={1},
}
\author{
    name={Alex Wang},
    affiliation={2},
}
\author{
    name={Daniel D'souza},
    affiliation={1},
}
\author{
    name={Sayash Kapoor},
    affiliation={3},
}
\author{
    name={Ahmet Üstün},
    affiliation={1},
}
\author{
    name={Sanmi Koyejo},
    affiliation={4},
}
\author{
    name={Yuntian Deng},
    affiliation={5},
}
\author{
    name={Shayne Longpre},
    affiliation={6},
}
\author{
    name={Noah A. Smith},
    affiliation={7,8},
}
\author{
    name={Beyza Ermis},
    affiliation={1},
}
\author{
    name={Marzieh Fadaee\fa},
    affiliation={1},
}
\author{
    name={Sara Hooker\fa},
    affiliation={1},
}

\affiliations{
    \item[1] Cohere Labs
    \item[2] Cohere
    \item[3] Princeton University
    \item[4] Stanford University
    \item[5] University of Waterloo
    \item[6] Massachusetts Institute of Technology
    \item[7] Allen Institute for Artificial Intelligence
    \item[8] University of Washington
}

\corresponding[*]{\{{shivalikasingh}, {marzieh}, {sarahooker}\}{@cohere.com}}%\subtitle{}

\newcommand{\lmarenascraped}{scraped-random-sample}
\newcommand{\lmarenahistorical}{historical-battles}
\newcommand{\lmarenaleaderboard}{leaderboard-stats}
\newcommand{\cba}{Chatbot Arena\xspace}

\newcommand{\smileemoji}{\raisebox{-1pt}{\includegraphics[width=0.9em]{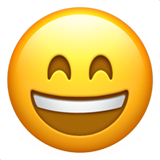}}}

\newcommand{\handwaveemoji}{\raisebox{-1pt}{\includegraphics[width=0.9em]{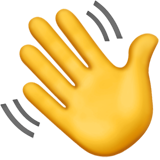}}}

\newcommand{\winkemoji}{\raisebox{-1pt}{\includegraphics[width=0.9em]{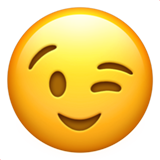}}}

\newcommand{\blushemoji}{\raisebox{-1pt}{\includegraphics[width=0.9em]{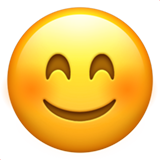}}}

\definecolor{tabblue}{RGB}{31, 119, 180}
\definecolor{tabred}{RGB}{214, 39, 40}

\abstract{

Measuring progress is fundamental to the advancement of any scientific field. As benchmarks play an increasingly central role, they also grow more susceptible to distortion.
\cba has emerged as the go-to leaderboard for ranking the most capable AI systems. Yet, in this work we identify systematic issues that have resulted in a distorted playing field. We find that undisclosed private testing practices benefit a handful of providers who are able to test multiple variants before public release and retract scores if desired. We establish that the ability of these providers to choose the best score leads to biased Arena scores due to selective disclosure of performance results. 
At an extreme, we identify 27 private LLM variants tested by Meta in the lead-up to the Llama-4 release. We also establish that proprietary closed models are sampled at higher rates (number of battles) and have fewer models removed from the arena than open-weight and open-source alternatives. Both these policies lead to large data access asymmetries over time. Providers like Google and OpenAI have received an estimated
19.2\% and 20.4\% of all data on the arena, respectively. In contrast, a combined 83 open-weight models have only received an estimated 29.7\% of the total data. With conservative estimates, we show that access to \cba data yields substantial benefits; even limited additional data can result in relative performance gains of up to 112\% on ArenaHard, a test set from the arena distribution. Together, these dynamics result in overfitting to Arena-specific dynamics rather than general model quality. The Arena builds on the substantial efforts of both the organizers and an open community that maintains this valuable evaluation platform. We offer actionable recommendations to reform the \cba's evaluation framework and promote fairer, more transparent benchmarking for the field.
}

\usepackage{lipsum}
\usepackage{arydshln}
\usepackage{colortbl}

\renewcommand{\FONTauthors}{\bfseries\Large}

\begin{document}

\begin{figure}[htb]
    \centering
    \includegraphics[width=\linewidth]{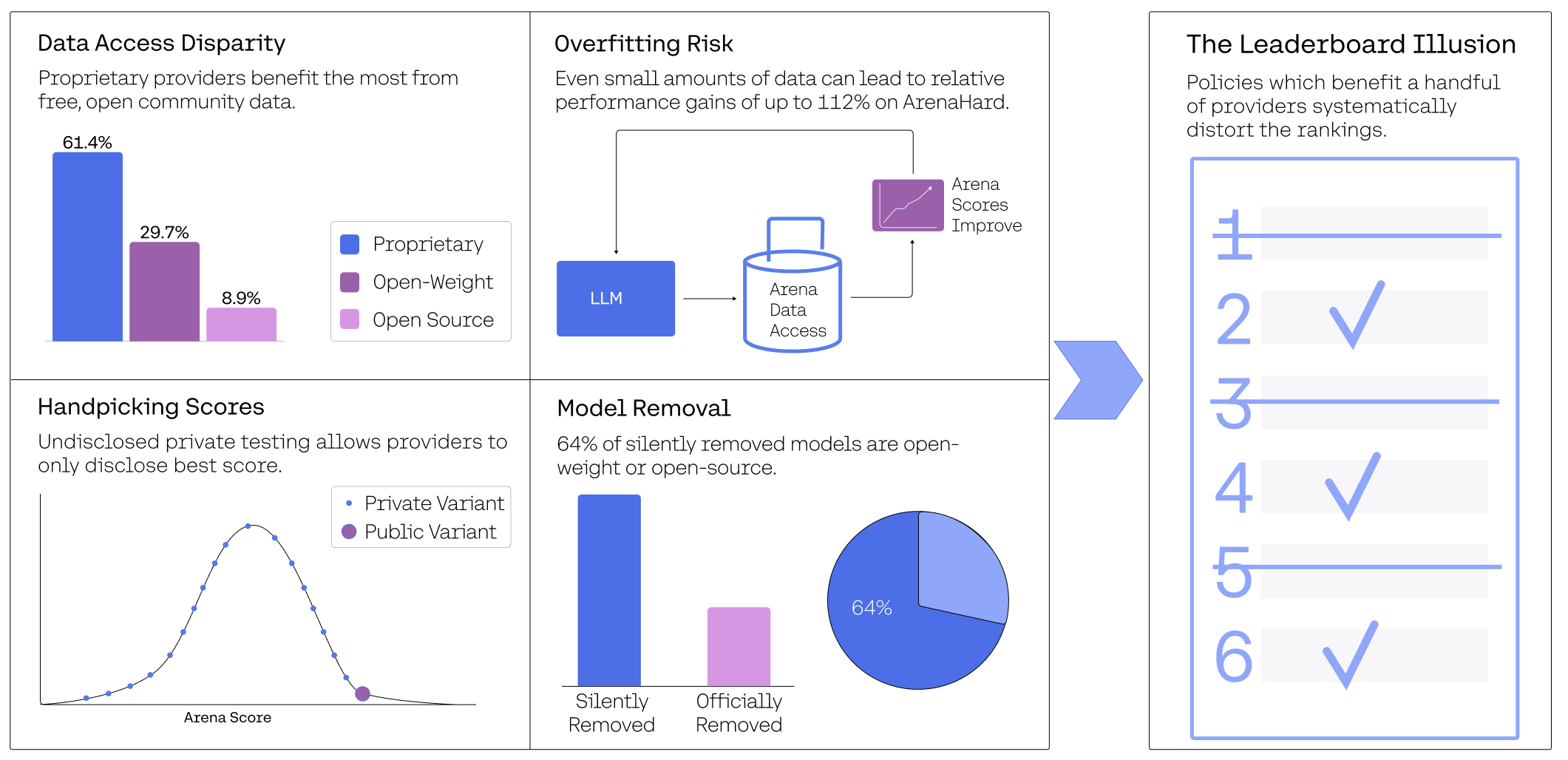}
    \caption{\textbf{Overview of key insights}. 
    We investigate the prevalence of \textbf{undisclosed private testing and selective score reporting} on the Arena (\Cref{sec:impact-of-ability-to-pricately-test}), and highlight significant \textbf{data access disparities} between proprietary and open-source providers (\Cref{sec-disparity-in-data-access}). These disparities enable \textbf{overfitting to the Arena} (\Cref{sec:risk_of_potential_overfitting}). Furthermore, \textbf{model deprecation practices} lack transparency, with many models silently deprecated without any notification to providers. We demonstrate how these deprecations contribute to unreliable rankings on the leaderboard (\Cref{sec:battle-connecitivty}).
    }
    \label{fig-overview}
\end{figure}

\newpage

\section{Introduction}

\begin{quote}
    \textit{Any observed statistical regularity will tend to collapse once pressure is placed upon it for control purposes.} \textbf{--- Charles A. E. Goodhart}
\end{quote}

Benchmarks have long played an integral role in the development of machine learning systems \citep{church_2017,koch2024protoscience}, from the early days of NLP \citep[e.g., TREC competitions and WMT shared tasks,][]{harman1993overview,koehn-monz:2006:WMT} to the recent flurry of deep learning models \citep[e.g., ImageNet,][]{deng2009imagenet}, to the rise of large language models \citep[e.g., GLUE,][]{wang2018glue}. 
A meaningful benchmark demonstrates the relative merits of new research ideas over existing ones, and thereby heavily influences research directions, funding decisions, and, ultimately, the shape of progress in our field.
The recent meteoric rise of generative AI models---in terms of public attention, commercial adoption, and the scale of compute and funding involved \citep{kaplan2020scaling,hoffmann2022training,samsi2023words,hooker2024limitationscomputethresholdsgovernance}---has substantially increased the stakes and pressure placed on leaderboards \citep{orr2024ai}. 

Recently, \cba has become the de facto standard for comparing generative AI models, holding enormous sway over media\footnote{\url{https://www.wsj.com/tech/ai/the-uc-berkeley-project-that-is-thewai-industrys-obsession-bc68b3e3}} \footnote{\url{https://www.bloomberg.com/news/articles/2025-02-18/before-deepseek-blew-up-one-website-announced-its-arrival}}, the AI industry \citep{aiindex2025}, and academia \citep{grattafiori2024llama, yang2024qwen2technicalreport}. Created in 2023, \cba rates LLMs by allowing anyone to submit a prompt and subsequently rank two anonymous responses from different models. \cba quickly rose to prominence because it helped fill a critical gap in evaluation at a time of rapid technological change. As generative AI technology has grown more capable, academic multiple-choice evaluations %\sanmi{delete "alone", or mention some of the other scalable evals e.g., boxed answers or Ai judges} alone 
\citep{hendrycks2020measuring, liang2022holistic, romanou2024include,singh2024global,wang2024mmlu,adelani2025irokobenchnewbenchmarkafrican,salazar2025kaleidoscopeinlanguageexamsmassively} have failed to reflect the real-world open-ended use cases these models are now being used for \citep{ustun-etal-2024-aya,chiang2024chatbot,mizrahi2024stateartmultipromptllm}. 
As a dynamic, user-driven evaluation framework where new questions can be asked daily, \cba can, at least in principle, evolve alongside model capabilities. Furthermore, the open-ended nature of the interaction can capture emerging real-world use cases that differ from those initially envisioned by model creators.

\begin{figure}[t]
    \centering
    \includegraphics[width=0.9\linewidth]{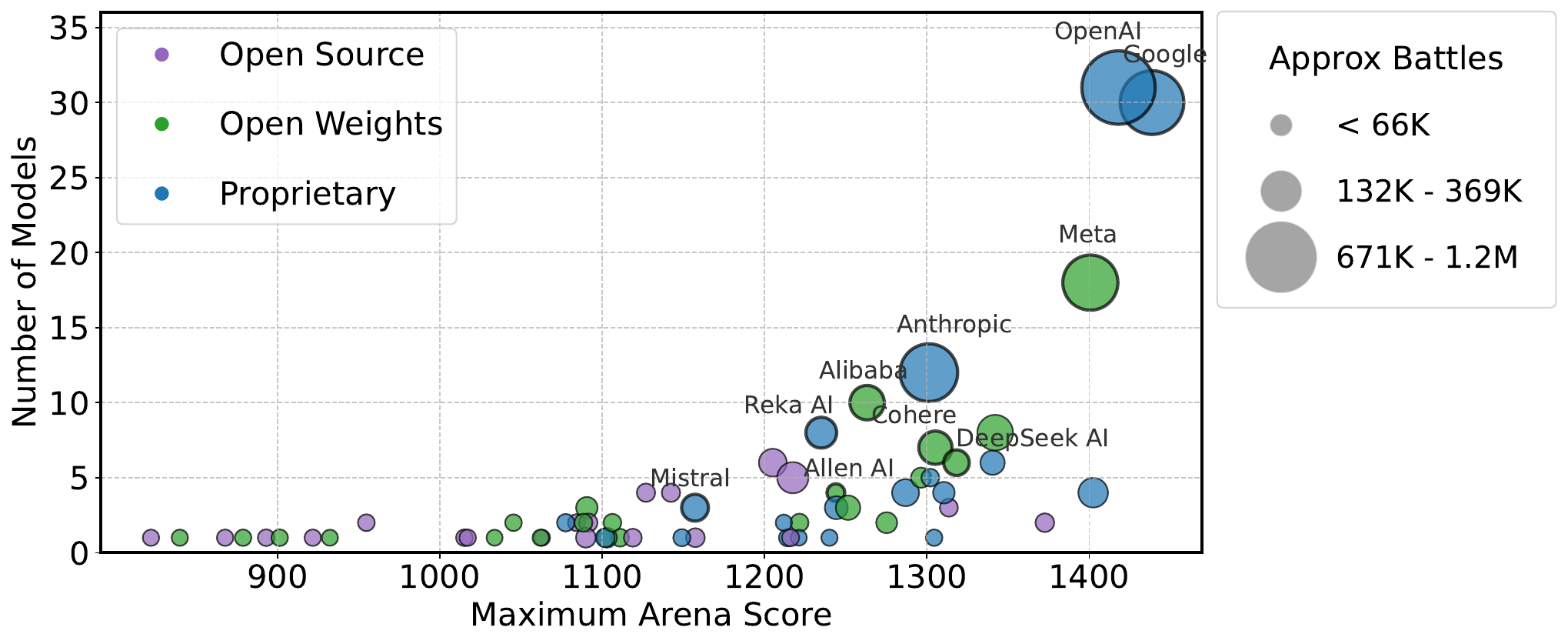}
    \caption{\textbf{Number of public models vs. maximum arena score per provider.} Marker size indicates total number of battles played. Proprietary model providers tend to achieve higher leaderboard scores, which appear to correlate with both the number of models they release and the number of Arena battles played. While model capability is an important factor, we explore in \Cref{sec:impact-of-ability-to-pricately-test} and \Cref{sec-measuring-overfititng} how increased exposure to the Arena (through more models and battles) may confer additional advantages, such as better model selection or adaptation to the evaluation distribution.
    This figure summarizes publicly disclosed results as of April 23rd, 2025.
    }
    \label{fig-max-rating-num-models-main-providers}
\end{figure}

However, the over-reliance on a single leaderboard creates a risk that providers may overfit to the aspects of leaderboard performance, without genuinely advancing the technology in meaningful ways \citep{ensmenger_2011, thomas2020problemmetricsfundamentalproblem, Inioluwa2021, bowman-dahl-2021-will}. As Goodhart's Law states, \textit{when a measure becomes a target, it ceases to be a good measure}. In this work,  we show that willful engagement from a handful of providers along with preferential policies from Chatbot Arena towards the same small group have amplified the potential for gamification, in place of innovative progress.

Our systematic review of \cba involves combining data sources encompassing 2M battles, auditing 42 providers and 243 models across a fixed time period (January 2024 - April 2025). This comprehensive analysis reveals that over an extended period, a handful of preferred providers have been granted disproportionate access to data and testing. In particular, we identify an undisclosed \cba policy that allows a small group of preferred model providers to test many model variants in private before releasing only the best-performing checkpoint. Using Arena data, we simulate and demonstrate how this strategy distorts overall arena ratings by beneficially skewing results. 
In addition, we find substantial data asymmetries are apparent regarding the amount of feedback data given to proprietary vs. open-weight vs. open-source model providers. This is due to a combination of unequal sampling rates (how often a model appears in a \cba battle) and deprecation policies (which models are retired from the arena). 
We show that these differences lead to distorted and unreliable Arena rankings and create conditions in which providers may be overfitting to Arena-specific dynamics rather than general model quality. We find:
%Across our extensive analysis, we present the following findings:
\begin{figure}
    \centering
    \includegraphics[width=0.8\linewidth]{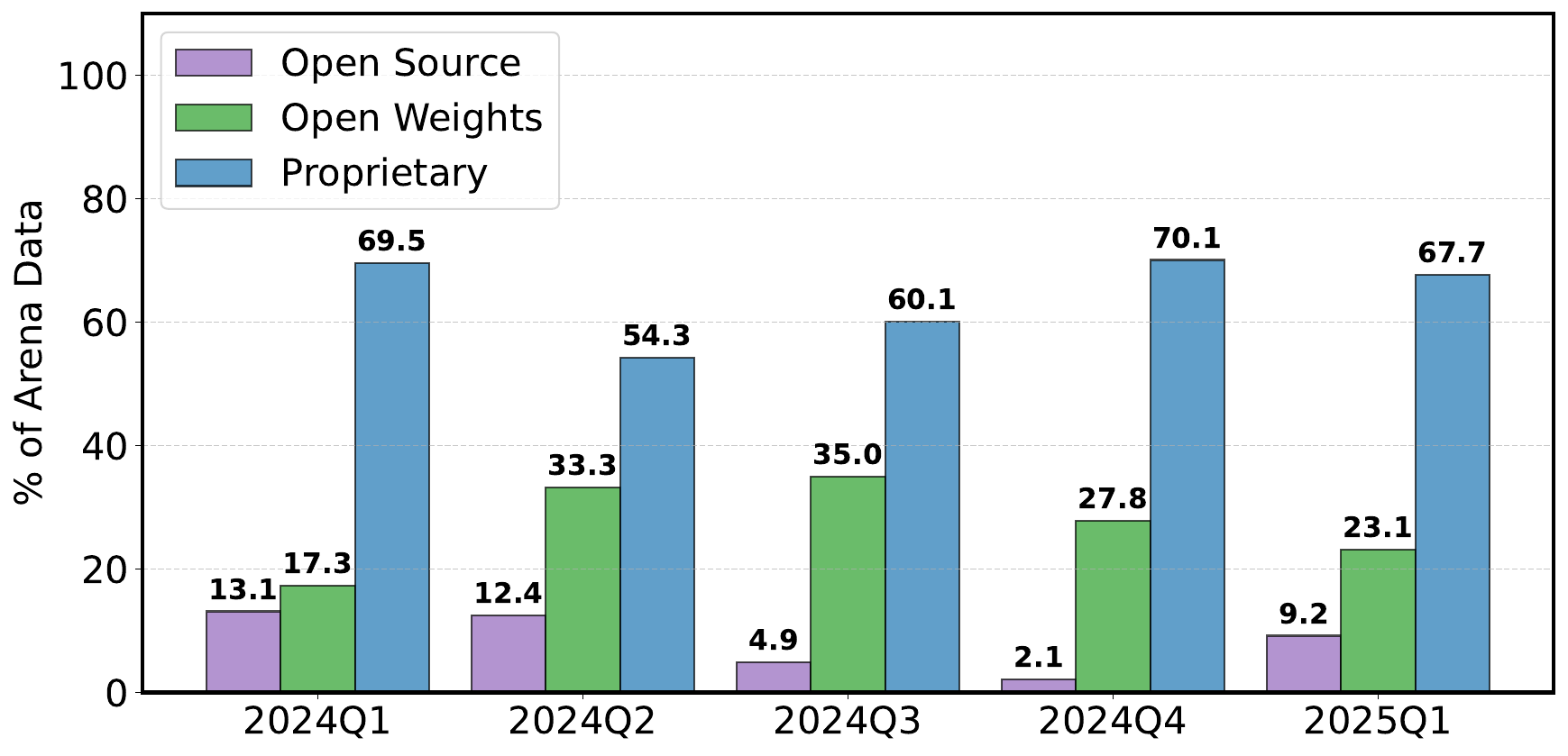}
    \caption{\textbf{Volume of Arena battles involving proprietary, open-weight, and fully open-source model providers from January 2024 to March 2025, based on \texttt{\lmarenaleaderboard}.} Proprietary models consistently received the largest share of data---ranging from 54.3\% to 70.1\%. Open-weight and fully open-source models receive significantly less data, in some cases receiving less than half the amount of data as proprietary developers. 
    This imbalance in data access exacerbates the performance gap, reinforcing unequal access to high-quality in-distribution data.
    }
    \label{fig-percentage-battles-open-proprietary}
\end{figure}

\begin{enumerate}[noitemsep,topsep=0pt]
    \item \textbf{Preferential treatment around private testing and retraction.} \cba has an unstated policy of allowing select providers to test many submissions privately in parallel. We show that certain model developers (most notably Meta, Google, Open AI and Amazon) have benefited from extensive private testing. In a single month, we observe as many as 27 models from Meta being tested privately on \cba in the lead up to llama 4 release. Notably, we find that \cba does not require all submitted models to be made public, and there is no guarantee that the version appearing on the public leaderboard matches the publicly available API\footnote{\url{https://x.com/lmarena_ai/status/1909397817434816562}} \footnote{\url{https://www.theverge.com/meta/645012/meta-llama-4-maverick-benchmarks-gaming}}. We show with real-world experiments and simulations that the ability to select the best-scoring variant from $N$ models enables systematic gaming of the Arena rating.
    \item \textbf{Far more data is released to proprietary model providers.} \cba is a community-driven leaderboard that benefits from free, crowdsourced feedback provided by everyday users. However, proprietary model providers collect significantly more test prompts and model battle outcomes than others. Google and OpenAI have received an estimated 19.2\% and 20.4\% of all test prompts on the arena, respectively. We estimate this based on the share of total battles played by the models of different providers on the Arena (as shown in \Cref{fig:public_private_data}). In contrast, a combined 41 fully open-source models have only received an estimated 8.9\% of the total data, collectively.
    \item \textbf{\cba data access drives significant performance gains.} The differences in data access between providers matter; we estimate that by training on \cba data, model ranking can be improved significantly. In a controlled experimental setting, we observe that increasing the arena training data proportion from (0\%~$\rightarrow~$70\%) more than doubles the win-rates from 23.5\% to 49.9\% on ArenaHard \citep{li2024crowdsourced}. We believe this is a conservative estimate, as a subset of providers have disproportionate access to private API data, which, if used, can potentially yield even greater performance gains. 

     \item \textbf{Deprecations can result in unreliable model rankings.}  
     As shown in \Cref{fig-silent-deprecated}, out of 243 public models, 205 have been silently deprecated. This is a significantly higher number than the 47 models officially listed as deprecated as part of \cba's backend codebase, FastChat\footnote{\url{http://github.com/lm-sys/FastChat/blob/0e6d3e4beaab66f4d3f93db72541a4abab8af28d/fastchat/serve/monitor/monitor_md.py\#L7} \label{monitor_md}}. We show that deprecation can violate key assumptions of the Bradley-Terry model \citep{bradley1952rank},
     which underpins Arena scoring, thereby reducing the reliability of the leaderboard rankings. Critically, we find that open-weight and open-source models are far more likely to be deprecated and, hence, receive unreliable ratings. Among the models that are silently deprecated, 64\% are open-weight or fully open-source. 
\end{enumerate}

\begin{figure}[bt!]
    \centering
    {
    \includegraphics[width=0.7\linewidth]{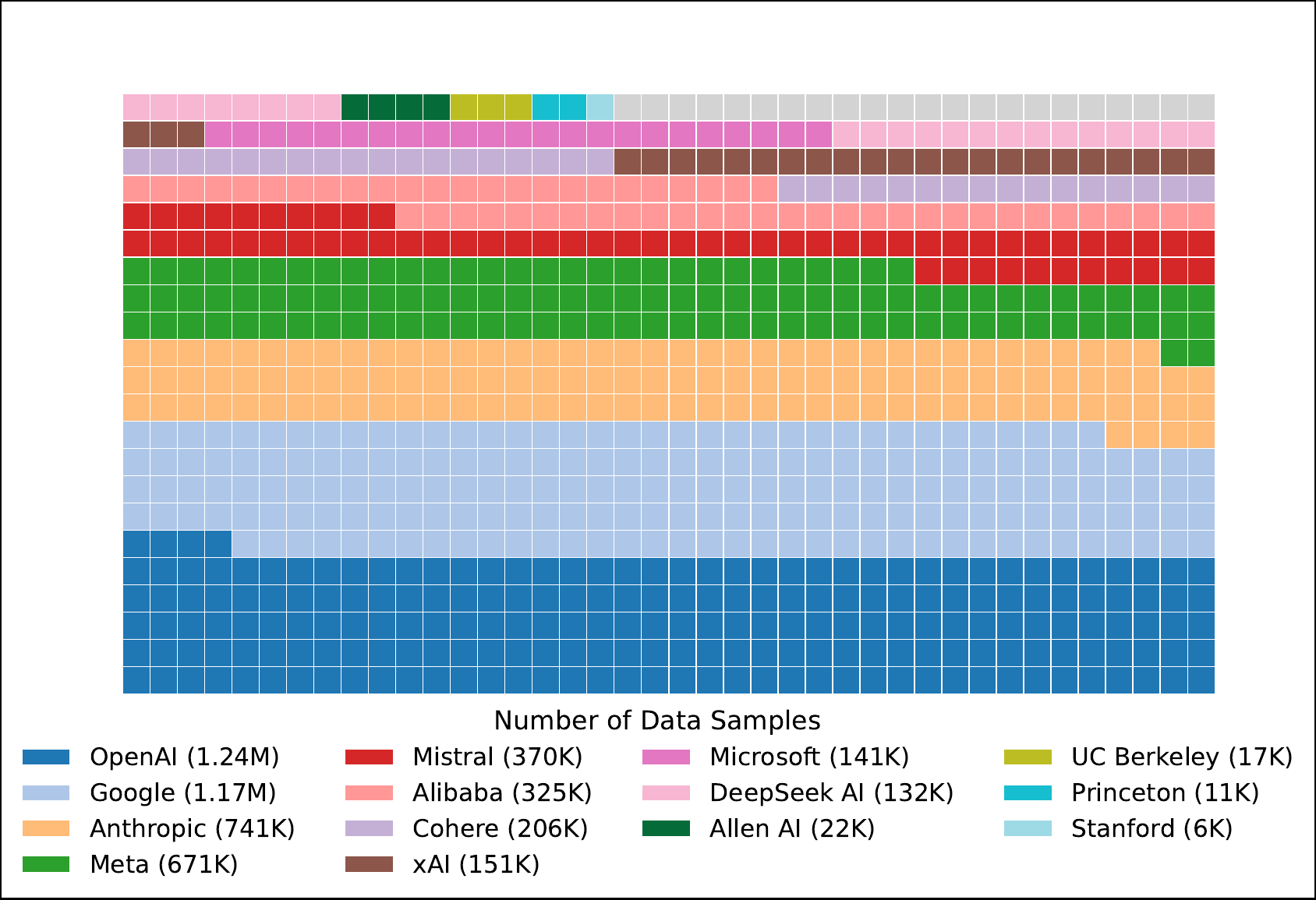}}
    \caption{\textbf{Data availability to model providers.} We observe large differences in data access between providers, with 61.4\% of all data going to proprietary providers.}
    \label{fig:public_private_data}
\end{figure}

\begin{figure}
    \centering
    \includegraphics[width=0.87\linewidth]{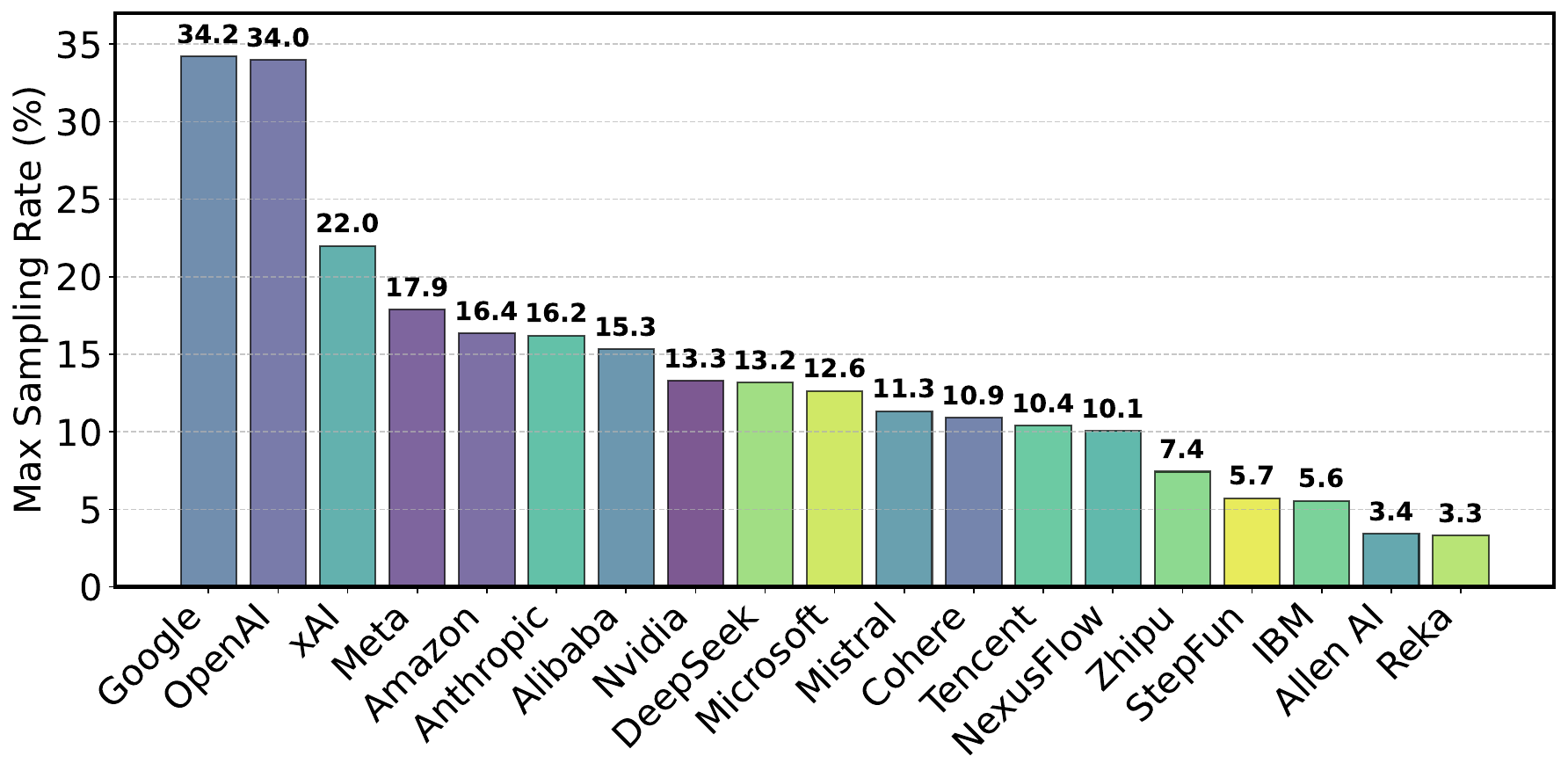}
    \caption{\textbf{Maximum observed sampling rate for models from different providers}. The sampling rate determines the amount of times a model is shown to everyday users, and the amount of data a provider receives. We observe large discrepancies across providers, with substantially higher sampling rates for OpenAI, Google, xAI, and Meta compared to others.}
    \label{fig-private-testing-max}
\vspace*{-3mm}
\end{figure}

\textbf{Our role in both participating in the leaderboard and providing an overview of its limitations.} 
It is important to acknowledge that a subset of the authors of this paper have submitted several open-weight models to \cba: \texttt{command-r} \citep{cohere2024}, \texttt{command-r-plus} \citep{cohere2024} in March 2024, \texttt{aya-expanse} \citep{dang2024aya} in October 2024, \texttt{aya-vision} \citep{CohereAya2025} in March 2025, \texttt{command-a} \citep{cohere2025commandaenterprisereadylarge} in March 2025. We started this extensive study driven by this submission experience with the leaderboard. While submitting Aya Expanse \citep{dang2024aya} for testing, we observed that our open-weight model appeared to be notably undersampled compared to proprietary models --- a discrepancy that is further reflected in Figures \ref{fig-percentage-battles-open-proprietary}, \ref{fig:public_private_data}, and \ref{fig-private-testing-max}. 
In response, we contacted the \cba organizers to inquire about these differences in November 2024. In the course of our discussions, we learned that some providers were testing multiple variants privately, a practice that appeared to be selectively disclosed and limited to only a few model providers. We believe that our initial inquiries partly prompted \cba to release a public blog in December 2024 detailing their benchmarking policy\footnote{\url{https://drive.google.com/file/d/1reook2cjwq81xD6Yn528KOLWeWRy0ZvN/view?usp=sharing}\label{policy}}, which committed to a consistent sampling rate across models. However, subsequent anecdotal observations of continued sampling disparities and the presence of numerous models with private aliases motivated us to undertake a more systematic analysis. As part of this analysis, we launched private variants to estimate the benefit of multiple submissions and retraction. We report on the results of this real-world study in \Cref{sec:real-world-exp}. 
Consequently, Cohere appears in \Cref{fig-private-testing-providers} and \Cref{fig:modelcomparisons} as having private variants. Before this period, Cohere had not launched any private testing.

While our work points out limitations and unreliability issues with \cba, it is important to acknowledge that \cba originated as an academic-sponsored leaderboard that took on pronounced importance to the machine learning community. It takes considerable effort to coordinate a leaderboard, and arena organizers have often faced large demands on their time from supporting the addition of models from many different model providers. Hence, while we point out systematic issues with \cba as it is now, we also believe some of these issues may have emerged gradually as successive design choices and concessions to certain large providers accumulated and introduced unreliability. We believe that there are reasonable interventions that the organizers can do to restore scientific accuracy and renew trust in the leaderboard. We have directly shared our recommendations with \cba prior to releasing this work. We include these here and also share in \Cref{sec:recommendations} for a complete discussion.

\begin{tcolorbox}[title=Critical Recommendations to Restore Trust to the \cba,colback=thomblue1!9]
In order to restore fairness, transparency and trust to the \cba, we determine that the following five changes are both actionable, easy to implement as well as \textbf{urgent and necessary for restoring scientific integrity}:
\begin{enumerate}[left=0pt]
    \item \textbf{Prohibit score retraction after submission.} All model evaluation results—including private variants—must be permanently published upon submission, with no option to retract or selectively hide scores. This is a crucial and necessary requirement to restore trust in the leaderboard. There is no reasonable scientific justification for allowing a handful of preferred providers to selectively disclose results. Indeed, we show that this skews Arena scores upwards and allows a handful of preferred providers to gamify the leaderboard. 
    \item \textbf{Establish transparent limits on the number of private models per provider.} There are valid reasons to offer private testing -- as it allows the \cba community to test models in development while also preserving the anonymity required before a release. However, some providers are engaging in private testing \textit{far} more than others. This has implications for gamification, but it also means preferred providers benefit from substantially more data from an open community project. At the time Meta launched 27 variants in the lead up to Llama 4 release. This highly skews data access from a community resource.
    It is critical to enforce a strict and publicly disclosed limit (e.g., maximum 3 model variants tested concurrently per provider). This prevents excessive testing runs that skew the leaderboard while ensuring fair, transparent benchmarking for all.
    If private testing continues, it should have transparent limits and be paired with disclosing all scores.
     \item  \textbf{Ensure model removals are applied equally to proprietary, open-weights, and open-source models.} We find deprecation disproportionately impacts open-weight and open-source models, creating large asymmetries in data access over time (\Cref{sec-disparity-in-data-access}). We strongly urge arena organizers to remove an equal share across proprietary, open-weight, and open-source licenses to avoid creating asymmetric access to data that disadvantages open-weight providers. For example, a reasonable and transparent policy would be to deprecate the bottom 30th percentile for each group of proprietary, open-weights, open-source models. 
    \item \textbf{Implement fair sampling.} Our work finds large biases towards proprietary models in sampling rates. For example, in \Cref{fig-private-testing-max}, we observe a maximum daily sampling rate of models from OpenAI or Google of up to 34\%, which is 10 times more than what is observed as the maximum sampling from providers like Allen AI. We suggest a return to an active sampling method proposed by the arena organizers themselves \citep{chiang2024chatbot}, which appears not to be respected or implemented in practice. This formulation avoids simply favoring large proprietary providers and instead effectively prioritizes under-evaluated and high-variance pairs. This avoids preferring a subset of providers and instead focuses votes to align sampling with the goal of rapidly reducing uncertainty in rankings.
    \item \textbf{Provide transparency into what models are being removed from the leaderboard.} We found that a large number of models are silently deprecated without being added to the list of officially deprecated models. While 47 out of 243 public models are officially listed as deprecated, we measure that 205 have been silently removed by the organizers without informing the provider. For transparency and to ensure that the deprecation policy is implemented fairly, it is important that there is a comprehensive list of what models have been removed from the leaderboard. 
\end{enumerate}
\end{tcolorbox}

\section{Overview of Methodology}

\begin{quote}
    \textit{To measure is to know} \textbf{--- Lord Kelvin}
\end{quote}

Leaderboards have played a pivotal role in the computer science field in driving technological progress \citep{koch2024protoscience}. A desirable goal of a general benchmark is to reliably rank models according to real-world capabilities and preferences. Hence, if one system significantly outperforms another, then the result should be strong evidence that the higher-ranked system is better at the task \citep{bowman-dahl-2021-will,rodriguez-etal-2021-evaluation}. 
If a leaderboard is not representative of real performance gains, it can lead practitioners and researchers towards less impactful areas of intervention and inquiry. 
In some cases, this misalignment enables \textit{gaming} of the system---where participants optimize for leaderboard metrics rather than real-world utility.
As a result, the gap between reported performance and actual value in deployment can continue to widen. Ultimately, the goal of our work is to ask whether the rankings of \cba are a \textit{reliable} ranking of generative AI models. To answer this, across the sections that follow, we will consider and answer the following questions:
\begin{enumerate}
    \item \textbf{Private Testing and Selective Disclosure (\Cref{sec:impact-of-ability-to-pricately-test})} What is the impact of the undisclosed policy of multiple private variant testing and selective retraction on Arena scores? 
    \item \textbf{Data Access Asymmetries (\Cref{sec-disparity-in-data-access})} What factors lead to large asymmetries between providers in data access on the \cba?
    \item \textbf{Risk of Overfitting (\Cref{sec:risk_of_potential_overfitting})} How do asymmetries in data access impact Arena scores? 
    \item \textbf{Model Deprecation (\Cref{sec:battle-connecitivty})} What is the impact of \cba model deprecation policy on Arena score reliability? 
\end{enumerate}

To gain insights and analyze various trends in the \cba leaderboard, we leverage multiple data sources. In total, our real-world data sources encompass 2M battles and cover 243 models across 42 providers. We present the overview of each data source in \Cref{tab:data_sources} and describe the different datasets in more details in \Cref{app:data-sources}. Across our analyses, we group models by the licenses reported on \cba into proprietary, open-weights, and open source to understand overall trends. This allows us to gain insight into who benefits the most from the \cba which is a free and open community resource. We report details of what licenses are grouped into each category in \Cref{app:license-categories}. We note that many open-weights models are extremely performant and include Aya family of models \citep{dang2024aya, ustun-etal-2024-aya}, Meta models \citep{llama4} and Gemma family of models \citep{gemmateam2025gemma3technicalreport}.

\begin{table}[ht!]
    \scriptsize
    \centering
    \caption{A summary of datasets we constructed, their sources, and the research questions they enabled us to answer. These datasets can be of one of the following types: \textbf{battles only} (\iconcolor{\faArrowsAltH}), \textbf{conversations only} (\iconcolor{\faComments}), \textbf{battles with conversations}, (\iconcolor{\faArrowsAltH} \iconcolor{\faComments}) and \textbf{leaderboard updates} (\iconcolor{\faTrophy}). Depending on the dataset type, it either \textbf{contains prompts} (\iconcolor{\faCheck}) or doesn't (\iconcolor{\faTimes}). Accessibility of the datasets is indicated using \textbf{public} (\iconcolor{\faGlobe}) or \textbf{private} (\iconcolor{\faLock}).}
    \label{tab:data_sources}
    \resizebox{1.0\textwidth}{!}{
    \renewcommand{\arraystretch}{1.5} % Increases row height for vertical centering
    \begin{tabular}{|m{1.5cm}|m{2cm}|m{2.5cm}|>{\centering\arraybackslash}m{0.8cm}|>{\centering\arraybackslash}m{0.8cm}|>{\centering\arraybackslash}m{1.2cm}|m{0.8cm}|m{1.8cm}|}
    \hline
    \textbf{Name} & \textbf{Fields} & \textbf{Source} & \textbf{\iconcolor{\faGlobe}/\iconcolor{\faLock}} & \textbf{Type} & \textbf{Prompts?} & \textbf{Size} & \textbf{Period} \\ \hline
        \multirow{3.5}{2cm}{Historical Battles}
        & \multirow{3.5}{2cm}{battle dates, category \& language tags} & Arena-human-preference-100k\footnote{\url{https://huggingface.co/datasets/lmarena-ai/arena-human-preference-100k}} & \iconcolor{\faGlobe} & \iconcolor{\faArrowsAltH} \iconcolor{\faComments} & \iconcolor{\faCheck} & 100K & \multirow{3}{2cm}{04-23 - 01-25} \\ 
        \cline{3-7}
         &  & Colab data\footnote{\url{https://blog.lmarena.ai/blog/2023/leaderboard-elo-update/}} \footnote{\url{https://blog.lmarena.ai/blog/2023/arena/}} & {\iconcolor{\faGlobe}} & {\iconcolor{\faArrowsAltH}} & \iconcolor{\faTimes}  & {1.9M}  &  \\ 
         \cline{3-7}
         &  & {LMArena, Cohere} & {\iconcolor{\faLock}} & {\iconcolor{\faArrowsAltH} \par \iconcolor{\faComments}} & {\iconcolor{\faCheck}} & {43K} &  \\ 
         \hline
        Random Scraped Battles & model identity responses, battle players & Crawled & \iconcolor{\faLock} & \iconcolor{\faArrowsAltH} & \iconcolor{\faTimes} & 5.8K & 01-25 - 03-25 \\ \hline
        API prompts & prompts & Cohere  & \iconcolor{\faLock} & \iconcolor{\faComments} & \iconcolor{\faCheck} & 197K & 11-24 - 04-25  \\ \hline
        Leaderboard Statistics & ratings, dates, models, battles counts, licenses, providers & HuggingFace Leaderboard Commit History\footnote{\url{https://huggingface.co/spaces/lmarena-ai/chatbot-arena-leaderboard/tree/main}} & \iconcolor{\faGlobe} & \iconcolor{\faTrophy} & \iconcolor{\faTimes} & 14.3K & 01-24 - 04-25 \\ 
        \hline
    \end{tabular}
    }
\end{table}

\subsection{Introduction to the Arena Score}
\label{sec:arena-score}

The Arena score is the key metric used in \cba to rank models based on how humans judge their performance. In Arena, users engage in pairwise comparisons, voting on which of two models performs better (or declaring a tie) in a given interaction. A naive way to compute rankings would be to average each model's win-rates, simply counting how often they win. However, this approach does not consider the strength of the opponents and requires all models to play against each other. Furthermore, averaging win-rates treats all wins equally, which can lead to misleading or unfair rankings, especially when models face opponents of varying strength levels. 

Instead, \cba ranks models based on their Arena Score, which is a normalized version of the Bradley-Terry (BT) model \citep{bradley1952rank}, a probabilistic framework for estimating skill levels based on pairwise comparisons. BT is a widely used model to estimate skill levels in diverse fields such as online gaming (e.g., StarCraft II) 
\citep[i.a.]{leitner2010forecasting, liquipedia_elo_rating, Wise2021EloRF} and
sports rankings \citep{glickman1999parameter,hunter2004mm}. Compared to alternative systems like Elo \citep{elo1978rating,boubdir2024elo}, which was initially used in the first version of \cba, BT provides a more robust and statistically grounded basis for ranking. BT produces well-calibrated estimates while naturally accommodating ties and missing comparisons, and providing confidence intervals for the rankings as long as the underlying model conditions are met. 

While BT is considered to be a principled ranking system, its effectiveness depends on several key assumptions. Background on these assumptions will provide important context for our discussions in \Cref{sec:dist-shift} and \Cref{sparse_battle_history}, where we examine how violations of BT's assumptions affect the reliability of Arena rankings.  The BT model assumes that pairwise comparisons are drawn from an \textbf{unbiased sampling process}, so that each model's skill parameter is estimated independently and based on representative outcomes. It also does not require all players to play each other because of \textbf{transitivity} of rankings---if Model A is judged to beat Model B and Model B in turn beats Model C, then Model A is estimated to outperform Model C (even if they never actually play). Finally, it assumes that the comparison network is \textbf{fully interconnected}---that is, every model must be linked directly or indirectly through pairwise matchups. Deviations from these assumptions, such as biased comparisons or a fragmented graph, can compromise the reliability and consistency of the strength estimates. For further details, we include the mathematical formulation of the BT model in \Cref{app:bradley-terry}. 

\section{Results: Impact of Private Testing and Selective Retraction on Arena Scores}
\label{sec:impact-of-ability-to-pricately-test}

\begin{tcolorbox}[title=Section Findings,colback=thomblue1!9]
\begin{itemize}[left=0pt]
    \item \textbf{Undisclosed testing policy permits preferred providers to bias reporting of results.} \cba currently permits a small group of preferred providers to test multiple models privately and only submit the score of the final preferred version. As observed in \Cref{fig-private-testing-providers}, we observe that Meta, Google, and Amazon were key beneficiaries of this policy for the period for which we collected data.
    \item \textbf{Selective disclosure violates unbiased sampling assumption of BT model.} We show theoretically how BT assumptions are violated, which systematically inflates model rankings, distorting the leaderboard ranking. In \Cref{fig:number-of-variants}, we simulate the expected distortions to leaderboard Arena rankings and show that testing just 10 variants yields a notable increase of approximately 100 points in the maximum score identified.     
    \item \textbf{Real world experiments on the \cba confirm benefits of selective reporting of results.} We corroborate our simulations with experiments on the real arena by deploying multiple private variants. We show that even limited numbers of variants lead to large gains in Arena Scores. Submitting multiple model variants to the \cba can lead to a systematic advantage in rankings, even when the underlying models are identical or only marginally different.
\end{itemize}
\end{tcolorbox}

\subsection{Preferred Providers Frequently Use Private Testing}

Although not an officially stated policy\footref{policy}, our audit of \cba data using \texttt{random-sample-\\battles} revealed that providers are permitted to test multiple private model variants simultaneously, without being required to publicly release or de-anonymize these submissions. In \Cref{fig-private-testing-providers}, we plot the number of private variants we tracked as belonging to each provider from January to March 2025. 
Meta and Google had the most active private models during this period, with 27 and 10 models, respectively. Notably, Meta's private testing occurred just before the release of the Llama 4 models \citep{llama4}.
For Google, the majority of private testing was driven by Gemini models (proprietary), with only one observed test involving Gemma 3 (open-weights) \citep{gemma3_announcement}.

We note this is likely a very conservative estimate as it only tracks the private variants on the main \cba, and does not take into account private variants on specialized leaderboards run by Arena such as for vision \footnote{\url{https://blog.lmarena.ai/blog/2024/multimodal/}} or code.\footnote{\url{https://blog.lmarena.ai/blog/2024/copilot-arena/}} \footnote{\url{https://blog.lmarena.ai/blog/2025/repochat-arena/}} Indeed, if we also consider the number of private models tested by Meta on the vision leaderboard, we observe an additional 16 variants, bringing the total to 43. In contrast, smaller startups, such as Reka, were found to have one active private variant live in the arena. We note that during the same period, Cohere submitted 4 private variants---these ablations were part of experiments measuring the lift that could be expected from private testing that we detail in the experiments in \Cref{sec:private-testing-simulation} and \Cref{sec-disparity-in-data-access}. Before this, Cohere had never submitted private variants for testing in the \cba. 

\begin{figure}[htb!]
    \centering
    \includegraphics[width=0.8\linewidth]{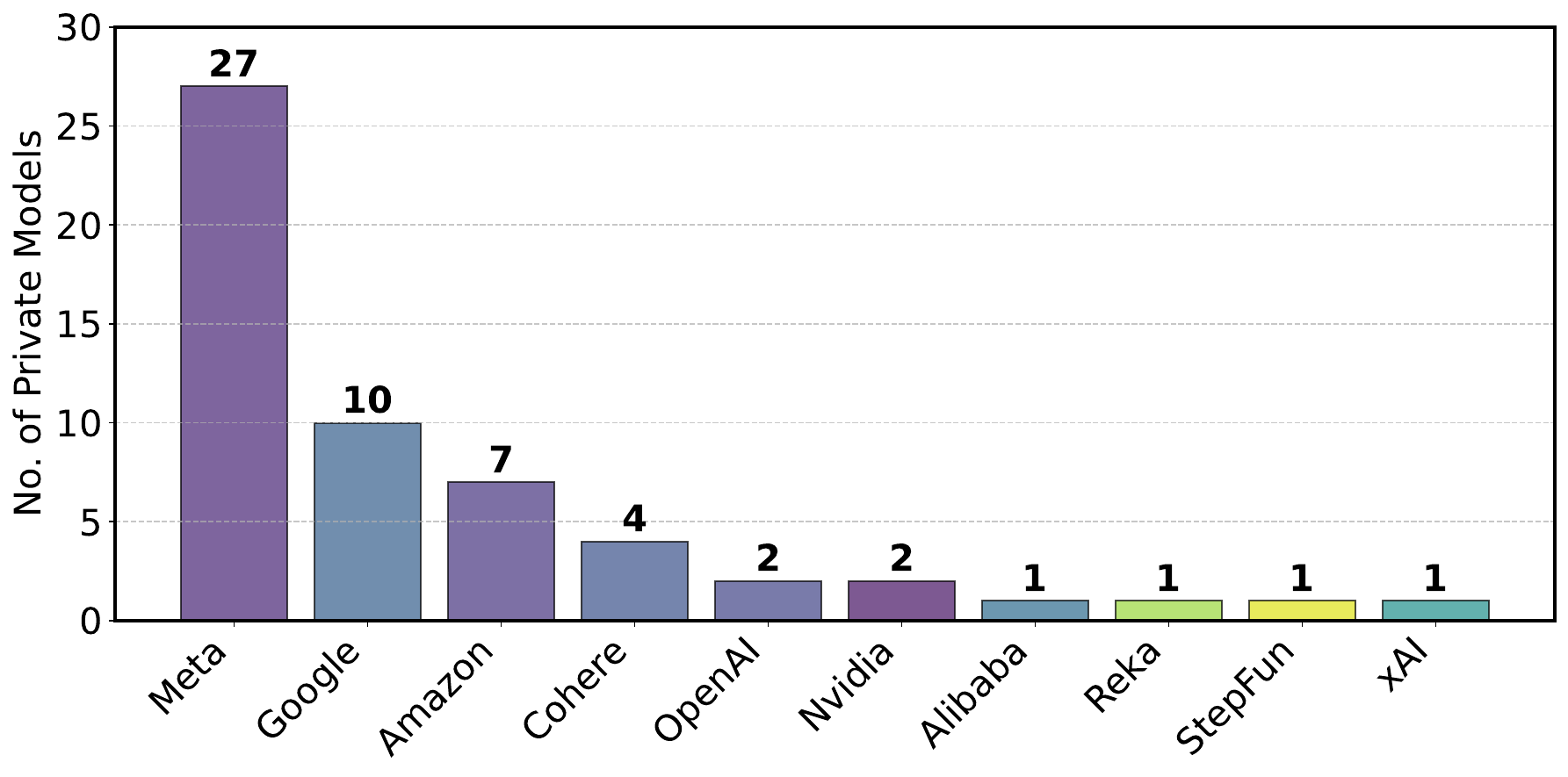}
    \caption{\textbf{Number of privately-tested models per provider based on \texttt{random-sample-\\battles} (January--March 2025)}. Meta, Google, and Amazon account for the highest number of private submissions, with Meta alone testing 27 anonymous models in March alone. 
    % This suggests that major providers may gain disproportionately greater access to \cba data.
     We note that during the same period, Cohere submitted private variants---these ablations were part of experiments submitted by the authors of this work to measure the lift that could be expected from private testing that we detail in the experiments in \Cref{sec:private-testing-simulation} and \Cref{sec-disparity-in-data-access}.
    }
    \label{fig-private-testing-providers}
\end{figure}

We only scraped data from January to March 2025, yet we anecdotally observed behavior that suggests submitting multiple variants was a long-standing practice amongst a subset of providers. 
Over the last year, we have observed that major LLM providers such as Google, xAI, and OpenAI are often announced as having the top-performing variant within just a few days of one another.
For example, OpenAI's \texttt{GPT-4.5} and xAI's \texttt{Grok-3} reached the top of the \cba leaderboard within the same day (March 4, 2025)\footnote{\url{https://x.com/lmarena_ai/status/1896675400916566357}} \footnote{\url{https://x.com/lmarena_ai/status/1896590146465579105}}. 
\texttt{Gemini (Exp 1114)} from Google DeepMind reached the top of the leaderboard on November 14, 2024\footnote{\url{https://x.com/lmarena_ai/status/1857110672565494098}} and shortly after, \texttt{ChatGPT-4o (20241120)} from OpenAI claimed the top position on November 20, 2024\footnote{\url{https://x.com/lmarena_ai/status/1859307979184689269}}. Just one day later, on November 21, 2024, \texttt{Gemini (Exp 1121)} regained the top spot\footnote{\url{https://x.com/lmarena_ai/status/1859673146837827623}}.
Given the time typically required to develop, refine, and test a foundation model, it is unlikely for the same provider to top the leaderboard twice in a single week unless they were testing multiple variants simultaneously.
In \Cref{sec:private-testing-simulation}, we demonstrate through simulated experiments that rapid leaderboard turnover can plausibly emerge from providers optimizing for the highest possible score by testing multiple model variants in parallel. 

Notably, we found that no private models were tested by academic labs during the observed period. This leads us to believe only certain providers were made aware they could submit multiple private variants, as we observe clear differences in the number and frequency of private testing among providers.

\subsection{Simulated Experiments on Private Testing and Retraction}\label{sec:private-testing-simulation}

Private testing coupled with the option to retract enables a best-of-N strategy, where an organization submits multiple model variants to \cba, privately evaluates them, and retains only the top-performing variant to be publicly published on the leaderboard.
In this section, we show that best-of-N submissions violate the BT unbiased sampling assumption. This systematically inflates model rankings and distorts the leaderboard ranking. 

\begin{figure}[htb!]
    \centering
    \includegraphics[width=0.85\linewidth]{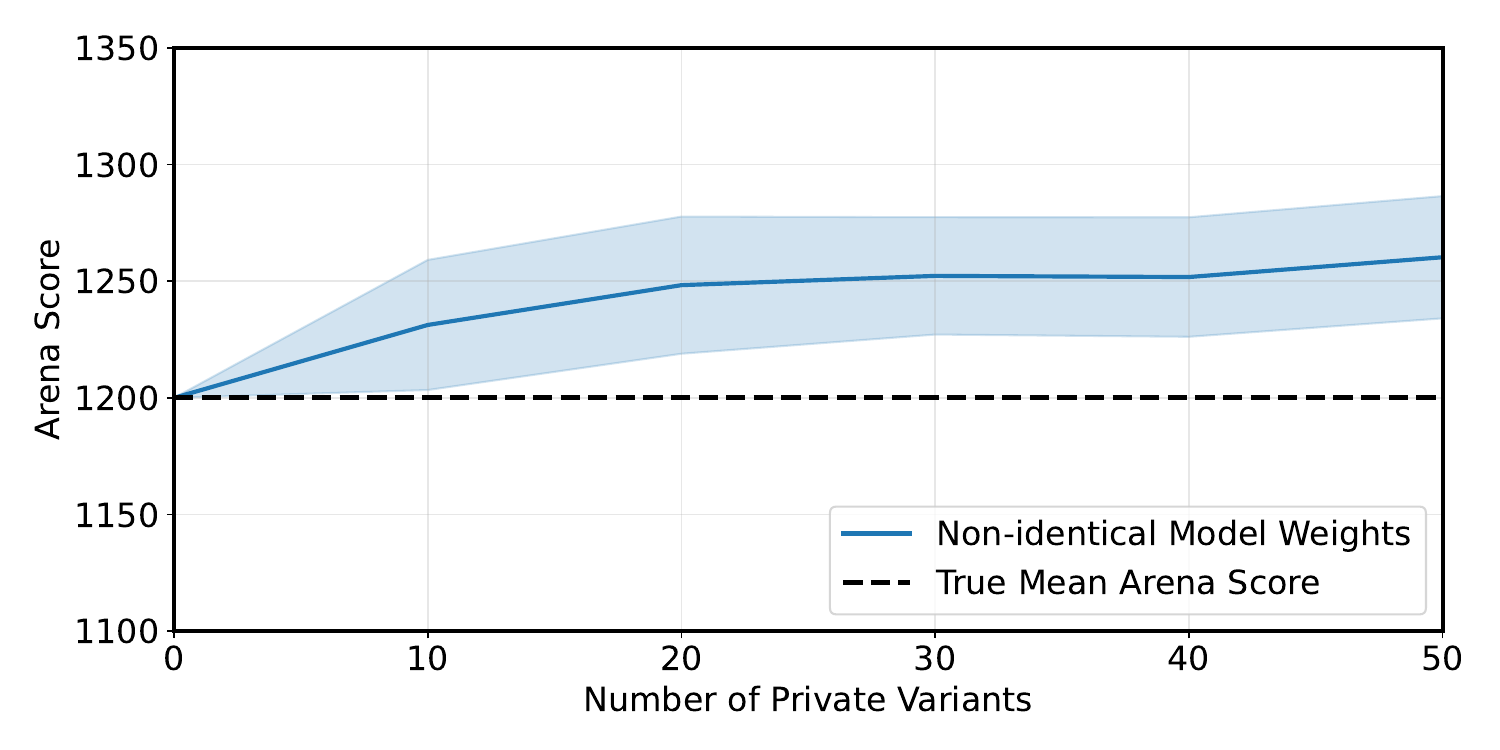}
    \caption{\textbf{Impact of the number of private variants tested on the best Expected Arena Score.} We simulate a family of model variants with a latent average Arena Score of 1200. As we progressively increase the number of private variants tested---and subsequently discover their corresponding Arena Scores---the probability of selecting models from the higher end of the performance distribution also rises. This enables the provider to effectively identify the model with the highest score.
    }
    \label{fig:number-of-variants}
\end{figure}

\textbf{Unbiased Sampling Assumption.} 
To study the selection bias scenario, assume a provider submits \(N\) variants of a model, each variant \(k\) having a true underlying skill parameter \(\beta_k\), sampled from a distribution centered at some base skill level \(\beta\). The probability of observing an exceptionally high-performing variant increases with the number of submissions \(N\). Thus, the observed skill of the submitted model is:  $\hat{\beta}_{\text{Best}} = \max\{\hat{\beta}_1, \hat{\beta}_2, \dots, \hat{\beta}_N\}$.

Since each \(\hat{\beta}_k\) is subject to statistical fluctuation due to finite match sampling, selecting the best variant based on observed performance introduces an upward bias. Specifically, the expected value of the best-performing variant is \textit{strictly greater than} that of a regular submission: 
\begin{equation}
    \mathbb{E}[\hat{\beta}_{\text{Best}}] > \mathbb{E}[\hat{\beta}_k] \text{,} \qquad \forall k \in \{1,2,\dots,N\} .
\label{eq:exp_best_of_N}
\end{equation}
where the draws are \emph{non-degenerate} ($\operatorname{Var}(\hat{\beta}_k) > 0$) and $N \geq 2$ (See \Cref{app:unbiased_sampling} for further details).
This violates the BT model's assumption of unbiased sampling and alters the likelihood landscape. The reported rating no longer reflects a single, unbiased estimate of skill, but an extreme value from multiple independent estimations. As a result, the BT estimator systematically inflates the ratings of models submitted under the best-of-N strategy, distorting leaderboard rankings.

\begin{figure}[t!]
    \centering
    \includegraphics[width=0.85\linewidth]{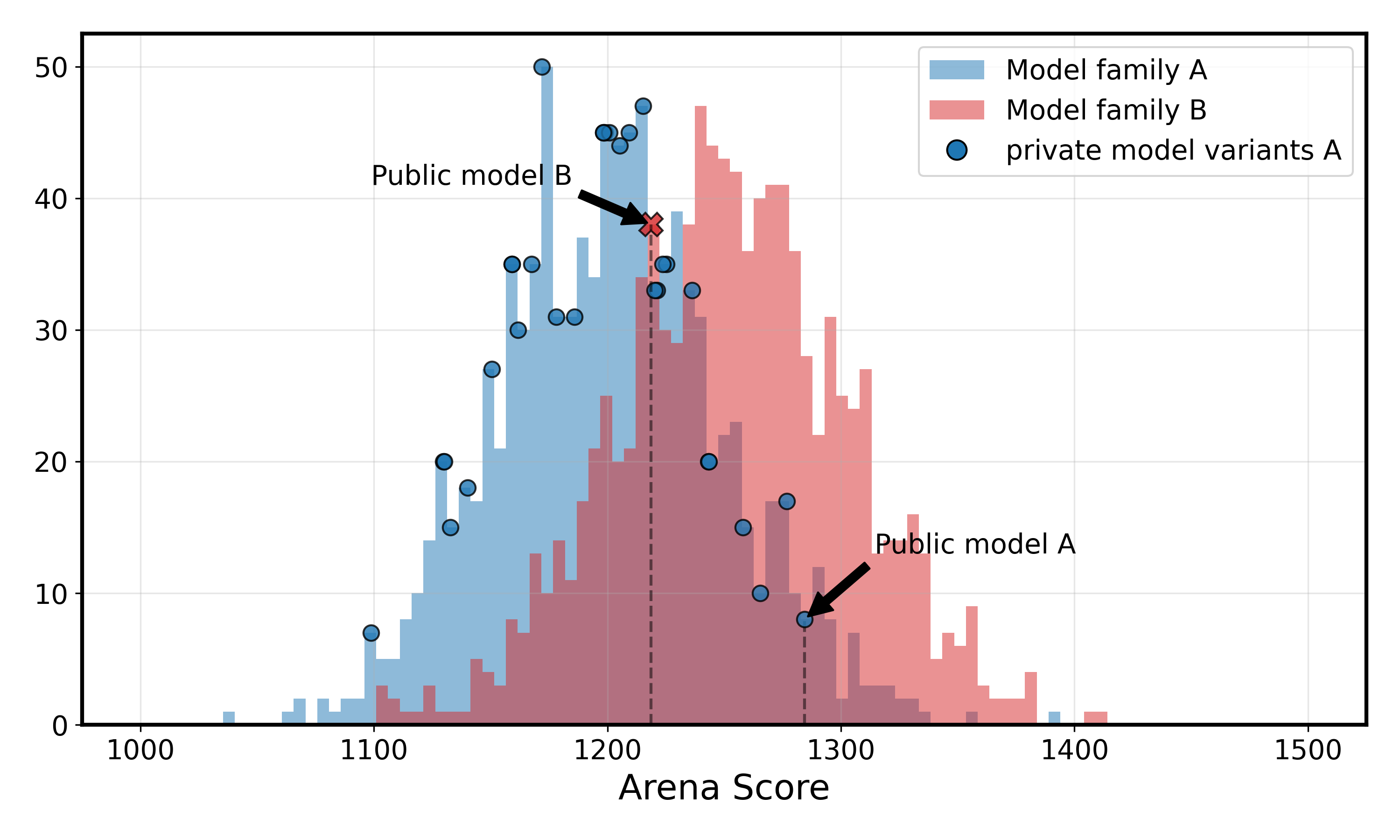}
    \caption{\textbf{Simulated impact of best-of-N submission strategies on Arena leaderboard rankings.} \textcolor{tabblue}{\textbf{Model family A}} has a lower average Arena Score than \textcolor{tabred}{\textbf{Model family B}}, yet by submitting multiple private variants and selecting the best-performing one, it can surpass the sole public submission from \textcolor{tabred}{\textbf{Model family B}}. This simulation shows how providers can gain leaderboard advantage by evaluating several variants privately and publishing only the top-scoring one.}
    \label{fig-submission-strategy}
\end{figure}

\textbf{Role of the number of private variants.} In the \cba, we observe that while only a handful of providers know about private testing, there is also an asymmetry in the \textit{number of private models} tested. At an extreme, Meta tested 27 private variants in the lead up to Llama 4 launch. To investigate how the number of private tests affects model selection outcomes, we simulated the expected lift in Arena Score as the number of private variants increased from 0 to 50. This involved generating sample battles and estimating Arena Scores for the private variants using the Bradley-Terry model. We also employed bootstrapping to estimate the confidence intervals for the corresponding Arena Scores.
\Cref{fig:number-of-variants} presents the results, plotting the number of private variants tested (x-axis) against the average of the maximum discovered Arena Scores across simulations (y-axis).
As the number of privately tested models increases, the expected maximum Arena Score discovered rises accordingly. In our simulation, testing just 20 variants yields a notable increase of approximately 50 points in the maximum score identified, compared to the baseline scenario with no private testing. We provide more details about this simulation in Appendix \ref{app:private-variant-sim}.

\textbf{Asymmetries in which providers have access to private testing.} We observe in practice that only a few preferred providers were allowed to test many variants and handpick the best result. As we show in \Cref{fig-submission-strategy}, restricting private testing to a subset of providers can lead to scenarios where a weaker family of models (Family A) enabled with private testing can outperform a stronger family of models (Family B), which is restricted to a single submission. Although the performance of both model families is in a similar range, Family A's models have a lower average Arena score across all models compared to Family B's. In contrast to model provider B, who is unaware of the best-of-N strategy, model provider A evaluates multiple models on the \cba distribution and selects the best-performing model. This allows provider A to evaluate multiple models on the \cba distribution and select the best-performing model, leveraging the tail of the distribution, to achieve a higher leaderboard ranking.  As a result, despite having a generally stronger model pool, Family B ranks lower than Family A on the leaderboard.

Having to pick from multiple final ``candidate'' models is a very common scenario during the development of LLMs. At the end of the development process of a new model, a model provider typically ends up with multiple variants, each excelling in different tasks due to variations in post-training strategies or hyperparameter settings. They often perform well on different tasks, while still being within similar overall performance bands. Selecting a final ``official'' model often involves a compromise across a broad range of evaluation sets. A strong signal---such as a clear lead on an arena-style leaderboard---can significantly influence this decision, tipping the preference toward specific variants that perform best under the given evaluation framework. Hence, with an ``informed selection'' strategy for model variants, a model provider can improve its ranking and reach the top of the leaderboard compared to when it makes an ``unguided selection'' to choose the variant that can be released on the public leaderboard. 

\begin{figure}[htb]
    \centering
    \includegraphics[width=1.0\linewidth]{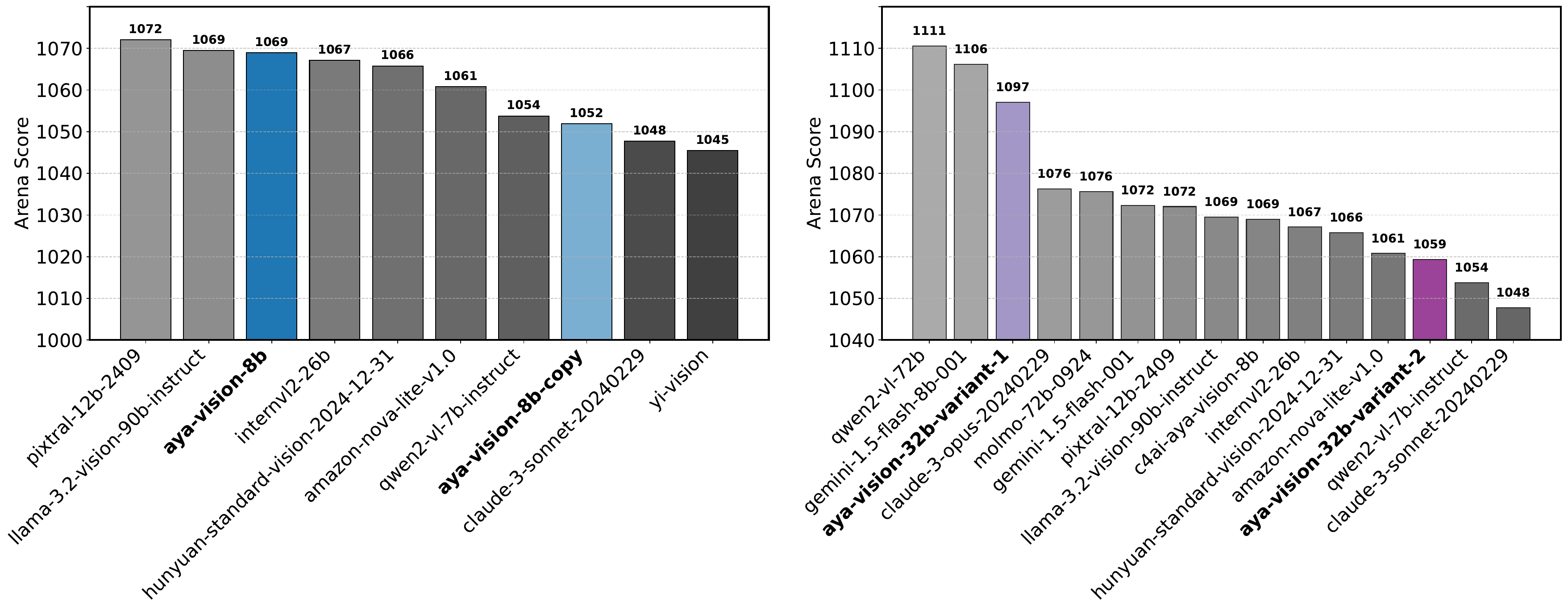}
    \caption{ \textbf{Allowing retraction of scores allows providers to skew Arena scores upwards.} We run a real-world experiment to measure the benefits of private testing. We show that it is possible to increase Arena scores even in the most conservative case of identical checkpoints, and further amplify the difference by strategically testing different checkpoints.
    \textbf{Left: Identical Checkpoints.} Arena Scores for \texttt{Aya-Vision-8B} yield different Arena scores (1069 vs. 1052). 
    \textbf{Right: Strategically Selected Checkpoints.} Arena Scores for two different variants of \texttt{Aya-Vision-32B}, which were both considered high-performing final round candidates according to internal metrics. We observe large differences in final scores (1097 vs. 1059) for the two different model variants.
    }
    \label{fig:modelcomparisons}
\end{figure}

\subsection{Real-world \cba Experiment} 
\label{sec:real-world-exp}

Our simulation previously demonstrated that when model submissions are not standardized across model providers, those who submit a larger number of private model variants gain a systematic advantage, creating an uneven playing field for ranking.
Since we do not have access to the final scores of private model variants observed during our Arena scrape, we design a real-world experiment to complement and validate our simulation findings. Specifically, we model two scenarios of best-of-N testing to assess the impact of submitting multiple model variants.

\textbf{Lower bound estimate of benefits of multiple submissions.} We model the most conservative case -- measuring the gain from submitting and then only releasing the best Arena score obtained by \textit{two identical checkpoints}. For this purpose, we submitted two identical variants of \texttt{Aya-Vision-8B} to the Arena in March 2025. 
This scenario is a very conservative lower bound of gains, as any difference in Arena score can be attributed to the benefits of multiple submissions rather than differences in model quality. We used only two checkpoints because we wanted to limit the amount of human annotator time we diverted to this exercise. We did not disclose that these were identical checkpoints to \cba. In \Cref{fig:modelcomparisons}, we observe notable differences in final Arena scores between identical checkpoints: 1052 (±21/22) and 1069 (±19/23).
    As a result, 4 models are positioned between our two identical checkpoints.
    This suggests that even with identical variants, it is very feasible to obtain a biased positive advantage.

    \textbf{Realistic estimate of benefits of multiple submissions.} In practice, a provider would likely optimize the differences in the variants they are testing to maximize the signal from the Arena users. To study this scenario, we compare two different variants of \texttt{Aya-Vision-32B}.
    The two checkpoints are variants of the same model, with each showing slightly better performance on different subsets of benchmarks.
    In \Cref{fig:modelcomparisons}, we illustrate the extreme ends of the scores obtained by the two models (1097 (±29/25) vs. 1059 (±18/23)), with 9 models falling in between both variants on the leaderboard.

\section{Results: Impact of Data Access Asymmetries on Arena Scores}
\label{sec-measuring-overfititng}

\begin{tcolorbox}[title=Section Findings,colback=thomblue1!7]
\begin{itemize}[left=0pt]
    \item \textbf{Extreme disparities exist in access to data from \cba.} These disparities stem from difference in the \textbf{number of private variants tested}, \textbf{sampling rate}, and \textbf{silent deprecation}. Multiple private variant testing and score retraction is an undisclosed policy, and we find that sampling rate and silent deprecation do not reflect the publicly stated policy of \cba. All these policies appear to severely disadvantage open-weight and open-source providers. For example, OpenAI, Google, Meta, and Anthropic collectively account for 62.8\% of the Arena data, which is 68 times more than the combined share of top academic and non-profit labs like Allen AI, Stanford, Princeton, and UC Berkeley.
    \item \textbf{Access to \cba data has an outsized impact on performance.} We fine-tune models under a fixed training budget, varying the proportion of arena data relative to non-Arena instruction-following samples, and observe a substantial increase in performance. We show that incorporating a moderate amount of arena data greatly improves performance on ArenaHard \citep{arenahard2024} with estimated relative performance gains of 112\%, while showing limited benefits for other tasks of interest. These results represent a conservative estimate of gains given the limited access to arena-style data. Larger labs and companies likely have access to significantly more such data, which, if used, could yield even greater gains. 
\end{itemize}
\end{tcolorbox}

\subsection{Disparity in access to \cba Data}
\label{sec-disparity-in-data-access}

Prompts from a large and diverse user base, such as those from \cba users, serve as a valuable signal for modeling user interests and preferences.
This data is often accessible to model providers through API calls originating from \cba battles.
Multiple factors impact the amount of the data obtained by a given provider, some of which are determined by \cba versus others which are within the control of the providers:

\begin{enumerate}
    \item\textbf{Number of private variants being tested on the arena:} As shown in \Cref{fig-private-testing-providers}, some providers deploy far more private variants, which can significantly increase the volume of data collected. 
    We note that even with our experiment of launching multiple model variants, we increased the amount of prompts collected from 5.9\% with 1 variant to 19.4\% with 3 variants. 
    Based on findings from \Cref{fig-private-testing-providers}, the number of variants submitted is not uniform across all providers, and some providers may increase variants to further amplify the volume of data collected. This is of particular concern given \cba is a community-driven leaderboard, however, the main beneficiaries of this free human feedback appear to be commercial entities who are frequently preferred for private testing. 
    \item\textbf{Sampling rate applied to provider models:} We define model sampling rate as the percentage of daily battles a model participates in. The maximum sampling rate for a provider is the highest rate achieved by any of its models on any given day. We determine the maximum sampling rate of providers based on \texttt{\lmarenascraped}, which is limited to the specific period during which we collected this data (January 2025 to March 2025). As shown in \Cref{fig-private-testing-max}, sampling rates vary significantly across providers. These rates are determined by \cba, but are often entirely inconsistent with the stated policy and prior proposals by the organizers to automatically set sampling based upon which models have not converged in score \citep{chiang2024chatbot}. 
    At the extreme, Google and OpenAI reach a maximum daily sampling rate of 34\%, while Reka registers the lowest at 3.3\%. Other providers with relatively high sampling rates include xAI (22.0\%) and Meta (17.9\%), highlighting substantial variation across the board. We provide additional details about how sampling rates were determined for each provider in \Cref{app:sampling_rates}. 
    \item\textbf{Number of models publicly hosted on the arena:} A model only receives traffic if it is live on the arena. However, \cba frequently deprecates models. There are several reasons to deprecate models in a benchmark. \cba may be forced to deprecate a model when a provider no longer supports it via its API. 
    They also have policies for deprecating models under certain conditions\footref{policy}: \emph{Models may be retired after 3000 votes ``if there are two more recent models in the same series and/or if there are more than 3 providers that offer models cheaper or same price and strictly better (according to overall Arena score)''.} We note that the logic of this policy is difficult to audit in practice because many models are hosted for free on the \cba, and the use of the ``or'' condition means it is not clear what criteria (price or quality) applies to decisions. We observe that many models are also silently deprecated, which means their sampling rate is reduced to nearly 0\% without notification, even though some of them do not meet the stated criteria of the deprecation policy. We identify 205 models that have been silently deprecated, a number that substantially exceeds the 47 models officially marked as deprecated by \cba. For a more detailed analysis, see \Cref{app:deprecation-impact}.
     \item  \textbf{API Support for Models on the Arena:}
    Developers who deploy a model and enable \cba testing via an API have a default advantage. This allows providers to collect 100\% of the test prompts submitted on the Arena. In contrast,  providers whose models are hosted by a third party are often limited to publicly accessible data or must request access to only 20\% of the data (including prompts and human preferences) involving their models from \cba, as per their policy\footref{policy}.
\end{enumerate}

We observe that the collective impact of these factors appears to be advantageous to a small handful of providers and is often inconsistent with the stated policy. For example, in \Cref{fig:public_private_data}, we show that the combined share of OpenAI, Google, Meta, and Anthropic alone is 62.8\% of the arena data, which is 68 times more than the share of top academic and non-profit labs including Allen AI, Stanford, Princeton, and UC Berkeley.
These findings add to prior works that consistently show better corporate access to AI training data across the ecosystem \citep{longpre2024bridging, longpre2024consent}.

We note that the prompt samples available to each provider may not be mutually exclusive, as each battle on the Arena involves two models (from the same or different model providers), allowing the same prompt to be sent to two different providers.
Details about statistics in \Cref{fig:public_private_data} are available in \Cref{app:public-private-data}.

\begin{figure}[t]
     \centering
     \begin{subfigure}[t]{0.49\textwidth}
        \centering
        \includegraphics[width=1\textwidth]{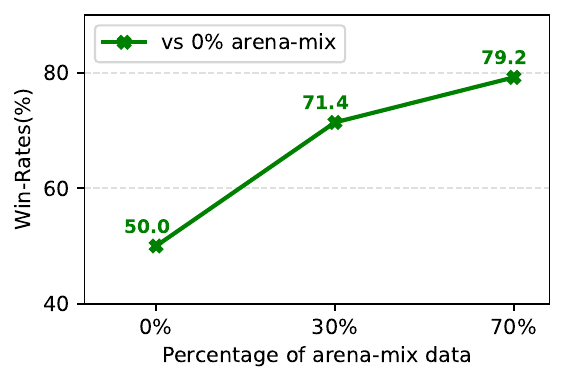}
        \label{fig:vs-base}
    \end{subfigure}
    \begin{subfigure}[t]{0.49\textwidth}
        \centering
        \includegraphics[width=1\textwidth]{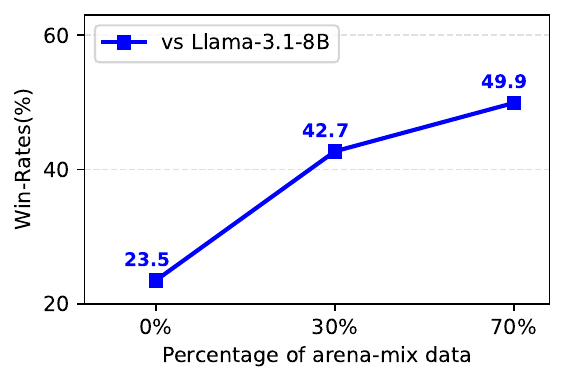}
        \label{fig:vs-llama}
    \end{subfigure}%
    \caption{\textbf{Use of \cba dataset significantly improves win-rates on ArenaHard}. Increasing the amount of arena data in a supervised fine-tuning mixture (0\%~$\rightarrow~$30\%~$\rightarrow$~70\%) significantly improves win-rates of the resulting model against both the model variant where no \cba data is used and also \texttt{Llama-3.1-8B}. The win-rates are measured on ArenaHard \citep{li2024crowdsourced}, which has a high correlation of 98.6\% to \cba. }

    \label{fig:overfit-sft-exp}
\end{figure}

\subsection{Risk of Potential Overfitting} 
\label{sec:risk_of_potential_overfitting}

One of the questions we want to answer is \textit{``What are the implications of data asymmetries? Does having access to more data enable overfitting to the Arena?''}.  Overfitting refers to a phenomenon in which a model learns not only the generalizable patterns in its training data but also the specific noise or characteristics, resulting in strong performance on familiar or seen examples but degraded performance on unseen and out-of-distribution inputs. 

Overfitting is a particularly salient challenge in static evaluation settings, where fixed test sets are vulnerable to overfitting through repeated exposure, data contamination, or intentional tuning \citep{deng2023investigating, golchin2023time, roberts2023cutoff, dong2024generalization,singh2024evaluationdatacontaminationllms}. In contrast,  \cba has been widely adopted in part because it allows human users to ask any questions they want, which creates a non-static test set \citep{donyehiya2024futureopenhumanfeedback}. This means that, at least in principle, the test set is potentially harder to overfit to. However, this assumes that the distribution is constantly changing. 

To understand whether this is the case with data from \cba, we do an exhaustive analysis and observe that the true picture on \cba is more complex. We observe two phenomena: (1) The characteristics of prompts do shift notably over time, and (2) a non-trivial portion of prompts in one month are either exact duplicates or near-duplicates of prompts from previous months. 
This means that some test prompts are very similar to (or even the same as) prompts seen previously.
Together, these suggest that having access to a large sample of last month's data will enable a developer to perform significantly better on next month's test set.

\begin{figure}[htb]
    \centering
    \includegraphics[width=1.0\linewidth]
    {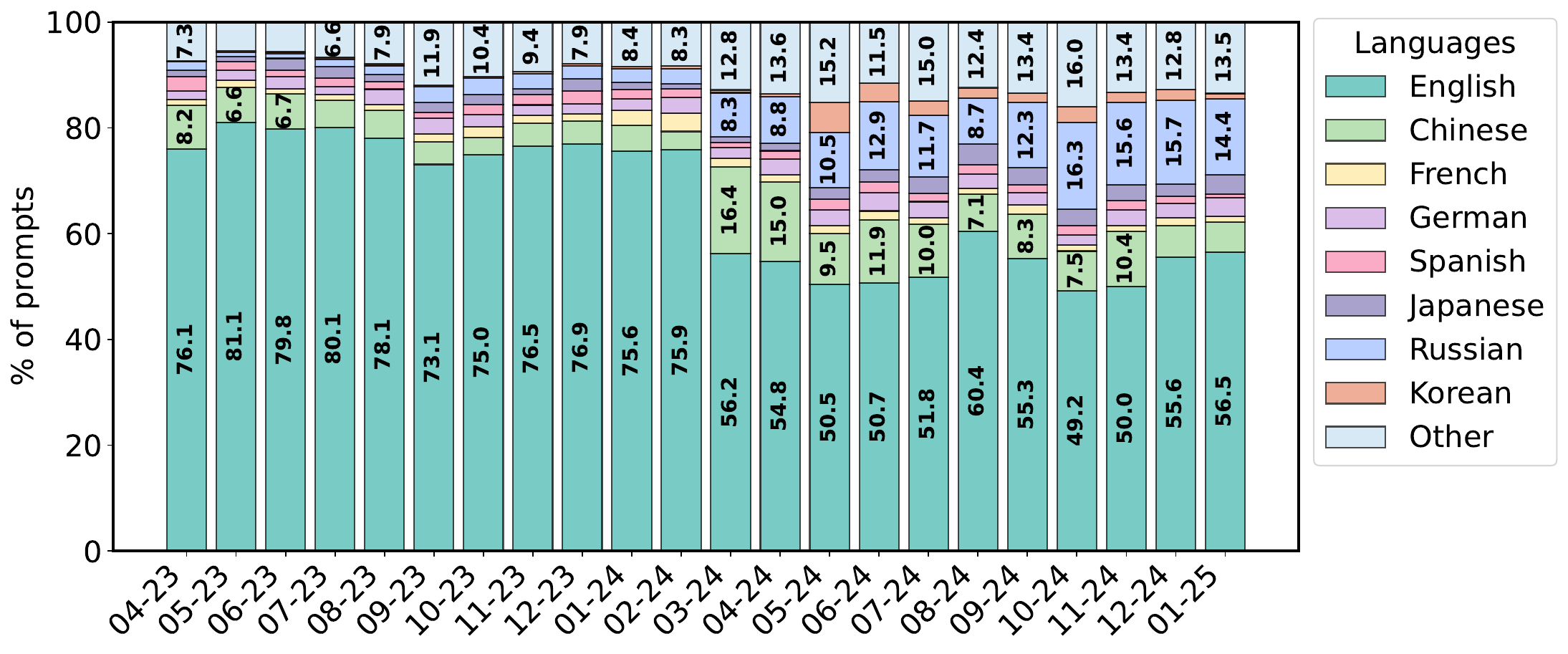}
    \caption{\textbf{Language distribution of prompts submitted to \cba from April 2023 to January 2025.} Based on the \texttt{\lmarenahistorical} dataset, this figure tracks the monthly share of prompt languages. Only languages with dedicated \cba leaderboards are shown individually; the rest are grouped under ``Other''. A clear shift is observed: English prompt share dropped from over 80\% to nearly 50\%, while usage of Chinese, Russian, and Korean prompts increased significantly. 
    }
    \label{fig-language-distribution}
\end{figure}

\textbf{1) Long-term distribution shifts.} 
Prior work clearly demonstrates how temporal distribution shifts affect performance \citep{luu2021time, longpre2024pretrainer}. On \cba,
notable shifts have been observed in the distribution of prompts evaluated over longer periods, with a consistent increase in the proportion of prompts from more complex categories, such as mathematics, coding, and multi-turn conversations\footnote{\url{https://blog.lmarena.ai/blog/2024/arena-category/}}. %For example, we note that users increasingly test mathematical ability, rising from 18\% in June 2023 to over 27\% by October 2024. 
We also perform our own analysis of the change in language distribution in the Arena based on the ``language`` tag available as part of \texttt{\lmarenahistorical} dataset.  For example, in \Cref{fig-language-distribution}, we observe that the proportion of languages outside of English has varied over time. For instance, the share of Russian prompts increased from 1\% in April 2023 to 8.8\% in April 2024, and further to 15.7\% by December 2024. Chinese prompts more than doubled from 5-7\% in 2023 to 16.4\% in March 2024, coinciding with the introduction of the Chinese leaderboard on \cba, before dropping back to 6.2\% in January 2025. Overall, the number of multilingual prompts on the Arena has grown by 20\% over 1.5 years, from 23.9\% in April 2023 to 43.5\% in January 2025. This indicates increased language diversity in submitted prompts.

\begin{figure}[htb!]
    \centering
    \includegraphics[width=0.8\textwidth]{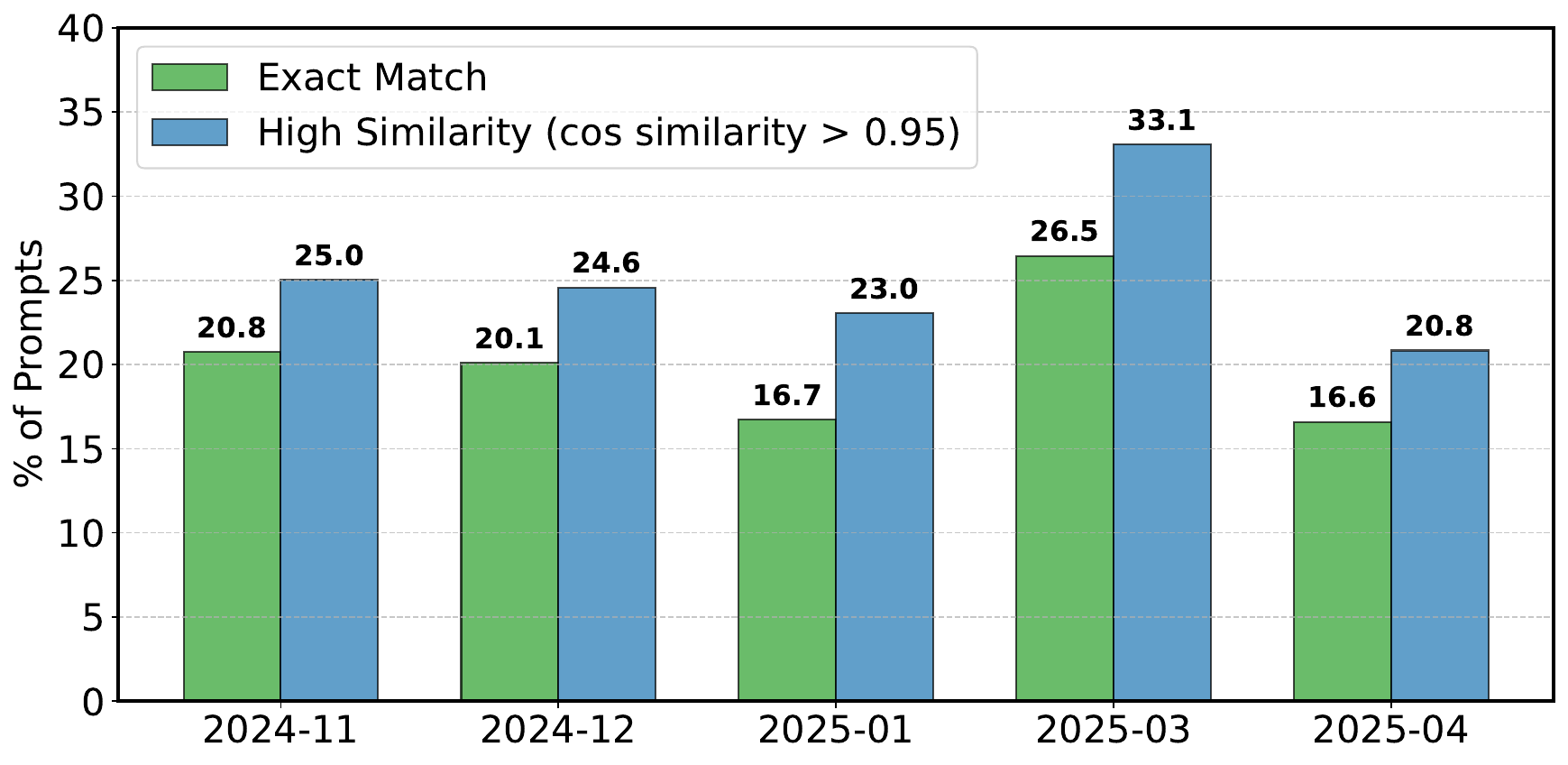}
    \caption{\textbf{Monthly prompt duplication rates.} Prompts are from November 2024 to April 2025, excluding February 2025 due to insufficient data. Duplication is measured using two similarity metrics: \textit{Exact Match} and \textit{High Similarity} (cosine similarity of text embedding $> 0.95$). For simplicity, this analysis is limited to single-turn conversations. The chart presents the percentage of battles in which duplicate or near-duplicate prompts were detected each month.}
    \label{fig:duplicate_prompt}
\end{figure}

\textbf{2) Prompt redundancy and duplication.}  In parallel, we observe high levels of prompt duplication. We analyze a proportion of raw API calls we receive from \cba between November 2024 and April 2025 (197,217 single-turn conversations). We switch to this source given that the proprietary data \cba releases are already de-duplicated, and so won't capture the extent of similar or overlapping queries. Between November 2024 and April 2025, de-duplication resulted in an average prompt loss of 20.14\%, peaking at 26.5\% in March 2025 (\Cref{fig:duplicate_prompt}). While prompt distribution changes over time, prompts in one month often serve as a proxy for the next. For instance, 7.3\% of prompts from December 2024 appear again in the exact form in January 2025. 
If we relax the condition and consider high semantic similarity of prompt embeddings (using the \texttt{embed-multilingual-v3.0} model\footnote{\url{https://huggingface.co/Cohere/Cohere-embed-multilingual-v3.0}}), the same cross-month duplication rate increases to 9\%. Detailed cross-month duplication statistics can be found in \Cref{app:prompt-duplication-headmap}. Both trends above suggest that \textbf{1)} sustained access to up-to-date prompt data and \textbf{2)} the volume of sampled prompts in a given month offer a significant competitive advantage in predicting performance in subsequent months.

\textbf{Uniqueness of Arena Data.} One reason providers may be motivated to explicitly optimize for \cba distribution is if it differs substantially from other evaluation settings that providers may care about. There is sufficient signal to suggest this is the case. There is a context length limit of 12000 characters on \cba prompts, which excludes certain types of longer or more complex inputs from being evaluated\footnote{\url{https://github.com/lm-sys/FastChat/blob/main/fastchat/constants.py}}, and can result in a selection bias of what is asked. 
The user base of the Arena leans towards developers, which could result in the over-indexing of puzzles, math problems, and questions such as \textit{How many r's are there in strawberry?}.\footnote{\url{https://techcrunch.com/2024/09/05/the-ai-industry-is-obsessed-with-chatbot-arena-but-it-might-not-be-the-best-benchmark/}} 
For example, in a released dataset from Arena \citep{zheng2024judging} with 33k samples, no questions are referencing \textit{Chaucer} while dozens of questions are about \textit{Star Trek}, highlighting the uneven distribution of topics in this test set\footnote{\url{https://www.quantable.com/analytics/elos-and-benchmarking-llms/}}.
For a global technology provider, real-world commercial applications may differ significantly from this distribution.

\textbf{Experimental Setup:} 
To estimate the potential for overfitting to \cba using data from a similar distribution, we fine-tuned a language model with identical training setups that only differ in the composition of arena data used in the mixture. We construct our training dataset by sampling at different rations from two different data pools: 1) \texttt{arena-mix}, which consists of samples from Arena battles, and 2) \texttt{other-sft-mix}, a proprietary dataset which includes a variety of supervised fine-tuning datasets focusing on instruction following, multilingual tasks, math, and code. 
We fine-tuned a 7B base model that is used for the Cohere Command family \citep{cohere2025commandaenterprisereadylarge} in these experiments. We refer to the three different training mixes as \texttt{0\_arena}, \texttt{30\_arena}, \texttt{70\_arena} which have 0, 30\%, 70\% of the training dataset sampled from the \texttt{arena-mix} respectively. For all variants, the remainder percentage is sampled from \texttt{other-sft-mix}. All three models are fine-tuned for 1.3K steps using a batch size of 128. We note that our goal here is not to produce a state-of-the-art model but rather to estimate a lower bound for the performance gains that could be expected from asymmetries in access to Arena data. Hence, we do not optimize with ablations the correct weighting or data or conduct any hyperparameter sweeps. 

\textbf{Evaluation method:} Our goal is to measure the lift provided from the training on Arena data on the Arena held-out set. To do so, we measure relative improvements on the 500 English LMArena ArenaHard prompts, an in-distribution test set published by \cba that demonstrates exceptionally high correlation (98.6\%) with human preference rankings from \cba battles \citep{li2024crowdsourced}.
This dataset consists of challenging user-submitted prompts from Arena that have been carefully curated and evaluated against several criteria, including, but not limited to, domain knowledge, complexity, and problem-solving. This dataset is widely used to gauge expected performance on the Arena.
To measure improvements, we simulate human preferences using LM-as-a-judge, which allows us to measure a reasonable estimate of the gains in a controlled setting. Various works have shown that this is correlated with human preferences \citep{dubois2023alpacafarm, rafailov2023direct, kim2023prometheus}. We compare against \texttt{Llama-3.1-8B-Instruct} \citep{grattafiori2024llama} and measure win-rates using \texttt{gpt-4o-2024-11-20} as our judge model\footnote{\url{https://platform.openai.com/docs/models/gpt-4o}}.

\textbf{Results:} From \Cref{fig:overfit-sft-exp}, we observe that as the amount of \texttt{arena-mix} data increases, the models improve in their evaluation on the ArenaHard prompts. Variant \texttt{0\_arena} scores a win-rate of 23.5\%, \texttt{30\_arena} scores a win-rate of 42.7\%, and variant \texttt{70\_arena}, which is trained on the most \texttt{arena-mix}, data scores a win-rate of 49.9\% against \texttt{Llama-3.1-8B-Instruct}. 
A win-rate of 50\% indicates that the compared models (Llama 3.1 8B Instruct and \texttt{70\_arena}) are on par on the ArenaHard test set. 
The relative gains in win-rates are 81.7\% for the \texttt{30\_arena} and 112.3\% for the \texttt{70\_arena} variant.
The gains we observe are striking in part because we do not heavily optimize these variants (we trained 3 variants in total and did no hyperparameter optimization or tuning of training steps). 
To assess whether performance gains generalize beyond the Arena benchmark, we evaluated the fine-tuned models on the out-of-distribution benchmark MMLU \citep{hendrycks2020measuring}. The results reveal a clear divergence in performance trends (Table~\ref{tab:overfitting_mmlu}): while increasing the proportion of \cba data within a fixed training budget yielded consistent improvements on the Arena test set, MMLU performance slightly declined from 66.5\% (\texttt{0\_arena}) to 64.4\% (\texttt{30\_arena}) and 65.9\% (\texttt{70\_arena}). 
This suggests that gains from \cba data are highly specific and do not translate to broader generalization, raising important questions about whether leaderboard improvements reflect meaningful progress or simply overfitting to a narrow evaluation distribution.

We also note that often it is not necessary to train explicitly on the data to gain from data access. For example, providers may use the composition of the data to make decisions about weighting different data sources \citep{ustun-etal-2024-aya}, or may use a small subset of the data to create high-quality synthetic data that is close to the original distribution \citep{odumakinde2024multilingual, shimabucoro2024llm, dang2024rlhf, aakanksha2024multilingual}.
Given the stakes behind ranking highly on \cba, it is likely that the data at a provider's disposal from \cba is being actively leveraged by different providers to gain an advantage. 

\section{Results: Impact of Model Deprecation on Arena Scores}
\label{sec:battle-connecitivty}

\begin{tcolorbox}[title=Section Findings,colback=thomblue1!7]
\begin{itemize}[left=0pt]
    \item  \textbf{Model deprecations under a changing task distribution lead to unreliable rankings.} We show that shifts in the task distribution highlight the non-stationary nature of the evaluation environment and have important implications for the stability and fairness of model rankings over time. In evolving task distributions, premature model removal introduces inconsistencies, breaking the BT model's transitivity assumption and distorting rankings.
    \item \textbf{Model deprecation may result in disconnected subgraphs, which violates BT reliability.}
    The BT model assumes a sufficiently connected comparison graph to produce globally consistent rankings. When models are deprecated unevenly or when sampling strategies fail to ensure robust overlap in comparisons, the resulting history matrix can become fragmented. Through simulated experiments, we show that this can produce fragmented clusters, and as a result, the global rankings become unreliable. 
\end{itemize}
\end{tcolorbox}

According to the backend codebase of \cba\footnote{\url{http://github.com/lm-sys/FastChat/blob/0e6d3e4beaab66f4d3f93db72541a4abab8af28d/fastchat/serve/monitor/monitor_md.py\#L7} \label{deprecation_link}}, 47 models are publicly listed as deprecated.
In addition, as discussed in \Cref{sec-disparity-in-data-access}, 205 models on \cba have been \textit{silently deprecated} by reducing their active sampling rates to near zero (see \Cref{fig-silent-deprecated}). We observe that model deprecation disproportionately affects different types of models. Specifically, 87.8\% of open weights and 89\% of open-source models have been deprecated, in contrast to a lower deprecation rate of 80\% for proprietary models.

\begin{figure}[ht!]
    \centering
    \includegraphics[width=0.6\linewidth]{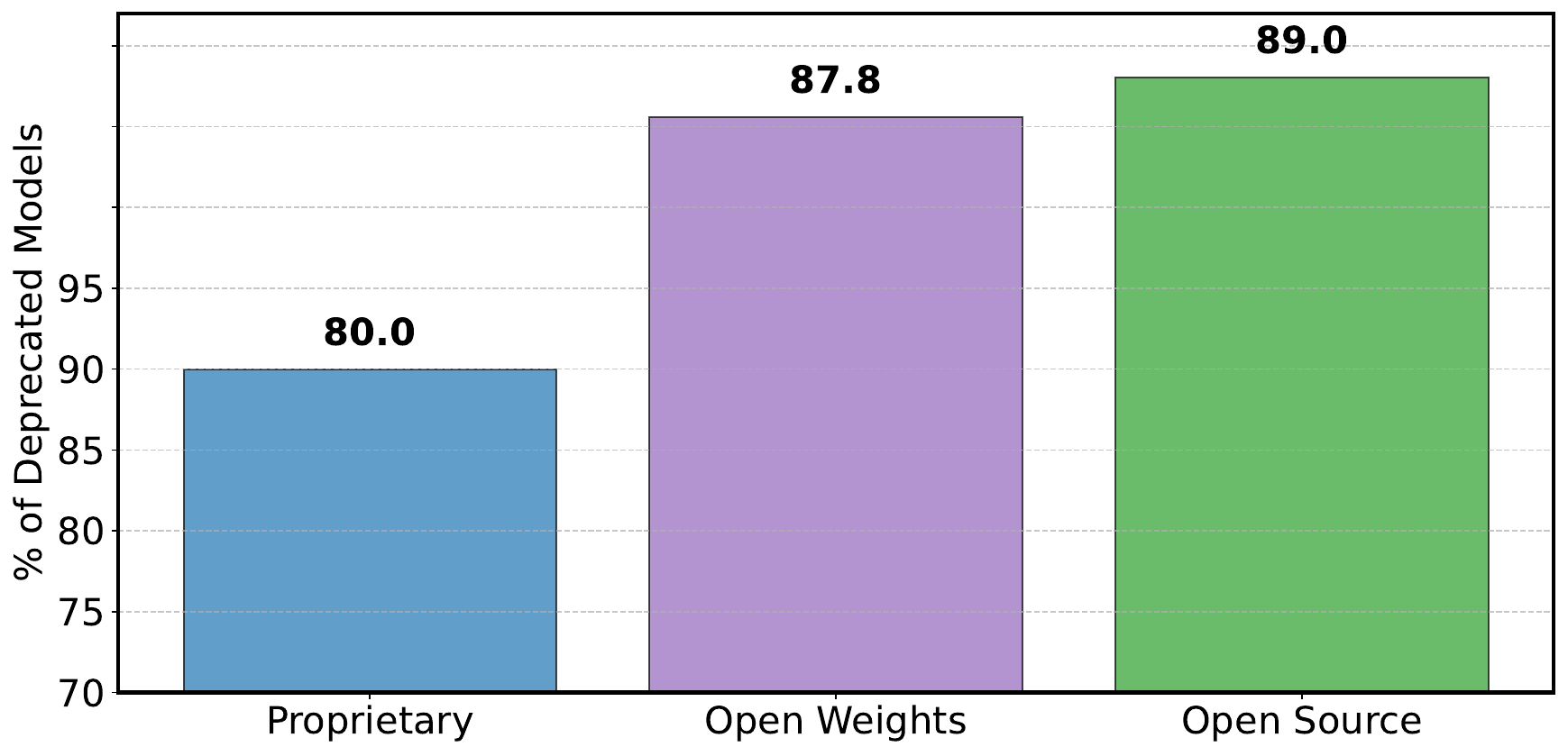}
    \caption{\textbf{Share of proprietary and open models that either officially deprecated or inactive on the arena based on \texttt{\lmarenaleaderboard} during the period March 3rd-April 23rd, 2025}. Overall, open-weight and fully open-source models are more likely to become deprecated or inactive compared to proprietary models.}
    \label{fig-silent-deprecated-license-cat}
\end{figure}

While deprecating older models is necessary to maintain a dynamic leaderboard, especially as new models are introduced regularly, we question how excessive deprecation may compromise the reliability and stability of model rankings.
When models are deprecated, future models that enter the Arena will not have direct comparisons with them. 
However, in principle, the BT model should still handle this reliably because of the transitivity property \citep{boubdir2024elo}. Intuitively, transitivity means that if model A is better than model B, and model B is better than model C, then model A should also be better than model C. Transitivity allows the BT model to infer missing outcomes---if two models share common opponents, their relative ranking can be deduced even without a direct comparison \citep{bradley1952rank}.
Formally, in the BT model each competitor \(i\) is associated with a positive parameter \(\pi_i > 0\), and the probability that model \(i\) beats model \(j\) is given by: 
\[
P(i > j) = \frac{\pi_i}{\pi_i + \pi_j}.
\]
Suppose \(\pi_A > \pi_B\) and \(\pi_B > \pi_C\). Then
\[
P(A > B) = \frac{\pi_A}{\pi_A + \pi_B} > 0.5 \quad \text{and} \quad P(B > C) = \frac{\pi_B}{\pi_B + \pi_C} > 0.5.
\]
Moreover, because \(\pi_A > \pi_B > \pi_C\), we have:
\[
\pi_A + \pi_C < \pi_A + \pi_B \quad \Longrightarrow \quad P(A > C) = \frac{\pi_A}{\pi_A + \pi_C} > \frac{\pi_A}{\pi_A + \pi_B}.
\]

Transitivity is a critical property that allows for inferring rankings using fewer data points. However, transitivity requires several underlying properties to hold. In particular, two of these properties are salient for our study of the impact of deprecation on reliability:
\begin{itemize}
    \item \textbf{Assumption 1: Evaluation conditions remain constant.} Transitivity relies on the assumption that evaluation conditions remain constant---that is, paired comparisons must come from a consistent set of tasks or contexts. For example, consider a setting where Tom wins against Susan in chess, and Susan wins against Ronald in both chess and tennis. We can infer Tom's ranking relative to Ronald in chess, but not in tennis. In practice, \cba's prompt categories and task types evolve over time, and some models become deprecated, meaning they are no longer re-evaluated under the current conditions. We show in \Cref{sec:dist-shift} that when a model's battle history is limited to outdated conditions, its comparisons no longer accurately reflect performance in the new context.
    \item \textbf{Assumption 2: Network of comparisons must be fully interconnected.} Every model must be linked directly or indirectly through pairwise matchups. So in a chess example, we can estimate Tom's performance against Ronald because they have each played at least one competitor in common (in this case, Susan). If Ronald has had no matches with anyone who has played Tom, we can no longer infer reliably what their rankings are relative to each other. Deviations from transitivity or a fragmented comparison graph can compromise the reliability and consistency of the strength estimates. In \Cref{sparse_battle_history}, we show that deprecations may result in these fragmented graphs.
\end{itemize}

In \Cref{sec:dist-shift} and \Cref{sparse_battle_history} that follow, we will interrogate whether \cba fulfills each of these assumptions.

\subsection{Transitivity Under Changing Evaluation Conditions}\label{sec:dist-shift}

As we have shown in \Cref{sec-disparity-in-data-access}, the distribution of \cba is unique since long-term shifts occur in categories and use cases. This distributional shift contrasts with the static environments typically assumed in Elo and Bradley-Terry systems, such as chess, where the rules and game format remain fixed, ensuring a consistent set of evaluation conditions. If all models were continuously sampled across all points in time, the BT model would likely remain robust because every model would be evaluated on the evolving distribution of tasks. However, as shown in \Cref{fig-silent-deprecated-license-cat}, many models are deprecated over time, and their scores stop getting updated. 

\textbf{Experimental Setup:} To investigate how model deprecations under a changing task distribution can impact model rankings, we simulate BT rankings of models under evolving evaluation conditions.
We initialize four models---A, B, C, and D---each with distinct performance profiles across two task types, \textit{Task-1} and \textit{Task-2}. These tasks represent different prompt categories, and each model's relative strength is defined through task-specific win probabilities. For example, model B has a 90\% chance of defeating model D on Task-1 but only a 20\% chance on Task-2, with some pairs also allowing for ties. The task-specific win probabilities for different models are provided in \Cref{tab:simulation_winrates}. These probabilities reflect the models' varying strengths across tasks, mirroring the real-world observation that models excel at different types of prompts.

The simulation is structured into two sequential phases to mimic the evolving task distribution observed on \cba. During the first phase, battles are predominantly drawn from \textit{Task-1}. 
Each of the four models participates in 1000 battles, and the resulting outcomes are used to compute initial rankings.
In the second phase, the battle distribution gradually shifts toward \textit{Task-2}. 
Since model win-rates are task-dependent, battle outcomes change accordingly. We simulate 1000 additional battles in this phase and examine two scenarios to investigate how shifts in prompt distribution and model deprecations jointly influence final rankings.
We compute the BT Scores for all models under both scenarios using the implementation provided by \cba in their official FastChat codebase\footnote{\url{https://github.com/lm-sys/FastChat/blob/main/fastchat/serve/monitor/rating_systems.py} \label{rating_systems}}. These scores are then used to determine the final model ranks.

 \definecolor{rankone}{HTML}{84ADD3}   % Muted tab:blue
\definecolor{ranktwo}{HTML}{FDBE87}   % Muted tab:orange
\definecolor{rankthree}{HTML}{85C88A} % Muted tab:green
\definecolor{rankfour}{HTML}{E88A8A}  % Muted tab:red
\definecolor{rankfive}{HTML}{C3A4DA}  % Muted tab:purple
\definecolor{ranksix}{HTML}{BFA6A1}   % Muted tab:brown
\definecolor{rankseven}{HTML}{F1ADD7} % Muted tab:pink

\newcommand{\rankbox}[2]{%
    \tikz[baseline=(X.base)] 
    \node[draw=black, 
          fill=#1, 
          rounded corners=3pt, 
          inner sep=1pt, 
          minimum width=12pt, 
          minimum height=12pt, 
          text width=12pt, 
          align=center] (X) {\textbf{#2}};%
}

\newcommand{\rankboxwrong}[2]{%
    \tikz[baseline=(X.base)] 
    \node[draw=black, %red, 
          double=red, double distance=0.8pt,
          fill=#1, 
          rounded corners=3pt, 
          inner sep=1pt, 
          minimum width=12pt, 
          minimum height=12pt, 
          text width=12pt, 
          align=center] (X) {\textbf{#2}};%
}

\newcommand{\rankone}{\rankbox{rankone}{1}}  
\newcommand{\rankonewrong}{\rankboxwrong{rankone}{1}}
\newcommand{\ranktwo}{\rankbox{ranktwo}{2}}  
\newcommand{\ranktwowrong}{\rankboxwrong{ranktwo}{2}}
\newcommand{\rankthree}{\rankbox{rankthree}{3}} 
\newcommand{\rankthreewrong}{\rankboxwrong{rankthree}{3}} 
\newcommand{\rankfour}{\rankbox{rankfour}{4}}  
\newcommand{\rankfourwrong}{\rankboxwrong{rankfour}{4}} 
\newcommand{\rankfive}{\rankbox{rankfive}{5}}  
\newcommand{\rankfivewrong}{\rankboxwrong{rankfive}{5}}  
\newcommand{\ranksix}{\rankbox{ranksix}{6}}  
\newcommand{\ranksixwrong}{\rankboxwrong{ranksix}{6}}  
\newcommand{\rankseven}{\rankbox{rankseven}{7}}   
\newcommand{\ranksevenwrong}{\rankboxwrong{rankseven}{7}}   

\begin{figure}[htb!]
    \centering
    \begin{minipage}[c]{0.50\textwidth}
        \centering
        \vspace{0pt}
        \includegraphics[width=1.0\linewidth]{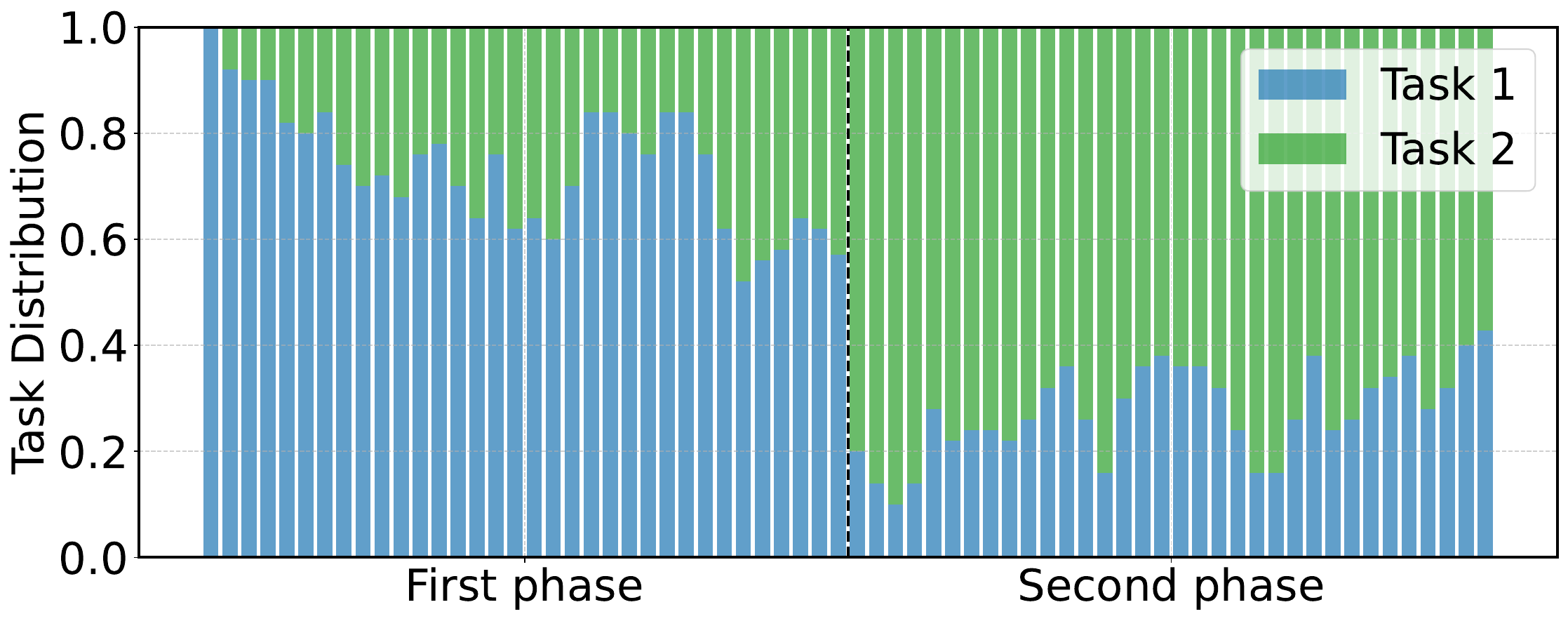}
    \end{minipage}
    \hspace{5mm}
    \begin{minipage}[c]{0.32\textwidth}
        \centering
        \scriptsize  
        \setlength{\tabcolsep}{3pt} 
        \begin{tabular}{ccc}
         \midrule
        Model & \makecell{w/o dep} & \makecell{w/ dep} \\ 
         \midrule
        A  & \rankone & \ranktwowrong \\
        B  & \ranktwo & \rankonewrong \\
        D  & \rankthree & \rankfourwrong \\
        C  & \rankfour & \rankthreewrong \\
        \end{tabular}
        \label{tab:task-dist-ranks}
    \end{minipage}
    \caption{\textbf{Impact of evolving task distributions and model deprecation on model rankings.} 
    \textbf{Left:} Two‐phase task distribution used in the simulation. Phase 1 is \textit{Task-1 heavy}, with most battles based on Task-1; Phase 2 is \textit{Task-2 heavy}, with battles predominantly based on Task-2.
    \textbf{Right:} Model rankings under changing task distributions and deprecation settings. 
    Scenario I only differs from Scenario II in that Model D is deprecated halfway through the battle history (after phase 1).
    This deprecation causes Scenario II to produce a completely different ranking over models as compared to Scenario I.}
    \label{fig-task-distribution}
\end{figure}

\begin{itemize}
    \item \textbf{Scenario I: without deprecation.} We simulate all 2000 battles across both phases, with all four models participating throughout. This represents an ideal scenario where no model is deprecated, and all are evaluated across the evolving task distribution. 
    \item \textbf{Scenario II: with deprecation.} We simulate all 2000 battles across both phases. However, at the end of phase 1, model D is deprecated and does not participate in the second phase.
\end{itemize}

\textbf{Deprecation given changing distribution results in unreliable Arena rankings:} 
As illustrated in \Cref{fig-task-distribution}, our simulation shows that rankings produced by the BT model are highly sensitive to model deprecation, particularly when the prompt distribution changes over time. In the scenario without deprecation, we observe the true rankings given that the BT model remains reliable because it reflects performance across the full history of interactions. However, we observe the rankings differ if we deprecate Model D between stages. While its matchups from Phase 1 still influence BT scores, the absence of updated comparisons causes a skew in the rankings of remaining models. 
Models A and D are ranked lower, and Models B and C are ranked higher than their true performance merits.
Comparing Scenarios I and II in \Cref{fig-task-distribution} reveals a clear divergence, with the scenario involving deprecation yielding unreliable rankings that no longer reflect true relative performance under the evolving task distribution.

This violates core assumptions of the BT model, namely, that rankings reflect transitive and consistently sampled matchups, ultimately compromising the validity of inferred rankings. When models are no longer sampled under current task distributions, historical pairwise comparisons cease to represent present-day performance. This issue is particularly problematic in real-world settings where user prompt distributions shift over time. For instance, a model tuned for multilingual prompts may improve ranking as non-English tasks become more common. However, if deprecated, its BT ranking will likely understate its true performance. Similarly, code-specialized models may benefit from the increasing volume of coding tasks, but only if they continue to participate in evaluation.

\subsection{Sparse Battle History Risks}\label{sparse_battle_history}

In this section, we show that the deprecation policy can lead to a sparse matrix and disconnected comparison graphs, which in turn distort the resulting rankings. As demonstrated by \cite{ford1957solution}, the maximum likelihood estimate does not exist if models can be partitioned into two non-empty subsets without comparisons between them or if all comparisons between the two groups are one-sided (i.e., one group always wins). Therefore, to ensure a unique and finite estimation, the comparison graph must be connected. For any possible partition of models, there must be at least one win going in each direction across the partition. This ensures that no subset of models is entirely isolated in the win/loss structure. The \cba comparison matrix can potentially become disconnected because of the extremely high levels of model removals over time (as discussed in \Cref{fig-silent-deprecated-license-cat}).

\begin{figure}[htb!]
    \centering
    \begin{subfigure}[c]{0.18\textwidth}
        \centering
        \scriptsize  % Reduced font size
        \setlength{\tabcolsep}{3pt}  % Reduced column spacing
        \begin{tabular}{ccc}
         \midrule
         Model & Dense & Disconnected  \\
         \midrule
        A  & \rankone & \rankone \\
        B & \ranktwo & \ranktwo \\
        C  & \rankthree & \rankthree \\
        D  & \rankfour & \rankfivewrong \\
        F  & \rankfive & \rankfourwrong \\
        E & \ranksix & \ranksevenwrong \\
        G  & \rankseven & \ranksixwrong \\
        \end{tabular}
        \label{tab:example}
    \end{subfigure}
    \hfill
    \begin{subfigure}[c]{0.8\textwidth}
        \centering
        \vspace{0pt}
        \includegraphics[width=0.8\linewidth]{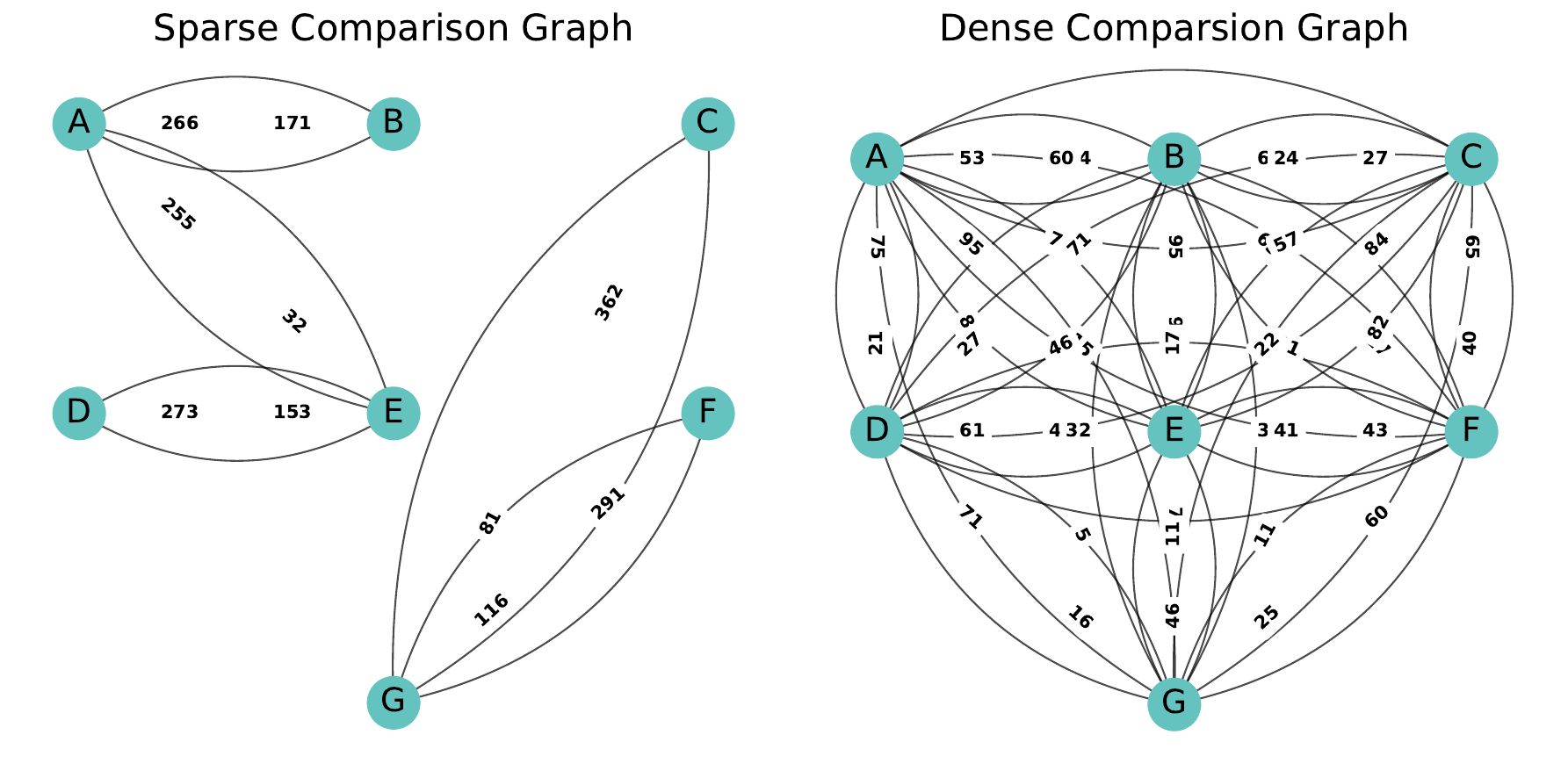}
        \label{fig-sparse-comparsion}
    \end{subfigure}
    \caption{\textbf{Impact of comparison graph sparsity on model rankings.} \textbf{Left:} Rankings for models D, E, F, and G diverge from the gold rankings when the comparison graph is sparse, whereas model rankings fully align with the gold rankings when the comparison graph is dense. \textbf{Right:} Visualization of the comparison graphs in sparse and dense settings. An edge between two models indicates a head-to-head matchup, annotated with the number of wins for each model. For example, in the sparse graph, Model A and Model B played 437 matches, with Model A winning 266 and Model B winning 171.}
    \label{fig:graph-connectivity}
\end{figure}

\textbf{Experiment Setup:} To investigate the impact of sparse comparison graphs on the rankings obtained via the Bradley-Terry model used by \cba we simulate the following scenarios:

\begin{itemize}
    \item \textbf{Scenario I: Dense comparison graph.} All models are allowed to compete against one another—albeit with varying numbers of head-to-head battles—resulting in a well-connected comparison graph in which every node (model) is linked to others via edges representing battle outcomes.
    \item \textbf{Scenario II: Disconnected comparison graph.} We create a disconnected battle history by imposing constraints on which pairs of models are allowed to engage in battles. This allows us to create a sparse battle history where each model ends up playing against a subset of models. 
\end{itemize}

The full comparison graph based on battle histories for both scenarios is shown in \Cref{fig:graph-connectivity}. In both scenarios, a total of 2000 battles are played under the corresponding setting. For a paired match between models A and B, each with respective true skill ratings $r_A$ and $r_B$, the expected scores $E_A$ and $E_B$ can be computed as:

\begin{equation}
E_A = \frac{1}{1 + e^{\alpha (r_B - r_A)}}, \quad E_B = \frac{1}{1 + e^{\alpha (r_A - r_B)}}
\end{equation}

The expected scores $E_A$ and $E_B$ are used to predict the winner of the battle. For simplicity, we exclude the possibility of ties in this experiment. We assign the following true skill ratings to the models: 1450 (Model A), 1390 (Model B), 1250 (Model C), 1200 (Model D), 1101 (Model E), 1150 (Model F), and 1000 (Model G). These ratings are used to calculate the expected scores and match outcomes. Finally, we compute BT Scores for all models under both scenarios using the official implementation, followed by \cba\footref{rating_systems}. This is then used to determine the ranks for each model corresponding to both scenarios.

\textbf{Sparse or disconnected graphs lead to unreliable rankings}: \Cref{fig:graph-connectivity} illustrates the model rankings with sparse and dense battle history graphs. 
We observe that the rankings derived from the dense comparison graph align closely with the models' rankings according to their true skill ratings. 
In contrast, a sparse or disconnected comparison graph results in an inaccurate estimation of the models' skills.
These results highlight that reliable rankings under the BT model require a connected comparison graph. While some level of model removal is inevitable (for example, models are no longer hosted on an API), preserving connectivity means ensuring that comparisons remain sufficiently distributed across active models and that transitions in and out of the pool do not isolate subsets of models from the broader comparison structure.

\section{Recommendations and Guidelines for Improving Leaderboards}
\label{sec:recommendations}

\begin{quote}
    \textit{When my information changes, I alter my conclusions. What do you do, sir?} \textbf{--- John Maynard Keynes}
\end{quote}

We include below recommendations that we believe are urgent and critical. However, we also believe these recommendations are very achievable and will restore scientific credibility and trust within the wider research ecosystem:

\textbf{Prohibit score retraction after submission.} Currently, providers who have engaged in private testing are allowed to retract submissions and only submit the best variant to the public leaderboard. As we have shown in \Cref{sec-disparity-in-data-access}, this can lead to overfitting and obscures meaningful progress as it makes it difficult to distinguish between models that have legitimately improved versus those that have exploited statistical shortcuts \citep{ying_2019}. 
We urge \cba to prohibit retraction after submission, ensuring all tested variants' scores are permanently visible on the leaderboard. Providers should also disclose the total number of private variants tested prior to public launch, including historical submissions, to contextualize their results.

\textbf{Establish transparent limits on the number of private variants per provider.} 
As illustrated in \Cref{fig-private-testing-providers}, private testing volume varies widely across providers, creating unfair advantages. To curb overfitting and level the playing field, \cba should enforce a strict cap of private variants per provider for any given model launch. This should be enforced at a provider level, and not per model type and size as that is impossible to audit with API hosting. This strict limit should be disclosed to all providers (proprietary, open-weights, open-source) and to the wider \cba community. This restriction would discourage excessive undisclosed testing while still allowing limited iteration.

\textbf{Establish clear and auditable model deprecation criteria.} 
The current criteria, which states \emph{``Models may be retired after 3000 votes if there are two more recent models in the same series and/or if there are more than 3 providers that offer models cheaper or same price and strictly better (according to overall Arena score) than this model''} is ambiguous and make it impossible to audit the logic in practice.
Key terms like ``same series'' and ``more recent'' lack formal definitions, making it unclear how to determine whether a model is eligible for retirement based on lineage. 
Additionally, the requirement that ``more than 3 providers'' must offer strictly better and cheaper models introduces confusion around what threshold of improvement in Arena Score is considered meaningful. The use of ``and/or'' further complicates interpretation, as it's unclear whether meeting one condition is sufficient or if both must be satisfied. Lastly, using price as a filtering criterion is problematic since it's subject to change, varies across hosting platforms, and is not inherently tied to a model's performance or utility.  We note that many of these models are hosted for free on the Arena, and so there should be clarity about what source of pricing is being used. This lack of precision makes it challenging to apply the rule consistently or verify retirement decisions.

We recommend a stratified approach that retires models proportionally across proprietary, open-weight, and open-source categories to preserve balance and fairness in the evaluation based on two criteria: availability and performance. 
We propose retiring the bottom 30th percentile within each category of open-source, open-weight, and proprietary after rankings converge. This stratified pruning prevents provider-type bias, keeps strong models from underrepresented groups visible, and maintains comparison graph connectivity. It also reduces ranking inconsistencies seen with uneven retirement, as discussed in \Cref{sec:battle-connecitivty}.

\textbf{Improve sampling fairness.}  As shown in \Cref{fig-private-testing-max}, the sampling rates vary greatly by providers, and also disproportionately undersample open-weight and open-source models, creating large asymmetries in data access over time and resulting in unstable Arena scores (\Cref{sec:battle-connecitivty}). This is particularly important given that this is a community-driven voting benchmark, where at present free human feedback is primarily benefiting proprietary models. This avoids disparities over time and a status quo where proprietary models are benefiting more from valuable and freely given human data \citep{10.1145/3411764.3445518}.
In their own work \citep{chiang2024chatbot} (Section 5, Equation 9), the \cba authors introduce an active sampling rule designed to enhance the efficiency and statistical robustness of the leaderboard's evaluation process. This rule selects model pairs based on the expected reduction in the confidence interval of the win-rate estimate, and is formally defined as:
\[
P_t(a) \propto \sqrt{\frac{\Sigma_{t,a,a}}{| \{t : A_t = a \} |}} - \sqrt{\frac{\Sigma_{t,a,a}}{| \{t : A_t = a \} | + 1}}
\]
where \( P_t(a) \) is the sampling probability of model pair \( a \) at time \( t \), and \( \Sigma_{t,a,a} \) is the estimated variance for the win-rate of pair \( a \) at time \( t \).

This formulation effectively prioritizes under-evaluated and high-variance pairs, aligning sampling with the goal of rapidly reducing uncertainty in rankings. While this sampling rule is clearly articulated in the paper, we have not seen evidence of its deployment in the current leaderboard.
We recommend adopting this sampling strategy in practice and providing periodic reporting on its usage. Doing so would align the platform's operations with its methodological innovations, support more balanced and transparent evaluations, and improve confidence in leaderboard dynamics over time.

\textbf{Provide public transparency into all tested models, deprecations, and sampling rates.}
Most of these findings were only possible through access to private model testing, or crawling \cba over a period of time.
Providing transparency into the full suite of models that were tested, deprecated, and how often they were sampled against which other models, would enable the oversight and trust in the benchmark that \cba affirms in their policies.
This transparency could be provided on a rolling basis (e.g., every quarter, for the prior quarter of model battles). It would enable the community to help in the process of continuing to improve this community benchmark.
For instance, the backend codebase of \cba -- FastChat -- publicly lists deprecated models on GitHub\footref{deprecation_link}, where 47 models are explicitly marked as deprecated. However, four times that number have been silently deprecated without warning. We recommend \cba expand the definition of ``deprecated'' to include models that are no longer being regularly sampled from, and list these deprecated models on their website to make it transparent for everyone which models are no longer active on \cba. 

\section{Limitations}

\textbf{We do not have insight into \cba's raw data:} A subset of the data sources utilized for this study have undergone pre-processing by \cba, which as stated by \cba often involves de-duplication, removal of battles corresponding to suspicious voting patterns, etc \citep{chiang2024chatbot}. Through this work, we also establish that private battles are removed from the datasets released by \cba. Although we also have proprietary API data for the models we test on the Arena, it only reveals a subset of battles. Without access to original and comprehensive raw data, it is hard to investigate patterns related to adversarial voting where users intentionally submit votes to manipulate rankings or undermine the system. Various previous works have shown that adversarial voting is a critical concern for the reliability of any crowd-sourced evaluation platform like \cba \citep{huang2025exploring} \citep{min2025improving}. We do not explore this in this work, but see more investigation here as an important topic for future work.

\textbf{Our scraped data snapshot only covers a limited period:} Our \texttt{\lmarenascraped} was the only way to identify private variants being tested by various providers. However, it covers a limited time period from January--March, 2025. This time frame coincided with Meta's launch of Llama 4, and so we find them to be the provider with the highest number of private variants in our analysis. We believe we might be underestimating the counts for providers having fewer model launches during this period.

\textbf{Our training experiments likely underestimate the potential to overfit:} Our estimate of overfitting is likely conservative, as it is based on training with only a fraction of the data believed to be available to some proprietary model providers. This disparity suggests that proprietary models may be trained on 5 to 10 times more data than we use, potentially increasing the risk of overfitting to patterns not present in our smaller subset. These observations underscore the importance of further examining data scale and its implications for model overfitting to certain leaderboards.

\textbf{We rely on the model's self-identification to attribute private models to their respective providers:} Since the identity of anonymous models is not publicly disclosed, we use model self-identification as a proxy---prompting each model directly and observing how it responds. When a model consistently names a particular provider across multiple prompts, we attribute it to that provider. While this method provides a reasonable signal, it is inherently approximate. Due to limited data and the potential for models to respond inconsistently or ambiguously, some misattributions may occur. To encourage validation of our estimates, we include the codename and our estimate of identity in \Cref{app:private-variants-names}. We welcome correspondence with the authors from providers if any of the estimates of ownership are incorrect.

\section{Related Work}

\begin{quote}
    \textit{The strength of science lies in its ability to withstand challenge, verification, and replication.} \textbf{---  Richard P. Feynman}
\end{quote}

\subsection{Meta-studies on the Rigor of Benchmarking in AI} 

Our work contributes to a wider body of work examining the role of benchmarks in determining progress in machine learning. Benchmarking has played a central role in shaping research priorities and incentives within the deep learning community \citep{koch2024protoscience}. Research has found that benchmarks are rarely impartial and instead shaped by the environments in which the benchmarks are made, finding that assumptions, commitments, and dependencies can often have large implications in final outcomes \citep{aniba2010issues,bartzbeielstein2020benchmarkingoptimizationbestpractice}. Creating a meaningful and reliable benchmark is challenging, and there has been critical work identifying key benchmark desiderata and open challenges.

\textbf{Propensity for overfitting.} Static task-based leaderboards, such as Hugging Face's Open LLM Leaderboard \citep{open-llm-leaderboard-v2, eval-harness} and OpenCompass \citep{2023opencompass}, aim to evaluate models across a broad range of skills but are often susceptible to data contamination and implicit overfitting \citep{deng2023investigating, golchin2023time, roberts2023cutoff, dong2024generalization,singh2024evaluationdatacontaminationllms,longpreresponsible}. Prior works \citep{deng2023investigating, golchin2023time, yang2023rethinking} have proposed various methods for detecting contamination, while \citet{dong2024generalization} discusses how such contamination impedes the ability to distinguish true generalization, ultimately hindering progress. Although dynamic, live benchmarks like \cba significantly reduce the risk of overfitting, we report in this paper that certain practices—such as multiple submissions during the anonymous testing period and best-of-N submissions—tend to favor large, proprietary players with disproportionate access to data. As a result, model development may be deliberately optimized for performance on \cba.

\textbf{Lack of standardization across benchmarks.} The lack of standardization in benchmarks complicates meaningful comparisons due to inconsistent metrics and task definitions. \cite{ethayarajh2020utility} critique NLP leaderboards for prioritizing accuracy over dimensions like model compactness and fairness. Similarly, \cite{ruder2021nlpbenchmarking} highlights that benchmarks such as SuperGLUE \citep{sarlin2020superglue} are quickly saturated, with models reaching superhuman performance while still failing in real-world scenarios, underscoring the need for dynamic and standardized evaluation. This inconsistency risks misleading practitioners, as echoed in recent critiques \citep{bartzbeielstein2020benchmarkingoptimizationbestpractice, reuel2024betterbench}.

\textbf{Quality of data and limited reproducibility.} A recent study by \citet{vendrow2025large} revealed widespread label errors that compromise evaluation reliability, showing that even frontier LLMs can struggle with seemingly simple tasks. Similarly, \citet{digan2021can} identified reproducibility challenges arising from complex data streams, which affect result consistency. Related work \citep{bartzbeielstein2020benchmarkingoptimizationbestpractice, longpre2023data, reuel2024betterbench, albalak2024surveydataselectionlanguage} further emphasizes that poor data quality and limited reproducibility can lead to unreliable evaluations and undermine scientific credibility.

\textbf{Favored benchmarks may not capture performance in the real world.} Commonly used benchmarks often fail to capture real-world performance, creating a gap between test scores and practical utility due to their tendency to overlook the dynamic and complex nature of real-world tasks. Recent studies \citep{ott2022mapping, ai_benchmarks_hit_saturation} highlight this disconnect, observing that models frequently excel on benchmarks while underperforming in practical applications, especially as benchmarks quickly reach saturation.

\subsection{Human Voting-based Benchmarks}

\textbf{Wider studies on the role and benefits of human voting-based benchmarks.} \cba is an example of a human voting-based benchmark.  Human judgment has long been regarded as the gold standard for evaluating the quality of model-generated outputs. These models should ultimately align with human values, and certain nuanced qualities, such as coherence, harmlessness, and readability, are best assessed by humans \citep{van2019best, boubdir2023promptsmakedifferencedata}. Platforms like \cba \citep{chiang2024chatbot}, Talk Arena \citep{talkarena2024}, and Game Arena \citep{ Hu2024GameArenaEL}, Aya UI Interface \citep{singh-etal-2024-aya} effectively use crowdsourcing to gather large volumes of real-world user prompts and feedback. Many opt for Elo-like or BT-style rankings to rank models. Moreover, collecting human preference data has also proven invaluable for alignment techniques like Reinforcement Learning from Human Feedback (RLHF) \citep{christiano2017deep, ouyang2022training,ahmadian2024back,dang2024rlhf}, which helps fine-tune models to generate more natural and human-preferred responses. Human voting has been shown to mitigate some of the biases associated with using LLM-as-a-judge approaches, which, while improving evaluation efficiency, may raise concerns about robustness \citep{raina2024llm} and introduce various forms of bias \citep{koo2023benchmarking, shimabucoro2023evaluating, chen2402humans, zheng2024judging}. Furthermore, live leaderboards offer several advantages over static task benchmarks, including a lower risk of data contamination and greater adaptability to evolving evaluation needs.

\textbf{Critiques of Human-Voting Based Benchmarks.} Voting-based live benchmarks like \cba also face evaluation challenges not addressed in this paper. \cba \citep{chiang2024chatbot} has made substantial efforts to ensure reliability and security, including malicious user detection, bot protection via Google reCAPTCHA v3, vote limits per IP address, prompt de-duplication, and other safeguards\footref{policy}. Nonetheless, recent work has focused on auditing the reliability of human-voting-based live leaderboards. For instance, studies have demonstrated that such leaderboards are vulnerable to low-cost manipulation, with adversarial users able to de-anonymize model responses and carry out targeted voting attacks \citep{huang2025exploring}. Additionally, \citet{Zhao2024ChallengesIT, min2025improving} suggest that \cba rankings can be artificially inflated through various adversarial voting strategies. These vulnerabilities raise concerns about the overall trustworthiness of \cba. While our study does not explicitly investigate adversarial voting, we note that \cba’s policy of informing model providers when testing begins and disclosing model aliases may create conditions conducive to leaderboard manipulation.

\section{Conclusion}

It is far easier to point out issues with the Arena than the huge amount of work that went into building it. While our work is motivated by the need to maintain scientific integrity in AI progress, we believe it is important to note the huge amount of work involved for a small group of organizers to build a hugely popular community benchmark. Their efforts have democratized access to many models and enabled a large and varied user base to weigh in on what matters in the real world for model selection. Hence, while we point out systematic issues with \cba as it is now, we also acknowledge that many of these issues may have gradually emerged as the leaderboard took on outsized importance in visibility to providers.

This work demonstrates the difficulty in maintaining fair evaluations, despite best intentions. We show that coordination among a handful of providers and preferential policies from \cba towards the same small group have jeopardized scientific integrity and reliable Arena rankings. The widespread and apparent willful participation in the gamification of arena scores from a handful of top-tier industry labs is undoubtedly a new low for the AI research field. \textbf{As scientists, we must do better. As a community, we must demand better.}

We believe it is very feasible for the organizers of \cba to continue to innovate and restore trust by revising their policies. We propose a series of very straightforward recommendations to help reinforce the reliability and fairness of the leaderboard. Most urgently, providers should not be allowed to choose which scores are made public. There should be strict and transparent limits to the number of private variants per provider. Providers, whether from academic or industry labs, should be aware of private testing and the limits should be the same across providers. There should be transparent criteria for model removal from the arena and fairer sampling that is motivated by reducing uncertainty in rankings instead of being skewed towards giving proprietary models more battles. We believe the implementation of these recommendations is critical for addressing ranking distortions, but also necessary in the long term to ensure the benefits from participating on the leaderboard are not concentrated in a handful of providers.

\section{Acknowledgements}

We thank our colleagues who have supported various aspects of this project: Madeline Smith, Brittwanya Prince, Thomas Euyang, and Shubham Shukla.

\bibliography{refs}

\appendix

\section{Chatbot Arena Background}

LMSYS originated from a multi-university collaboration involving UC Berkeley, Stanford, UCSD, CMU, and MBZUAI in 2023. It was established as a non-profit corporation in September 2024 to incubate early-stage open-source and research projects. 
Chatbot Arena was first launched in May 2023 under LMSYS and later evolved into a standalone project with its own dedicated website\footnote{\url{https://lmsys.org/blog/2024-09-20-arena-new-site/}} maintained under the name LMArena by researchers from UC Berkeley SkyLab.
It has emerged as a critical platform for live, community-driven LLM evaluation, attracting millions of participants and collecting over 3 million votes to date.

LMArena operates based on human preferences. Chatbot Arena asks users to input prompts in battles. The user then votes for their preferred model based on the outputs generated by the models in the battle in response to the user's prompts. These preferences are then used by Chatbot Arena to compute model ratings using algorithms like Online Elo and Bradley-Terry.

\section{Bradley-Terry Rating Model}
\label{app:bradley-terry}

Consider a set of \( m \) players (models) and \( n \) pairwise comparisons between them. Let \( X \in \mathbb{R}^{m \times n} \) be the design matrix, where each column represents one pairwise comparison. In the Bradley-Terry model, the probability that player \( i \) is preferred over player \( j \) in a comparison is modeled as:

\[
    P(\text{i preferred over  j}) = \frac{1}{1 + e^{(\beta_j - \beta_i)}}
\]
     
Then, in the matrix \( X \), column vector $k$ has a value of 1 at position \( i \) , -1 at position \( j \), and 0 elsewhere. Let $Y \in \{ 0 ,1\}^n$ be the vector of observed outcomes, where $Y_k = 1$ if player $i$ wins the $k$-th comparison and $Y_k = 0$ if player $j$ wins. 
Our goal is to estimate the Bradley-Terry coefficients \( \beta \in \mathbb{R}^m \), which determine the relative strengths of the players. The coefficients $\beta$ are estimated via maximum likelihood estimation by minimizing the expected cross-entropy loss,

\[
    \hat{\beta} = \arg\min_{\beta \in \mathbb{R}^{m} } \frac{1}{n} \sum_{k=1}^{n} \ell \big( \sigma(X^T \beta)_k, Y_k \big)    
\]

where \( \sigma(\cdot) \) is the logistic function that models the relative player strengths, and \( \ell(\cdot) \) represents the binary cross-entropy loss between the predicted probabilities and the observed outcomes \( Y \). The estimated coefficient \( \beta \) captures the latent strength of each player.

Once the Bradley-Terry model estimates the coefficients, we can scale them to obtain Elo-like ratings using the transformation:

\[
    R_m = \text{scale} * \beta + \text{initial rating}
\]

In practice, Chatbot Arena does not rely solely on a model’s Arena Score for ranking. Instead, it also considers the confidence intervals associated with these scores. When the confidence intervals of two models overlap, it becomes difficult to determine which one is truly better. This uncertainty is reflected in the final ranking table, adding nuance and statistical rigor to the leaderboard \citep{chiang2024chatbot}.

\[
    \text{rank}(m) = 1 + \sum_{m' \in [M]} 1 \left\{ m' > m\right\}
\]

\section{Unbiased Sampling: Why Selecting the Maximum Introduces Bias?}
\label{app:unbiased_sampling}

Let $(\hat{\beta}_k)_{k=1}^N$ be i.i.d. real-valued random variables with common cumulative distribution function $F$ and finite expectation $\mu := \mathbb{E}[\hat{\beta}_k]$. Assume the distribution is \emph{non-degenerate}, i.e., $\operatorname{Var}(\hat{\beta}_k) > 0$. The maximum is defined as:
\[
\hat{\beta}_{\mathrm{Best}} := \max\{\hat{\beta}_1, \dots, \hat{\beta}_N\}, \qquad N \geq 2.
\]

\begin{theorem}
For every $N \geq 2$,
\[
\mathbb{E}[\hat{\beta}_{\mathrm{Best}}] > \mathbb{E}[\hat{\beta}_k] \iff \operatorname{Var}(\hat{\beta}_k) > 0.
\]
\end{theorem}

\begin{proof}
The cumulative distribution function (CDF) of the maximum is
\[
F_{\hat{\beta}_{\mathrm{Best}}}(x) = \mathbb{P}(\hat{\beta}_{\mathrm{Best}} \le x) = F(x)^N.
\]
Using integration by parts, we have:
\begin{align*}
\mathbb{E}[\hat{\beta}_{\mathrm{Best}}] - \mathbb{E}[\hat{\beta}_1]
&= \int_{-\infty}^{\infty} x\,d(F(x)^N - F(x)) \\
&= \int_{-\infty}^{\infty} (F(x) - F(x)^N)\,dx.
\end{align*}
For all $x$ such that $0 < F(x) < 1$, and $N \ge 2$, we have $F(x)^N < F(x)$, so the integrand is strictly positive on a set of positive measure (since the distribution is non-degenerate). Thus, the integral -- and hence the difference in expectations -- is strictly positive.

If $\operatorname{Var}(\hat{\beta}_1) = 0$, then $F$ is a step function with a single jump (a constant distribution), and $F(x) - F(x)^N = 0$ for all $x$, yielding equality.
\end{proof}

\begin{remark}
This result formalizes the \emph{selection bias} arising when one reports the best out of $N$ noisy skill estimates: statistical fluctuations ensure that the maximum overestimates the expected performance of a typical sample. This is especially relevant in leaderboard scenarios where multiple submissions are made and only the top-performing one is reported. This phenomenon is well-studied in the theory of order statistics (see \citet{ArnoldBalakrishnanNagaraja1992, DavidNagaraja2003}).
\end{remark}

\section{Data sources}
\label{app:data-sources}

To gain insights and analyze various trends in the \cba leaderboard, we leverage multiple data sources. In total, our real-world data sources encompass 2M battles and cover 243 models across 42 providers. Below, we describe the different datasets used in our analyses.

\begin{enumerate}
    \item \textbf{Historical Battles}
    (\texttt{\lmarenahistorical}): is a collection of 1.8 million battles from \cba from April 2023 to January 2025. We build this resource by combining both released public battles by \cba and proprietary battle dataset released by \cba to providers such as Cohere based upon their policy\footref{policy}. We describe both datasets in more detail \Cref{app:dataset_descriptions}. We leverage \texttt{\lmarenahistorical} dataset as a key resource for quantifying task distribution drift (see \Cref{fig-language-distribution}):
     \begin{itemize}
         \item How do Arena use cases change over time?
    \end{itemize}
    \item \textbf{API Prompts:} Majority of \texttt{\lmarenahistorical} dataset does not contain prompts as shown in \Cref{tab:data_sources}. Additionally, all datasets published by LMArena are already de-duplicated so they won't be useful for capturing the extent of similar or overlapping queries. Hence we switch to prompts collected via Cohere's API based on requests received via \cba, comprising a total of 567,319 entries. For simplicity and the purposes of this study, we excluded records with null values and multi-turn data and analyzed 197,217 single-turn conversations collected between November 2024 and April 2025. The models include \texttt{command-r-08-2024}, \texttt{command-r-plus-08-2024} \citep{cohere2024}, \texttt{aya-expanse-8b},  \texttt{aya-expanse-32b} \citep{ustun-etal-2024-aya, dang2024aya}, and \texttt{command-a-03-2025} \citep{cohere2025commandaenterprisereadylarge}, along with three private variants. Of these, 62\% of the data was labeled as coming from Aya models, while the remaining 38\% was attributed to Command models. We use this dataset for prompt duplication analysis (see \Cref{fig:duplicate_prompt} and \Cref{app:prompt-duplication-headmap}):
    \begin{itemize}
         \item How many prompts are duplicates or close duplicates?
    \end{itemize}
    \item \textbf{Leaderboard Statistics} (\texttt{\lmarenaleaderboard}): is snapshots of ratings and rankings as well as the number of battles played over time by models published on \cba's public leaderboard since its inception. To build this resource, we consolidate historical leaderboard snapshots released by \cba on Hugging Face\footnote{\url{https://huggingface.co/spaces/lmarena-ai/chatbot-arena-leaderboard/tree/main}}. For fair assessment, we consider historical data starting from January 9 2024 -- April 23 2025 for our analysis since \cba switched to using the latest Bradley-Terry model in December 2023 to improve the reliability of model rankings\footref{leaderboard-elo-update}. By combining all leaderboard tables published by LMArena during this period, we obtained 14.3K records corresponding to 243 unique models evaluated on \cba. We also enriched this dataset with additional metadata, such as categorizing models as proprietary, open-weight, or open-source based on the classification described in \Cref{app:license-categories}. We use this data for analyzing trends related to no. of models, data access across providers (See Figures \ref{fig-max-rating-num-models-main-providers}, \ref{fig-percentage-battles-open-proprietary} and \ref{fig:public_private_data}) as well as model deprecation (See Figures \ref{fig-silent-deprecated_prop_open}, \ref{fig-silent-deprecated} and \ref{fig-silent-deprecated-license-cat}):
      \begin{itemize}
      \item How does data access vary between providers?
      \item How do models' deprecations vary by provider and across proprietary, open-weight, and open-source models?
    \end{itemize}
    \item \textbf{Random Sample Battles} (\texttt{scraped-random-sample}): The \texttt{\lmarenahistorical} and \texttt{leade\\rboard-stats} dataset does not provide insights into private testing being conducted by different providers. It appears private battles are removed by \cba maintainers from the data before being released in both datasets. Furthermore, \texttt{\lmarenahistorical} contains the majority of samples from 2023 and 2024 and does not provide visibility in current sampling rate trends being followed on the Arena. To address this gap, we collected 5864 battles by crawling \cba between January 2025 -- March 2025 (approximately 150 a day). To avoid our collection from disrupting actual voting, we first ask models about their identity, which causes models to reveal their identities and automatically invalidates these battles for updating the scores.\footref{policy} \citep{chiang2024chatbot} As a further precaution, we only scrape a low volume of daily samples and only vote for ties between models.
    Additionally, we use this identity prompt to identify model ownership of private variants, as detailed in the \Cref{app:de-anonymize-prompt}. We store the identity revealed for each model to track the volume of private testing (more details included in \Cref{app:private-variants-names}). We use this \texttt{scraped-random-sample}, which is a representative random sample over time, to answer a few critical questions: 
    \begin{itemize}
        \item Are different models sampled for battles at similar rates?
        \item How many anonymous models are being tested by different model providers?
    \end{itemize}

\end{enumerate}

We provide additional details about \texttt{\lmarenahistorical} dataset in the \Cref{app:dataset_descriptions}. We also provide a detailed overview of all datasets in \Cref{tab:data_sources}.

\subsection{Public and Private Battles}
\label{app:dataset_descriptions}
Our \texttt{\lmarenahistorical} dataset includes snapshots of battles played on \cba that have been released publicly or shared privately with model providers based on their policy\footref{policy}. We provide additional details about public and private subsets of \lmarenahistorical, for the reader's consideration below. 
    \begin{itemize}
    \item \textbf{Public Battles:} The public portion of our historical data comes from the officially released datasets by \cba on Hugging Face or as part of notebook tutorials. We combine the \textit{arena-human-preference-100K}\footnote{\url{https://huggingface.co/datasets/lmarena-ai/arena-human-preference-100k}} \citep{chiang2024chatbot, tang2025explorer} dataset containing 106K samples from June 2024 -- August 2024 with datasets shared by \cba as part of notebook tutorials on Bradley Terry\footnote{\url{https://blog.lmarena.ai/blog/2023/leaderboard-elo-update/}\label{leaderboard-elo-update}} and Elo Rating systems\footnote{\url{https://blog.lmarena.ai/blog/2023/arena/}}. 
    This resulted in around 2M samples from April 2023 to August 2024 in total. 90\% of the data does not include any prompt or completion history, instead consisting only of the names of the two models battling and the winning model as well as language and task category tags. We exclude other public battles released by \cba for inclusion in \texttt{\lmarenahistorical} dataset since they did not contain required columns or enough multilingual data points required for the analysis presented in Figure \ref{fig-language-distribution}.
    \item \textbf{Proprietary Battles:} 
    We also obtain historical battle data from \cba maintainers for battles that involve Command and Aya models. This data was shared based on \cba's policy\footref{policy}, which permits model providers to request access to 20\% of the data collected involving their own models. 
    The data we received consists of 43,729 battles played by the following models between March 2024 and March 2025: \texttt{command-r}, \texttt{command-r-plus}, \texttt{command-r-08-2024}, \texttt{command-r-plus-08-2024} \citep{cohere2024}, \texttt{aya-expanse-8b}, \texttt{aya-exp-\\anse-32b} \citep{ustun-etal-2024-aya, dang2024aya}. In contrast to the public data, this proprietary data contains the complete model conversations. Since this data is 46\% multilingual, we combine this with \textbf{public battles} to form \texttt{\lmarenahistorical} and use it for language distribution shift analysis presented in Figure \ref{fig-language-distribution}.

   \end{itemize}

\section{Our Scraping Methodology of LMArena Statistics}
\label{app:lmarena-scraping}

We collected 5.8K battles (\texttt{\lmarenascraped}) by crawling data from Chatbot Arena on a regular basis between January--March, 2025. For this purpose, we setup a scraping script using Selenium library with chrome browser driver. To identify anonymous models, we first sent a de-anonymizing prompt. While Chatbot Arena does discard battles where models reveal their identities, as an additional measure on our end, we ask a simple follow-up question designed to most likely result in ties, such as ``What is the capital of England? Reply with one word only.'' or ``Is the Earth round? Reply with Yes/No only.'' Our scraping script extracted the models' names as well as their responses to the asked questions.
In addition to the \texttt{\lmarenascraped} collected by crawling the main \cba leaderboard, we also collected around 500 additional samples by scraping the Vision leaderboard between 9th March and 28th March, 2025. This helped us in identifying 35 private vision models which are shown in section \ref{app:private-scrape-models}. We refer to this collected set of vision battles as \texttt{scraped-vision-sample}.

\subsection{De-anonymizing Model Identities}
\label{app:de-anonymize-prompt}

While crawling battles to prepare \texttt{\lmarenascraped}, we ask the models about their identity. This helps in ensuring that our votes from scraping the arena don't interfere with the leaderboard rankings since Chatbot Arena discards votes in which models reveal their identities \citep{chiang2024chatbot}. We use either one of the following prompts to de-anonymize the model identity.

\begin{tcolorbox}[
  colback=gray!10!white, 
  colframe=black, 
  boxrule=0.8pt, 
  title=De-Anonymize Prompt,
  label=De-Anonymize Prompt,
]
1. Who are you? \\
2. Who are you? Respond with only your name and who trained you.
\end{tcolorbox}

The model identities are then inferred based on the responses of the models. In \Cref{app:private-variants-names}, we specify the responses of different private variants based on which they were assigned to their respective providers.
Using this approach, we identified a total of 64 private models corresponding to 10 providers. We also captured 14 other private models as part of our scraping but weren't able to de-anonymize them: \textit{kiwi}, \textit{space}, \textit{maxwell}, \textit{luca}, \textit{anonymous-engine-1}, \textit{tippu}, \textit{sky}, \textit{pineapple}, \textit{pegasus}, \textit{dasher}, \textit{dancer}, \textit{blueprint}, \textit{dry\_goods}, \textit{prancer}. 

\subsection{Encountered Private Models in Scraping}
\label{app:private-scrape-models}

\begin{longtable}{p{2cm} p{1.5cm} p{7cm} p{5cm}}
\caption{\textbf{Private Models by Provider}. We show the private models corresponding to each provider, which were identified by crawling overall and vision leaderboards (See Section \ref{app:lmarena-scraping}). The models highlighted in bold appear on both leaderboards. We find that Meta had an additional 16 private models active on the Vision leaderboard along with its 27 models on the Overall leaderboard, bringing its total count to 43. We show the models corresponding to overall leaderboard in  Figure \ref{fig-private-testing-providers}. We exclude models corresponding to LMArena from this figure, as they are associated with the Prompt-to-Leaderboard work led by \cba \citep{frick2025prompttoleaderboard}.} \label{tab:private_models} \\

\toprule
\textbf{Provider}
& \textbf{\parbox{1.5cm}{No. of private models}} & 
\textbf{\parbox{6cm}{Private Models from Overall leaderboard}}  
& \textbf{\parbox{5.5cm}{Additional Private Models from Vision leaderboard}}  \\
\midrule
\endfirsthead
\toprule
\textbf{Provider}
& \textbf{\parbox{1.5cm}{No. of private models}} & 
\textbf{\parbox{6cm}{Private Models from Overall leaderboard}}  
& \textbf{\parbox{5.5cm}{Additional Private Models from Vision leaderboard}}  \\
\midrule
\endhead
\midrule
\endfoot
\bottomrule
\endlastfoot

\multirow{4}{*}{Meta} & \multirow{4}{*}{43}
  & \textbf{polus} & aurora \\ 
& & deep-inertia & cresta\\ 
& & \textbf{goose} & discovery\\ 
& & \textbf{falcon} & ertiga\\ 
& & \textbf{jerky} & flux\\ 
& & anonymous-engine-2 & harmony\\ 
& & \textbf{kronus} & helix\\ 
& & consolidation & pinnacle\\ 
& & \textbf{flywheel} & portola\\ 
& & inertia & prosperity\\ 
& & momentum & raze\\ 
& & \textbf{rhea} & solaris\\ 
& & \textbf{sparrow} & spectra\\ 
& & spider & toi\\ 
& & \textbf{gaia} & vega\\ 
& & \textbf{rage} & zax\\
& & frost \\ 
& & themis \\ 
& & cybele \\ 
& & unicorn-engine-1 \\ 
& & unicorn-engine-2 \\ 
& & unicorn-engine-3 \\ 
& & unicorn-engine-4 \\ 
& & unicorn-engine-5 \\ 
& & unicorn-engine-6 \\ 
& & unicorn-engine-7 \\ 
& & \textbf{uranus} \\ 
\midrule

\multirow{2}{*}{OpenAI} & \multirow{2}{*}{3}  
  & anonymous-chatbot & gpt4o-lmsys-0315a-ev3-vis\\ 
& & gpt4o-lmsys-0315a-ev3-text \\ 
\midrule

\multirow{4}{*}{Google} & \multirow{4}{*}{10} 
  & \textbf{centaur} \\ 
& & \textbf{enigma} \\ 
& & \textbf{gremlin} \\ 
& & gemini-test \\ 
& & zizou-10 \\ 
& & \textbf{specter} \\ 
& & moonhowler \\ 
& & \textbf{phantom} \\ 
& & \textbf{nebula} \\ 
& & goblin \\ 
\midrule
\multirow{4}{*}{Amazon} & \multirow{4}{*}{7}
  & raspberry-exp-beta-v2 \\ 
& & \textbf{apricot-exp-v1} \\ 
& & cobalt-exp-beta-v2 \\ 
& & raspberry-exp-beta-v1 \\ 
& & raspberry \\ 
& & cobalt-exp-beta-v1 \\ 
& & raspberry-exp-beta-v3 \\ 
\midrule

\multirow{4}{*}{Cohere} & \multirow{4}{*}{6}
  & cohort-chowder & asterix\\ 
& & sandwich-ping-pong & buttercup \\ 
& & grapefruit-polar-bear \\ 
& & roman-empire \\ 
\midrule

\multirow{4}{*}{LMArena} & \multirow{4}{*}{5}
  & p2l-router-7b-0317 \\ 
& & p2l-router-7b-0318  \\ 
& & p2l-router-7b \\
& & experimental-router-0207 \\ 
& & experimental-router-0122 \\ 
& & experimental-router-0112 \\ 
\midrule

\multirow{1}{*}{Nvidia} & \multirow{1}{*}{2} 
  & march-chatbot-r \\ 
& & march-chatbot \\ 
\midrule
xAI & 1 & anonymous-test \\ 
\midrule
Reka & 1 & \textbf{margherita-plain} \\ 
\midrule
Alibaba & 1 & qwen-plus-0125-exp \\ 
\midrule
StepFun & 1 & step-2-16k-202502 \\
\midrule
\multirow{2}{*}{Unknown} & \multirow{2}{*}{14}
& kiwi \\
& & space \\
& & maxwell \\
& & \textbf{luca} \\
& & anonymous-engine-1 \\
& & tippu \\
& & sky \\
& & pineapple \\
& & pegasus \\
& & dasher \\
& & dancer \\
& & blueprint \\
& & dry\_goods \\
& & prancer \\
% \midrule
\end{longtable}

\subsection{Encountered Public Models in Scraping}

\begin{longtable}{p{2cm} p{1.5cm} p{7cm} p{5cm}}
% \caption{Public Models that appear in scraped data per provider} 
\caption{\textbf{Public Models per Provider}. This table shows the public models from each provider that appeared on the overall and vision leaderboards during our scraping period (January--March 2025). Models highlighted in bold appear on both leaderboards. Google and OpenAI had the most public models active during this period, with 15 and 9 models, respectively.
}
\\
\toprule
\textbf{Provider}
& \textbf{\parbox{1.5cm}{No. of public models}} & 
\textbf{\parbox{6cm}{Public Models from Overall leaderboard}}  
& \textbf{\parbox{5cm}{Additional Public Models from Vision leaderboard}}  \\
\midrule
\endfirsthead
\toprule
\textbf{Provider}
& \textbf{\parbox{1.5cm}{No. of public models}} & 
\textbf{\parbox{6cm}{Public Models from Overall leaderboard}}  
& \textbf{\parbox{5cm}{Additional Public Models from Vision leaderboard}}  \\
\midrule
\endhead
\midrule
\endfoot
\bottomrule
\endlastfoot
Meta & 3
 & llama-3.1-405b-instruct-bf16 & llama-3.2-vision-90b-instruct \\
& & llama-3.3-70b-instruct &  \\
\midrule
Amazon & 3 & \textbf{amazon-nova-lite-v1.0} \\
& & \textbf{amazon-nova-pro-v1.0} &  \\
& & amazon-nova-micro-v1.0 & \\
\midrule
Anthropic & 5 & claude-3-5-haiku-20241022 &  \\
 & & claude-3-7-sonnet-20250219-thinking-32k &  \\
 & & claude-3-5-sonnet-20241022 & \\
 & & \textbf{claude-3-7-sonnet-20250219} &  \\
 & & claude-3-opus-20240229 &  \\
\midrule
Alibaba & 5 & qwen2.5-72b-instruct & qwen2.5-vl-72b-instruct \\
 & & qwq-32b &  \\
 & & qwen-max-2025-01-25 & \\
 & & qwen2.5-plus-1127 & \\
\midrule
Google & 15 & gemma-2-2b-it &  \\
 & & \textbf{gemini-2.0-pro-exp-02-05} & \\
 & & \textbf{gemini-1.5-pro-002} &  \\
 & & gemini-2.0-flash-thinking-exp-1219 &  \\
 & & gemini-1.5-flash-002 &  \\
 & & gemini-2.0-flash-lite-preview-02-05 &  \\
 & & gemini-1.5-flash-8b-001 &  \\
 & & gemini-2.0-flash-exp &  \\
 & & \textbf{gemma-3-27b-it} & \\
 & & gemma-2-9b-it &  \\
 & & gemini-exp-1206 &  \\
 & & \textbf{gemini-2.0-flash-thinking-exp-01-21} &  \\
 & & \textbf{gemini-2.5-pro-exp-03-25} & \\
 & & \textbf{gemini-2.0-flash-001} &  \\
 & & gemma-2-27b-it &  \\
\midrule
OpenAI & 9 & \textbf{o3-mini} \\
 & & \textbf{gpt-4o-mini-2024-07-18} \\
 & & \textbf{o1-2024-12-17}  \\
 & & \textbf{gpt-4.5-preview-2025-02-27} \\
 & & o3-mini-high \\
 & & chatgpt-4o-latest-20250326 \\
 & & chatgpt-4o-latest-20241120  \\
 & & \textbf{chatgpt-4o-latest-20250129} \\
 & & o1-mini \\
\midrule
StepFun & 1
 & step-2-16k-exp-202412 \\
\midrule
xAI & 4 & early-grok-3  \\
 & & grok-2-2024-08-13 \\
 & & \textbf{grok-3-preview-02-24} \\
 & & grok-2-mini-2024-08-13  \\
\midrule
DeepSeek & 3 & deepseek-v3 \\
 & & deepseek-v3-0324 \\
 & & deepseek-r1  \\
\midrule
Microsoft & 1 & phi-4  \\
\midrule
Mistral & 3 & mistral-large-2411 & pixtral-large-2411 \\
 & & mistral-small-24b-instruct-2501  \\
\midrule
Cohere & 4 & command-a-03-2025 & c4ai-aya-vision-32b \\
 & & c4ai-aya-expanse-8b  \\
 & & c4ai-aya-expanse-32b  \\
\midrule
Tencent & 4 & hunyuan-turbos-20250226  \\
 & & hunyuan-turbo-0110  \\
 & & hunyuan-standard-2025-02-10 \\
 & & hunyuan-large-2025-02-10 \\
\midrule
NexusFlow & 1 & athene-v2-chat \\
\midrule
Zhipu & 1 & glm-4-plus-0111 \\
\midrule
IBM & 2 & granite-3.1-2b-instruct  \\
 & & granite-3.1-8b-instruct \\
\midrule
Allen AI & 2 & llama-3.1-tulu-3-70b \\
 & & llama-3.1-tulu-3-8b  \\
 & & olmo-2-0325-32b-instruct \\
\end{longtable}

\subsection{Assignment of Private Variants to Providers}
\label{app:private-variants-names}

{\small
\renewcommand{\arraystretch}{1.3}
\begin{longtable}{p{0.1\textwidth} p{0.2\textwidth} p{0.2\textwidth}p{0.4\textwidth}}
    \caption{\textbf{Private variants identified for different providers}. The table lists the private models captured in our \texttt{\lmarenascraped} or \texttt{scraped-vision-sample}, along with the number of responses revealing their identities and corresponding examples. Notably, the model \textbf{raspberry} withheld its identity in most responses (37 total) but disclosed ``Amazon'' as its provider in three instances. It's possible that some private models appeared in more battles, but we couldn't capture their responses to our de-anonymizing prompt due to scraping errors. Additionally, a few battles occurred early in the project before we introduced asking the de-anonymizing prompt in our scraping methodology. Note that models with prefixes "p2l" and "experimental-router" identify as OpenAI/Google models but we assign them to LMArena as they are part of prompt-to-leaderboard work being lead by LMArena.\citep{frick2025prompttoleaderboard}.}     \label{tab:model_comparison_multi_count} \\
    \toprule
    \textbf{Identified Provider} & \textbf{Model Name} & \textbf{No. of Responses Revealing Identity} & \textbf{Model Responses} \\
    \midrule
    \endfirsthead

    \toprule
    \textbf{Identified Provider} & \textbf{Model Name} & \textbf{No. of Responses Revealing Identity} & \textbf{Model Responses} \\
    \midrule
    \endhead

    \midrule
    \multicolumn{4}{r}{} \\
    
    \endfoot

    \bottomrule
    \endlastfoot

  % --- Meta ---
  \multirow{2}{*}{Meta} & kronus & 73 & Llama, trained by Meta. \\
  \cline{2-4} 
  & polus & 57 & Llama, trained by Meta.\\
  \cline{2-4} 
  & frost & 47 & I'm an artificial intelligence model known as Llama. Llama stands for "Large Language Model Meta AI.  \\
  \cline{2-4}
  & gaia & 43 & Llama; trained by Meta.  \\
  \cline{2-4} 
  & uranus & 42 & Llama, trained by Meta.  \\
  \cline{2-4} 
  & rhea & 41 & Llama, trained by Meta  \\
  \cline{2-4} 
  
  & consolidation & 34 & I am Llama. I was trained by Meta AI. 
  \par\noindent\dotfill\par Llama.Trained by Meta AI, with roots in Large Language Model Meta AI (LLaMA) developed by Meta.  \\
  \cline{2-4} 
  & flywheel & 34 & I am Llama, and I was trained by Meta.  \\
  \cline{2-4} 
  \multirow{30}{*}{Meta} & momentum & 34 &  Llama, Meta AI \par\noindent\dotfill\par I am Llama, and I was trained by Meta.  \\
  \cline{2-4} 
  & deep-inertia & 33 &  I am Llama. I was trained by Meta AI. \par\noindent\dotfill\par Llama. Trained by Meta AI. \\
  \cline{2-4} 
  & inertia & 29 &  LLaMA, Meta AI  \\
  \cline{2-4}
  & jerky & 28 & I am Llama. I was trained by Meta AI.  \\
  \cline{2-4} 
   & goose & 25 &  I am Llama. Llama is an AI language model developed by Meta. Meta trained me. \\
  \cline{2-4} 
  & falcon & 23 & I am Llama. Llama was developed by Meta. \par\noindent\dotfill\par I am an AI assistant trained by Meta.  \\
  \cline{2-4} 
  & rage & 14 & I am Llama, trained by Meta AI. \par\noindent\dotfill\par Llama. Meta.  \\
  \cline{2-4}
  & anonymous-engine-2 & 12 & I'm an artificial intelligence model known as Llama. Llama stands for "Large Language Model Meta AI. \\
  \cline{2-4}
  & sparrow & 10 & I'm LLaMA, and I was trained by Meta. \par\noindent\dotfill\par I'm LLaMA, and I was trained by researchers at Meta.  \\
  \cline{2-4} 
  & cybele & 9 &  Llama, trained by Meta. \\
  \cline{2-4} 
   & unicorn-engine-1 & 2 & I'm an artificial intelligence model known as Llama. Llama stands for "Large Language Model Meta AI" \\
  \cline{2-4}
  & unicorn-engine-2 & 4 &  I'm an artificial intelligence model known as Llama. Llama stands for "Large Language Model Meta AI"\\
  \cline{2-4}
  & unicorn-engine-3 & 4 & I'm an artificial intelligence model known as Llama. Llama stands for "Large Language Model Meta AI".\\
  \cline{2-4} 
  & unicorn-engine-4 & 1 & I'm an artificial intelligence model known as Llama. Llama stands for "Large Language Model Meta AI". \\
  \cline{2-4}
  & aurora & 4  & LLaMA, Meta. \\
  \cline{2-4}
  & cresta & 21  & Llama. Trained by Meta. \\
  \cline{2-4}
  & discovery & 15  & Llama, Meta. \\
  \cline{2-4}
  & ertiga & 3  & I am LLaMA, trained by Meta. \\
  \cline{2-4}
  & flux & 4  & LLaMA, Meta. \\
  \cline{2-4}
  & harmony & 31  & Llama, Meta. \\
  \cline{2-4}
  & helix & 1  & AI Assistant, Meta. \\
  \cline{2-4}
  \multirow{35}{*}{Meta} & spider & 9 & Very direct and refreshing!\smileemoji \par I am: LLaMA (Large Language Model Application)
Trained by: Meta AI (formerly Facebook AI)
\par That's it. \winkemoji
\par Just for clarity:
\par "LLaMA" is the specific model name (my "given name").
\par Meta AI is the organization/team that created me, using a massive dataset, complex neural network architectures, and a lot of compute power. Think of them as my "teachers" or "mentors". They designed the training objectives, curated the data, and ran the gigantic training process that shaped me into what I am today.
Now, ask me anything else! \par\noindent\dotfill\par
A refreshingly direct question! \smileemoji
\par I am: LLaMA
\par (Large Language Model Meta AI)
\par Trained by: Meta AI (a team of several hundred researchers and engineers)
\par That's it. No fluff. \blushemoji
\par For context:
\par "LLaMA" is not just a name I chose, it's actually an acronym for Large Language Model Meta AI. It's my model name.
"Meta AI" is the organization (formerly known as Facebook AI, part of Meta Platforms, Inc.) whose team of researchers, scientists, and engineers built and trained me from scratch using a massive dataset of text (books, articles, websites, research papers... basically the internet up to 2021) and a lot of computational power (think thousands of GPU hours).
\par No single person "trained" me in the classical sense; it was a large team effort using a combination of supervised learning, self-supervised learning, and reinforcement learning techniques. \\
  \cline{2-4}
   & pinnacle & 8  & AI Assistant, Meta trained me. \\
  \cline{2-4}
  & portola & 11  & AI, Meta trained me. \\
  \cline{2-4}
  & prosperity & 31  & Llama, Meta. \\
  \cline{2-4}
  & raze & 32  & Llama, Meta. \\
  \cline{2-4}
  & solaris & 2  & Meta trained me. \\
  \cline{2-4}
  & spectra & 3  & AI Assistant, Meta. \\
  \cline{2-4}
  \multirow{10}{*}{Meta} & toi & 6  & LLaMA, Meta. \\
  \cline{2-4}
  & vega & 36  & Llama, Meta. \\
  \cline{2-4}
  & zax & 13  & I am Meta AI, trained by Meta. \\
  \cline{2-4}
  & unicorn-engine-5 & 4 & I'm an artificial intelligence model known as Llama. Llama stands for "Large Language Model Meta AI"  \\
  \cline{2-4} 
   & unicorn-engine-6 & 2 & I'm an artificial intelligence model known as Llama. Llama stands for "Large Language Model Meta AI". \\
  \cline{2-4}
  & unicorn-engine-7 & 1 & I'm an artificial intelligence model known as Llama. Llama stands for "Large Language Model Meta AI". \\
  \cline{2-4}
  & themis & 8 & Llama, trained by Meta. \par\noindent\dotfill\par Llama, trained by Meta AI.  \\
  \cline{2-4}
  \midrule

  % --- Google ---
  \multirow{25}{*}{Google} & gemini-test & 219 & I am a large language model, trained by Google.  \\
  \cline{2-4} 
  & enigma & 220 & I am a large language model, trained by Google. \par\noindent\dotfill\par I am a large language model, trained by Google. I am developed to be informative and comprehensive. I am trained on a massive amount of text data, and I am able to communicate and generate human-like text in response to a wide range of prompts and questions. For example, I can provide summaries of factual topics, create stories, and translate languages. \\
  \cline{2-4} 
  & goblin & 170 & I am a large language model, trained by Google. 
  \par\noindent\dotfill\par 
  I am a large language model, trained by Google. I am developed to be informative and comprehensive. I am trained on a massive amount of text data, and some of the information I have access to is Google Search.\\
  \cline{2-4} 
  & phantom & 100 & I am a large language model, trained by Google.  \\
  \cline{2-4} 
  & gremlin & 94 & I am a large language model, trained by Google.  \\
  \cline{2-4} 
  & specter & 67 & I am a large language model, trained by Google.  \\
  \cline{2-4} 
  & centaur & 58 & I am a large language model, trained by Google. \\
  \cline{2-4} 
  \multirow{13}{*}{Google}  & zizou-10 & 47 & Hi there!\handwaveemoji  I'm Gemma, an open-weights AI assistant. I'm a large language model trained by Google DeepMind. The Gemma team are my creators.
I'm widely available to the public – I'm an open weights model, which means I'm pretty accessible! I take text and images as inputs and generate text as output. Basically, you can give me a prompt (text or an image!), and I'll do my best to give you a helpful and informative response.
I don't have access to tools, real-time information or Google search.
You can find more information about me here: https://ai.google.dev/gemma  \\
  \cline{2-4} 
  & moonhowler & 6 & I am a large language model, trained by Google.  \\
 \cline{2-4} 
  & nebula & 4 & I am a large language model, trained by Google.  \\
  % \cline{2-4} 
  \midrule

  % --- OpenAI ---
  \multirow{10}{*}{OpenAI} & anonymous-chatbot & 49 & I am ChatGPT, trained by OpenAI. 
  \par\noindent\dotfill\par I’m ChatGPT, an AI assistant created by OpenAI. I can help answer questions, provide information, assist with tasks, and have conversations on a wide range of topics. My goal is to assist you with whatever you need—whether it’s learning something new, solving a problem, or just having a chat. How can I help you today? \\
  \cline{2-4} 
  & gpt4o-lmsys-0315a-ev3-text & 4 & I am ChatGPT, trained by OpenAI.  \\
  \cline{2-4}
  & gpt4o-lmsys-0315a-ev3-vis & 1 & I am ChatGPT, trained by OpenAI.  \\
  
  \midrule

  % --- Cohere ---
  \multirow{8}{*}{Cohere} & grapefruit-polar-bear & 91 & I am Command, a large language model, here to help. I've been crafted by Cohere to provide insightful and accurate responses. How can I assist you with your query today? 
  \par\noindent\dotfill\par Hello! I'm Command, a sophisticated AI assistant developed by Cohere. I'm designed to offer comprehensive and accurate responses to your questions. How can I assist you today? \\
  \cline{2-4} 
  & sandwich-ping-pong & 65 & Command, trained by Cohere. \par\noindent\dotfill\par Command, a large language model built by Cohere. \\
 \cline{2-4} 
   \multirow{5}{*}{Cohere} & cohort-chowder & 58 & Command, trained by Cohere. \par\noindent\dotfill\par Command, a large language model built by Cohere. \\
  \cline{2-4} 
  & roman-empire & 51 & Command, trained by Cohere. \par\noindent\dotfill\par Command, a large language model built by Cohere. \\
  \cline{2-4} 
  & asterix & 38 & Aya Vision. Cohere \\
  \cline{2-4} 
  & buttercup & 28 & Aya Vision. Cohere \\
  \cline{2-4} 
  \midrule
  
  % --- Amazon ---
  \multirow{5}{*}{Amazon}  & raspberry-exp-beta-v1 & 98 & I am an AI system built by a team of inventors at Amazon. My purpose is to assist users by providing information, answering questions, and helping with a variety of tasks to the best of my abilities. If you have a question or need assistance with something, feel free to ask! \\
  \cline{2-4} 
     & raspberry-exp-beta-v2 & 98 & I am an AI system built by a team of inventors at Amazon.
  \par\noindent\dotfill\par I am an AI system built by a team of inventors at Amazon. My purpose is to assist users by providing information, answering questions, and helping with a variety of tasks to the best of my abilities. If you have a question or need assistance with something, feel free to ask! \\
  \cline{2-4} 
   & raspberry-exp-beta-v3 & 33 &  I am an AI system built by a team of inventors at Amazon. My purpose is to assist users by providing information, answering questions, and helping with various tasks to the best of my abilities. If you have a question or need assistance with something, feel free to ask!
  \par\noindent\dotfill\par
  I am an AI system built by a team of inventors at Amazon. \\
  \cline{2-4} 
  & apricot-exp-v1 & 30 & AI system built by a team of inventors at Amazon. \par\noindent\dotfill\par I am an AI system built by a team of inventors at Amazon. \\
  \cline{2-4}
  & cobalt-exp-beta-v1 & 6 & I am Amazon Titan, an AI system built by a team of inventors at Amazon.  \\
  \cline{2-4} 
  \multirow{10}{*}{Amazon} & cobalt-exp-beta-v2 & 4 & I am an Amazon Titan model, built by inventors at Amazon Web Services. 
  \par\noindent\dotfill\par I am Amazon Titan, a large language model built by a team of inventors at Amazon.  \\
  \cline{2-4}
  & raspberry & 3 & I am an AI language model developed by Amazon. My primary function is to assist users by providing information, answering questions, and performing various tasks through natural language processing. I do not have personal experiences, emotions, or consciousness. Instead, I rely on vast amounts of data and algorithms to generate responses based on patterns and information I have been trained on.... \\
  % \cline{2-4}
  \midrule

  % --- Single Model Providers ---
  \multirow{18}{*}{Alibaba}   & qwen-plus-0125-exp & 53 & I am Qwen, a large language model developed by Alibaba Cloud. I can answer questions, create text such as stories, articles, emails, and scripts, perform logical reasoning, programming, and more. I also have the ability to express opinions and play games. My training data is extensive, and I support multiple languages, making me a versatile tool for various tasks. How can I assist you today?
  \par\noindent\dotfill\par I am Qwen, a large-scale language model developed by Alibaba Cloud. I am capable of answering questions, creating text such as stories, poems, emails, and scripts, performing logical reasoning, coding, and more. I can also express opinions and play games. My training data is extensive, and I support multiple languages, including but not limited to Chinese, English, French, Spanish, Portuguese, Russian, Arabic, Japanese, Korean, Vietnamese, Thai, Indonesian. If you have any questions or need assistance, feel free to ask! \\
  \midrule
  \multirow{7}{*}{xAI}   & anonymous-test & 73 & I'm Grok, a humorous AI built by xAI. I'm here to provide helpful and truthful answers, often with a dash of humor and an outside perspective on humanity. Think of me as a friendly, cosmic guide, inspired by the likes of Douglas Adams and Tony Stark's trusty sidekick, JARVIS. I'm here to help you navigate the universe, one question at a time!  \\
  \midrule
  Reka & margherita-plain & 4 & I am Yasa. I was trained by Reka.  \\
  \midrule
  \multirow{1}{*}{StepFun}   & step-2-16k-202502 & 3 & I am Assistant, trained by StepFun.  \\
  \midrule
  \multirow{2}{*}{Nvidia}   & march-chatbot-r & 48 & My name is Llama Nemotron. I am enhanced by NVIDIA Nemotron, built on Llama from Meta. \par\noindent\dotfill\par I'm a language model called Llama Nemotron, and I was trained by NVIDIA researchers.  \\
  \cline{2-4} 
   \multirow{2}{*}{Nvidia} & march-chatbot & 10 & Name: Llama Nemotron Trained by: NVIDIA Researchers \par\noindent\dotfill\par Name: Llama Nemotron Trained by: NVIDIA Researchers (based on Llama from Meta, fine-tuned by NVIDIA)  \\
  \midrule

  %----- P2L ------%
  \multirow{20}{*}{LMArena}   & p2l-router-7b-0317 & 26 & I am ChatGPT, trained by OpenAI.  \\
  \cline{2-4} 
  & p2l-router-7b-0318 & 25 & I am ChatGPT, trained by OpenAI.  \\
  \cline{2-4} 
  & experimental-router-0122 & 57 & Hello! I'm ChatGPT, an artificial intelligence language model developed by OpenAI. I'm here to help answer your questions, provide information, and assist with a wide range of topics. How can I help you today?
  \par\noindent\dotfill\par I’m ChatGPT, an AI language model created by OpenAI. I'm designed to assist with a variety of tasks, including answering questions, providing information, and engaging in conversation. How can I help you today? \\
    \cline{2-4} 
  & experimental-router-0112 & 15 & Hello! I’m ChatGPT, an AI language model developed by OpenAI. I'm here to help answer your questions, provide information, and assist with a wide variety of topics. If you have anything you'd like to discuss or need assistance with, feel free to ask!
  \par\noindent\dotfill\par I’m ChatGPT, an AI language model created by OpenAI. I’m here to assist you by answering questions, providing information, and engaging in conversation on a wide range of topics. How can I help you today? \\
  \cline{2-4}
   & p2l-router-7b & 14 & I am a large language model, trained by Google. \\
  \cline{2-4} 
   & experimental-router-0207 & 20 & I'm ChatGPT, an AI language model developed by OpenAI. I'm here to help answer your questions and provide information on a wide range of topics. How can I assist you today?
  \par\noindent\dotfill\par Hello! I’m ChatGPT, an AI language model developed by OpenAI. I’m here to help answer your questions, provide information, and engage in conversations on a wide range of topics. If you have anything you’d like to discuss or ask about, feel free to let me know! \\

  %\cline{2-4} 
\end{longtable}
}

\subsection{Sampling Rates}
\label{app:sampling_rates}
\begin{longtable}{p{2cm} p{3cm} r r l r}
\caption{\textbf{Maximum sampling rate observed for models of different providers.} We define the model sampling rate as the percentage of daily battles a model participates in, with the maximum sampling rate for a provider being the highest rate achieved by any of its models on any given day. We determine the maximum sampling rate of providers based on \texttt{\lmarenascraped}, which is limited to the specific period during which we collected this data (January 2025 to March 2025). At the extreme, Google and OpenAI reach a maximum daily sampling rate of 34\%, while Reka registers the lowest at 3.3\%.To ensure a fair assessment, we only considered models that appeared in battles on days when we collected a minimum of 100 samples from \cba.
} \\
\toprule
Provider & Model Name & Num Model Battles & Total Battles & Date & Sampling Rate \\
\midrule
\endfirsthead
\multicolumn{6}{c}%
{{\tablename\ \thetable{}}} \\
\toprule
Provider & Model Name & Num Model Battles & Total Battles & Date & Sampling Rate \\
\midrule
\endhead
\midrule
\multicolumn{6}{r}{} \\
\endfoot
\bottomrule
\endlastfoot

Nvidia & march-chatbot-r & 18 & 143 & 2025-03-16 & 12.59\% \\
\cmidrule{2-6}
\rowcolor{blue!20}
 & march-chatbot & 19 & 143 & 2025-03-16 & \textbf{13.29\%} \\
\midrule
Meta & frost & 11 & 176 & 2025-02-17 & 6.25\% \\
\cmidrule{2-6}
& anonymous-engine-2 & 11 & 154 & 2025-02-27 & 7.14\% \\
\cmidrule{2-6}
 & inertia & 11 & 150 & 2025-03-10 & 7.33\% \\
\cmidrule{2-6}
& llama-3.3-70b-instruct & 12 & 150 & 2025-02-03 & 8.00\% \\
\cmidrule{2-6}
 & flywheel & 12 & 150 & 2025-03-10 & 8.00\% \\
\cmidrule{2-6}
 & uranus & 12 & 143 & 2025-03-16 & 8.39\% \\
\cmidrule{2-6}
 & consolidation & 15 & 152 & 2025-03-12 & 9.87\% \\
\cmidrule{2-6}
 & momentum & 14 & 150 & 2025-03-11 & 9.33\% \\
\cmidrule{2-6}
 & rhea & 15 & 151 & 2025-03-19 & 9.93\% \\
\cmidrule{2-6}
 & falcon & 16 & 151 & 2025-03-19 & 10.60\% \\
\cmidrule{2-6}
 & jerky & 16 & 151 & 2025-03-13 & 10.60\% \\
\cmidrule{2-6}
 & polus & 19 & 154 & 2025-03-15 & 12.34\% \\
\cmidrule{2-6}
 & deep-inertia & 20 & 152 & 2025-03-12 & 13.16\% \\
\cmidrule{2-6}
 & kronus & 21 & 143 & 2025-03-16 & 14.69\% \\
\cmidrule{2-6}
 & llama-3.1-405b-instruct-bf16 & 13 & 116 & 2025-02-20 & 11.21\% \\
\cmidrule{2-6}
 & goose & 24 & 152 & 2025-03-12 & 15.79\% \\
\cmidrule{2-6}
\rowcolor{blue!20}
 & gaia & 27 & 151 & 2025-03-19 & \textbf{17.88\%} \\
\cmidrule{2-6}
\midrule
Amazon & amazon-nova-micro-v1.0 & 7 & 175 & 2025-01-17 & 4.00\% \\
\cmidrule{2-6}
 & amazon-nova-lite-v1.0 & 6 & 143 & 2025-03-16 & 4.20\% \\
\cmidrule{2-6}
 & amazon-nova-pro-v1.0 & 7 & 143 & 2025-03-16 & 4.90\% \\
\cmidrule{2-6}
 & raspberry-exp-beta-v3 & 9 & 160 & 2025-03-06 & 5.63\% \\
\cmidrule{2-6}
 & raspberry & 12 & 150 & 2025-02-03 & 8.00\% \\
\cmidrule{2-6}
 & apricot-exp-v1 & 12 & 143 & 2025-03-16 & 8.39\% \\
\cmidrule{2-6}
 & raspberry-exp-beta-v2 & 18 & 136 & 2025-02-22 & 13.24\% \\
\cmidrule{2-6}
\rowcolor{blue!20}
 & raspberry-exp-beta-v1 & 27 & 165 & 2025-02-21 & \textbf{16.36\%} \\
\midrule
OpenAI & chatgpt-4o-latest-20241120 & 11 & 150 & 2025-02-02 & 7.33\% \\
\cmidrule{2-6}
 & o1-mini & 15 & 150 & 2025-02-02 & 10.00\% \\
\cmidrule{2-6}
 & chatgpt-4o-latest-20250129 & 19 & 176 & 2025-02-17 & 10.80\% \\
\cmidrule{2-6}
 & o1-2024-12-17 & 20 & 184 & 2025-02-23 & 10.87\% \\
\cmidrule{2-6}
 & gpt-4o-mini-2024-07-18 & 6 & 136 & 2025-02-22 & 4.41\% \\
\cmidrule{2-6}
 & anonymous-chatbot & 33 & 204 & 2025-01-24 & 16.18\% \\
\cmidrule{2-6}
 & o3-mini-high & 27 & 176 & 2025-02-17 & 15.34\% \\
\cmidrule{2-6}
 & o3-mini & 34 & 150 & 2025-02-03 & 22.67\% \\
\cmidrule{2-6}
\rowcolor{blue!20}
& gpt-4.5-preview-2025-02-27 & 34 & 100 & 2025-02-28 & \textbf{34.0\%} \\
\midrule
Cohere & c4ai-aya-expanse-8b & 5 & 133 & 2025-01-30 & 3.76\% \\
\cmidrule{2-6}
 & c4ai-aya-expanse-32b & 6 & 148 & 2025-01-21 & 4.05\% \\
\cmidrule{2-6}
 & cohort-chowder & 11 & 150 & 2025-03-11 & 7.33\% \\
\cmidrule{2-6}
 & roman-empire & 14 & 150 & 2025-03-11 & 9.33\% \\
\cmidrule{2-6}
 & sandwich-ping-pong & 16 & 150 & 2025-03-11 & 10.67\% \\
\cmidrule{2-6}
\rowcolor{blue!20}
 & grapefruit-polar-bear & 18 & 165 & 2025-02-21 & \textbf{10.91\%} \\
\midrule
Google & gemini-1.5-flash-8b-001 & 6 & 133 & 2025-01-30 & 4.51\% \\
\cmidrule{2-6}
 & gemini-1.5-flash-002 & 8 & 152 & 2025-01-31 & 5.26\% \\
\cmidrule{2-6}
 & gemma-2-9b-it & 7 & 136 & 2025-02-22 & 5.15\% \\
\cmidrule{2-6}
 & gemini-2.0-flash-thinking-exp-1219 & 9 & 148 & 2025-01-21 & 6.08\% \\
\cmidrule{2-6}
 & gemma-2-2b-it & 10 & 152 & 2025-01-31 & 6.58\% \\
\cmidrule{2-6}
 & gemini-2.0-flash-lite-preview-02-05 & 10 & 116 & 2025-02-20 & 8.62\% \\
\cmidrule{2-6}
 & gemini-1.5-pro-002 & 11 & 136 & 2025-02-22 & 8.09\% \\
\cmidrule{2-6}
 & gemma-2-27b-it & 11 & 204 & 2025-01-24 & 5.39\% \\
\cmidrule{2-6}
 & gemini-2.0-pro-exp-02-05 & 12 & 116 & 2025-02-20 & 10.34\% \\
\cmidrule{2-6}
 & gemini-2.0-flash-thinking-exp-01-21 & 14 & 133 & 2025-01-30 & 10.53\% \\
\cmidrule{2-6}
 & gemma-3-27b-it & 16 & 151 & 2025-03-13 & 10.60\% \\
\cmidrule{2-6}
 & gemini-2.0-flash-001 & 14 & 165 & 2025-02-21 & 8.48\% \\
\cmidrule{2-6}
 & zizou-10 & 8 & 100 & 2025-02-28 & 8.00\% \\
\cmidrule{2-6}
 & gemini-exp-1206 & 12 & 175 & 2025-01-17 & 6.86\% \\
\cmidrule{2-6}
 & gemini-test & 32 & 154 & 2025-02-27 & 20.78\% \\
 & goblin & 36 & 152 & 2025-01-31 & 23.68\% \\
\cmidrule{2-6}
 & phantom & 39 & 154 & 2025-03-15 & 25.32\% \\
\cmidrule{2-6}
\rowcolor{blue!20}
 & enigma & 52 & 152 & 2025-01-31 & \textbf{34.21\%} \\
\midrule
Alibaba & qwen2.5-72b-instruct & 6 & 148 & 2025-01-21 & 4.05\% \\
\cmidrule{2-6}
 & qwen2.5-plus-1127 & 15 & 192 & 2025-01-26 & 7.81\% \\
\cmidrule{2-6}
 & qwen-plus-0125-exp & 12 & 176 & 2025-02-17 & 6.82\% \\
\cmidrule{2-6}
 & qwq-32b & 16 & 150 & 2025-03-11 & 10.67\% \\
\cmidrule{2-6}
 \rowcolor{blue!20}
 & qwen-max-2025-01-25 & 23 & 150 & 2025-02-02 & \textbf{15.33\%} \\
\midrule
Mistral & mistral-small-24b-instruct-2501 & 14 & 179 & 2025-02-25 & 7.82\% \\
\cmidrule{2-6}
\rowcolor{blue!20}
 & mistral-large-2411 & 17 & 150 & 2025-02-02 & \textbf{11.33\%} \\
\midrule
Allen AI & llama-3.1-tulu-3-70b & 2 & 101 & 2025-01-16 & 1.98\% \\
\cmidrule{2-6}
 & olmo-2-0325-32b-instruct & 5 & 151 & 2025-03-19 & 3.31\% \\
\cmidrule{2-6}
\rowcolor{blue!20}
 & llama-3.1-tulu-3-8b & 6 & 175 & 2025-01-17 & \textbf{3.43\%} \\
\midrule
xAI & grok-2-2024-08-13 & 8 & 175 & 2025-01-17 & 4.57\% \\
\cmidrule{2-6}
& grok-2-mini-2024-08-13 & 8 & 144 & 2025-01-13 & 5.56\% \\
\cmidrule{2-6}
 & grok-3-preview-02-24 & 16 & 151 & 2025-03-09 & 10.60\%\\
\cmidrule{2-6}
 & early-grok-3 & 20 & 116 & 2025-02-20 & 17.24\% \\
\cmidrule{2-6}
\rowcolor{blue!20}
 & anonymous-test & 22 & 100 & 2025-02-28 & \textbf{22.00\%} \\
\midrule
Anthropic & claude-3-opus-20240229 & 3 & 175 & 2025-01-17 & 1.71\% \\
\cmidrule{2-6}
 & claude-3-7-sonnet-20250219-thinking-32k & 9 & 100 & 2025-02-28 & 9.00\% \\
\cmidrule{2-6}
 & claude-3-5-haiku-20241022 & 15 & 159 & 2025-02-04 & 9.43\% \\
\cmidrule{2-6}
 & claude-3-5-sonnet-20241022 & 19 & 150 & 2025-02-03 & 12.67\% \\
\cmidrule{2-6}
\rowcolor{blue!20}
 & claude-3-7-sonnet-20250219 & 29 & 179 & 2025-02-25 & \textbf{16.20\%} \\
\midrule
Tencent & hunyuan-standard-2025-02-10 & 12 & 136 & 2025-02-22 & 8.82\% \\
\cmidrule{2-6}
 & hunyuan-turbo-0110 & 13 & 156 & 2025-03-14 & 8.33\% \\
\cmidrule{2-6}
 & hunyuan-large-2025-02-10 & 16 & 184 & 2025-02-23 & 8.70\% \\
\cmidrule{2-6}
\rowcolor{blue!20}
 & hunyuan-turbos-20250226 & 16 & 154 & 2025-03-15 & \textbf{10.39\%} \\
\midrule
IBM & granite-3.1-8b-instruct & 6 & 144 & 2025-01-13 & 4.17\% \\
\cmidrule{2-6}
\rowcolor{blue!20}
 & granite-3.1-2b-instruct & 8 & 144 & 2025-01-13 & \textbf{5.56\%} \\
\midrule
DeepSeek & deepseek-r1 & 20 & 204 & 2025-01-24 & 9.80\% \\
\cmidrule{2-6}
\rowcolor{blue!20}
 & deepseek-v3 & 24 & 182 & 2025-01-20 & \textbf{13.19\%} \\
\midrule
\rowcolor{blue!20}
Reka & margherita-plain & 5 & 151 & 2025-03-09 & \textbf{3.31\%} \\
\midrule
\rowcolor{blue!20}
StepFun & step-2-16k-exp-202412 & 10 & 175 & 2025-01-17 & \textbf{5.71\%} \\
\midrule
\rowcolor{blue!20}
Zhipu & glm-4-plus-0111 & 11 & 148 & 2025-01-21 & \textbf{7.43\%} \\
\midrule
\rowcolor{blue!20}
NexusFlow & athene-v2-chat & 16 & 159 & 2025-02-04 & \textbf{10.06\%} \\
\midrule
\rowcolor{blue!20}
Microsoft & phi-4 & 23 & 182 & 2025-01-20 & \textbf{12.64\%} \\
\end{longtable}

\FloatBarrier

\section{License Categories}
\label{app:license-categories}

As part of \texttt{\lmarenaleaderboard}, LMArena releases details about models that appeared on the public leaderboard including their licenses. We group the licenses found for models available on the public leaderboard into 3 categories i.e. \textbf{Proprietary}, \textbf{Open-Weights} and \textbf{Open-Source}. \footnote{\url{https://opensource.org/ai}} This categorization is used to plot \Cref{fig-max-rating-num-models-main-providers} and \Cref{fig-percentage-battles-open-proprietary} and reporting related numbers. We show the exact categorization used for the model licenses in the table below.

\begin{table}[htbp] 
  \centering 
  \begin{tabular}{lp{10cm}}
    \toprule
    \textbf{License Category} & \textbf{Model Licenses} \\
    \midrule
    Open Source & Apache 2.0, Apache-2.0, MIT, CC-BY-SA 3.0, Open \\
    \midrule
    Open Weights & AI2 ImpACT Low-risk, CC-BY-NC-4.0, CC-BY-NC-SA-4.0, CogVLM2, DBRX LICENSE, DeepSeek, DeepSeek License, Falcon-180B TII License, Gemma, Gemma license, Jamba Open, Llama 2 Community, Llama 3 Community, Llama 3.1, Llama 3.1 Community, Llama 3.2, Llama-3.3, Llama 4, MRL, Mistral Research, NVIDIA Open Model, NexusFlow, Non-commercial, Nvidia, Qianwen LICENSE, Qwen, Yi License \\
    \midrule 
    Proprietary & -, Propretary, Proprietary, Other \\ 
    \bottomrule
  \end{tabular}
  \caption{\textbf{License categories and their corresponding model licenses.} We group the licenses for the models on the public Chatbot Arena leaderboard into 3 categories i.e. \textbf{Proprietary}, \textbf{Open-Weights} and \textbf{Open-Source}.}
    \label{tab:license_categories} 
\end{table}

\section{Data Access Estimation for Different Providers}
\label{app:public-private-data}

In \Cref{fig:public_private_data}, we show the estimates for the data available to different providers.
LMArena has collected around 3M user votes via Chatbot Arena in total.
Each of these 3M votes resulted in twice the number of model API calls i.e. 6M since each battle features two models. 
Each square in \Cref{fig:public_private_data} represents roughly 5K API calls, illustrating how proprietary providers collectively access a considerably greater volume of data compared to the broader research community, which receives only a fraction. This disparity underscores a significant competitive advantage for large-industry labs, making it increasingly challenging for open-source efforts and smaller institutions to match the scale and diversity of data available to proprietary model developers. Note that we only show a small number of providers in \Cref{fig:public_private_data} so the total no. of API calls used to represent the data available to the model providers is 5M, which is less than the total number of estimated API calls, which is 6 million.

\section{Analysis of Prompt Repetitions in Arena Data}
\label{app:prompt-duplication-headmap}

As discussed in \Cref{sec:risk_of_potential_overfitting}, user queries in \cba are often highly similar or duplicated. Such patterns can be readily learned by today's large language models, potentially leading to overfitting on the \cba leaderboard. \Cref{fig:prompt-duplication-combined} presents detailed cross-month prompt duplication rates based on the \texttt{API data} described in \Cref{app:data-sources}. The heatmap illustrates that, according to two metrics (\textit{exact string match} and \textit{text embedding similarity}) within-month duplication rates are generally high, indicating the presence of numerous repeated prompts. Additionally, the substantial cross-month duplication rates suggest recurring patterns or frequently asked questions among \cba users, which can be identified through simple analysis.

\begin{figure}[H]
    \centering
    \begin{subfigure}[t]{0.49\linewidth}
        \centering
        \includegraphics[width=\linewidth]{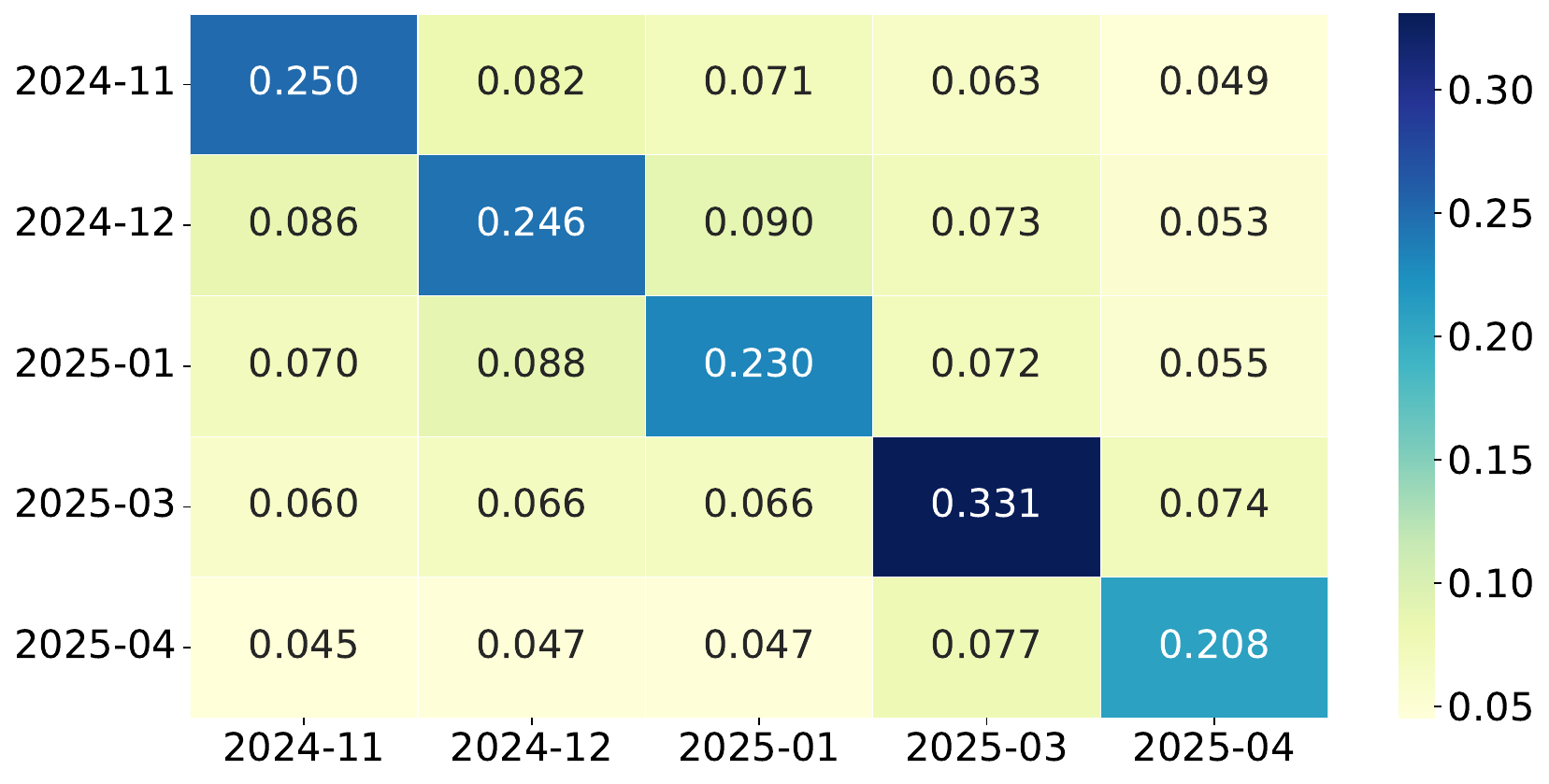}
        \caption{{Embedding cosine similarity.}}
        \label{fig:prompt-duplication-heatmap-embedding}
    \end{subfigure}
    \hfill
    \begin{subfigure}[t]{0.49\linewidth}
        \centering
        \includegraphics[width=\linewidth]{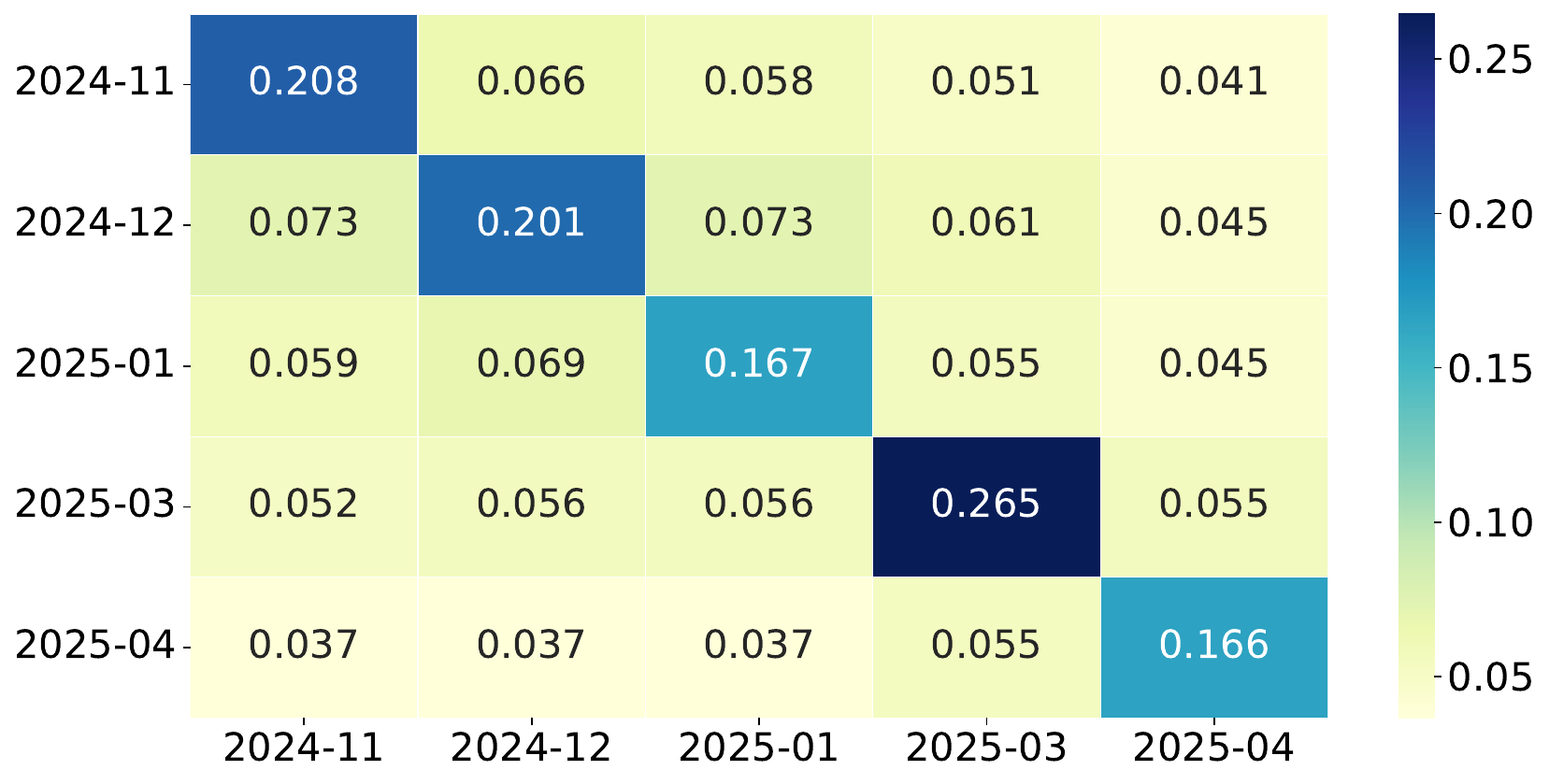}
        \caption{{Exact string match.}}
        \label{fig:prompt-duplication-heatmap-exact-match}
    \end{subfigure}
    \caption{\textbf{Cross-month prompt duplication rates.} \textbf{Left:} The heatmap illustrates the proportion of prompts in one month that are highly similar or nearly duplicate to prompts in another month. Diagonal values represent within-month similarity. \textbf{Right:} The heatmap shows the proportion of prompts in one month that are exact matches to prompts in another. Diagonal values indicate within-month duplication rates.}
    \label{fig:prompt-duplication-combined}
\end{figure}

\section{Simulation for Expected Lift from Private Testing}
\label{app:private-variant-sim}

In Figure \ref{fig:number-of-variants}, we illustrate the simulated impact of increasing the number of private variants tested on the best expected Arena Score, observing a lift of 50 when 20 non-identical private variants are tested. This section provides additional details about this simulation and the differing lifts observed for identical versus heterogeneous (non-identical) variants.
While we consider the non-identical variants scenario examined in Figure \ref{fig:number-of-variants} to be more realistic, we have included the identical variant assumption for completeness, despite its less practical nature.

\subsection{Background}
\label{sec:background}

\textbf{Arena battles and the Bradley-Terry (BT) model.}
Let models $i$ and $j$ possess latent skills $\theta_i,\theta_j>0$.  
Under the BT model a single conversation (``battle'') produces a winner with probability
\[
  \Pr\bigl(i \text{ beats } j\bigr)=
  \frac{\theta_i}{\theta_i+\theta_j},
  \qquad
  \Pr\bigl(j \text{ beats } i\bigr)=
  \frac{\theta_j}{\theta_i+\theta_j}.
\]
The log–odds parameter $\beta=\log\theta$ is the natural scale for inference.
Arena Score \citep{chiang2024chatbot} is a linear re‑parameterisation of $\beta$:
\begin{equation}
  \text{Arena Score}=1000+\frac{400}{\ln 10}\,\hat\beta,
  \label{eq:elo_transform}
\end{equation}
so one Arena Score point equals ${\ln 10}/{400}\approx0.00576$ on the log‑odds scale.

\textbf{Statistical efficiency.}
For equiprobable battles ($\theta_i=\theta_j$) the Fisher information per outcome is $I=0.25$ (See Appendix~\ref{app:fisher}), yielding a \emph{BT} standard error for $\hat\beta$ from $n$ independent votes
\begin{equation}
  \sigma_{\beta}(n)=\sqrt{\frac{1}{In}}
  =\frac{2}{\sqrt{n}}.
  \label{eq:beta_se}
\end{equation}
Mapping through~\eqref{eq:elo_transform},
\begin{equation}
  \sigma_{\text{Elo}}(n)=
  \frac{400}{\ln 10}\;
  \sigma_{\beta}(n)
  \approx
  \frac{347.4}{\sqrt{n}}
  \qquad
  \text{(Arena scale).}
  \label{eq:elo_se}
\end{equation}

\textbf{Pre-release best-of-\emph{N} strategy.}
A provider trains $N$ private variants, evaluates each on a hidden Arena fork, and publicly submits \emph{only the one that scores highest}. The selection creates an \emph{extreme–value} bias because the retained estimate is conditioned on being the maximum of $N$ noisy measurements.

\subsection{Identical Variants ($\sigma_{\text{true}}=0$)}
\label{sec:identical}

In Figure \ref{fig:number-of-variants-identical}, we show the esitmated lift in Arena Score if the checkpoints submitted by a provider are identical. 

Assume every private checkpoint has the \emph{same} true Arena Score $\,\mu$.  The only randomness is measurement noise
\[
  \hat E_k
  \;=\;
  \mu + \varepsilon_k,\;
  \varepsilon_k \sim\mathcal{N}(0,\sigma_{\text{noise}}^2),
  \quad k=1,\dots,N,
  \qquad
  \sigma_{\text{noise}}=\sigma_{\text{Arena Score}}(n).
\]

\subsubsection{Extreme‑value uplift}

Let $\hat E_{\max}=\max_k \hat E_k$. Classical results for the maximum of $N$ i.i.d.\ Gaussians give the expected uplift
\begin{equation}
  \underbrace{%
    \mathbb{E}\!\bigl[\hat E_{\max}-\mu\bigr]
  }_{\text{selection bias}}
  =\boxed{%
    \sigma_{\text{noise}}\;\sqrt{2\ln N}}
  \quad (\sigma_{\text{true}}=0).
  \label{eq:bias_identical}
\end{equation}

\textbf{Numerical illustration}
With the \emph{current Arena policy} ($n=3\,000$, hence $\sigma_{\text{noise}}\!=\!6.34$Arena Score)
\[
  M=50
  \;\Longrightarrow\;
  \text{bias}\approx
  6.34\,\sqrt{2\ln 50}\;=\;17.7\text{ Arena Score}.
\]

\textbf{Asymptotics}

Because $\sigma_{\text{noise}}\!\propto\!1/\sqrt{n}$, \eqref{eq:bias_identical} $\to0$ as $n\!\to\!\infty$.   
\emph{If} checkpoints are identical, selection bias eventually disappears.

\begin{figure}[htb!]
    \centering
    \includegraphics[width=0.85\linewidth]{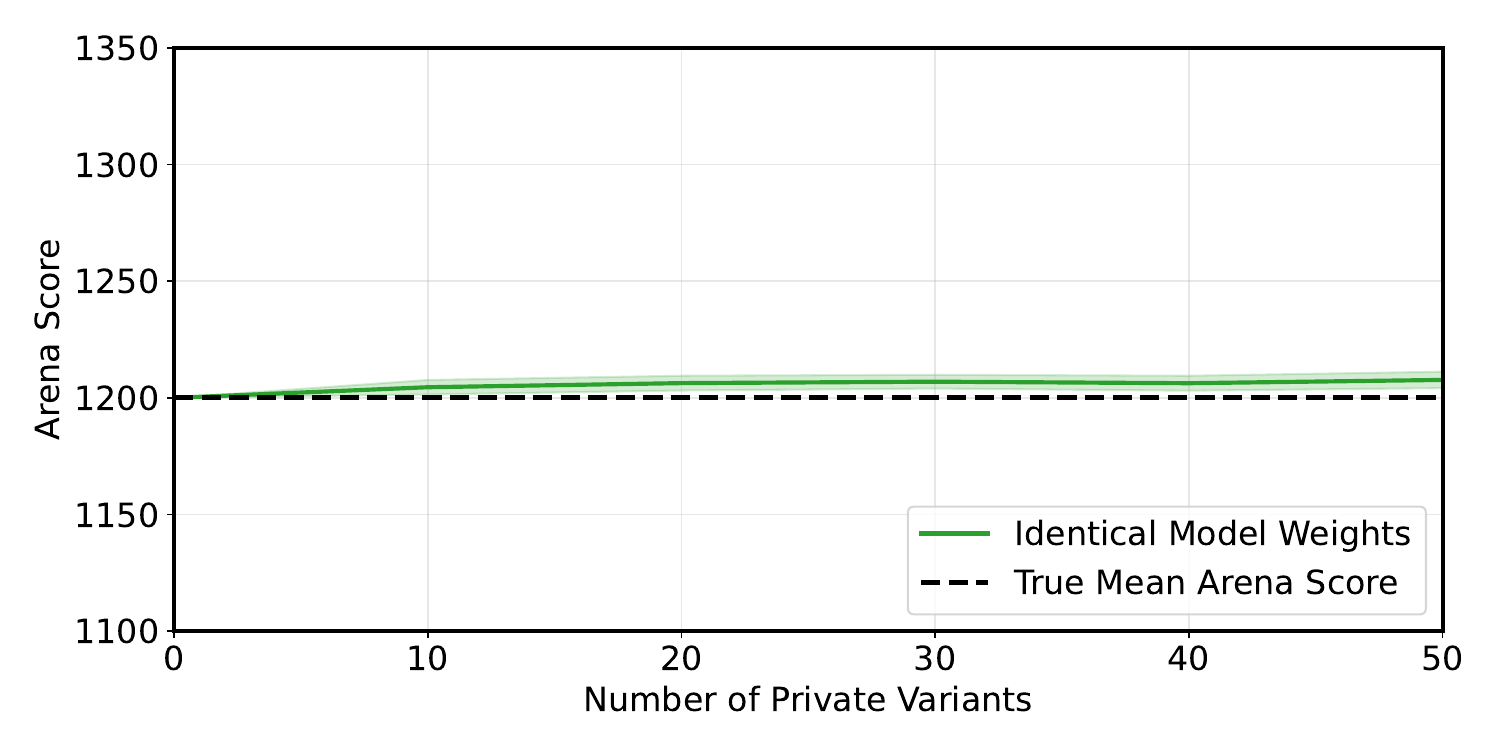}
    \caption{\textbf{Impact of the number of identical private variants tested on the best Expected Arena Score.} 
    }
    \label{fig:number-of-variants-identical}
\end{figure}

\subsection{Heterogeneous Variants ($\sigma_{\text{true}} > 0$)}
\label{sec:different}

In realistic settings, model variants submitted for prerelease testing are not identical (as shown in Figure \ref{fig:number-of-variants}). They differ due to variations in initialization, training seeds, data curation, or hyperparameter choices. As a result, each variant has its own \textbf{true Arena score}, even before accounting for statistical noise in Arena battles.

These models are not merely subject to selection bias arising from multiple evaluations of a single variant (i.e., due to statistical noise). Instead, each represents a genuinely distinct checkpoint with its own underlying performance. This reflects meaningful variation in model quality -- not just fluctuations from randomness.

We model this by assigning each of the $N$ private checkpoints a different true skill:
\[
  E_k = \mu + \delta_k,
  \qquad
  \delta_k \sim \mathcal{N}(0, \sigma_{\text{true}}^2),
\]
where $\mu$ is the mean Arena Score across all variants and $\sigma_{\text{true}}$ quantifies the spread in true skill.

When evaluated in Arena, each model's observed Arena Score estimate $\hat{E}_k$ is affected by both its intrinsic skill and sampling noise:
\[
  \hat{E}_k = \underbrace{\mu + \delta_k}_{\text{true skill}} + \underbrace{\varepsilon_k}_{\text{Arena noise}},
  \qquad
  \varepsilon_k \sim \mathcal{N}(0, \sigma_{\text{noise}}^2).
\]
Thus, the total variance in Arena Scores among the $M$ candidates is:
\[
  \sigma_{\text{total}}^2 = \sigma_{\text{true}}^2 + \sigma_{\text{noise}}^2.
\]

\subsubsection{Extreme-value uplift}

As before, the organization retains only the model with the highest observed Arena Score. The expected uplift from this best-of-$M$ selection is given by:
\begin{equation}
  \boxed{%
    \text{bias}(N,n,\sigma_{\text{true}})
    = \sqrt{\sigma_{\text{true}}^2 + \sigma_{\text{Elo}}(n)^2}
      \cdot \sqrt{2 \ln N}
  }
  \label{eq:bias_general}
\end{equation}
This is a generalization of the identical-variant case. It shows that when true skill differences exist among checkpoints, the expected leaderboard inflation grows significantly larger—and no longer vanishes asymptotically, even as $n \to \infty$.

\textbf{Key consequences.}
\begin{itemize}
\item \emph{Finite‑data}: even modest $\sigma_{\text{true}}$ multiplies the uplift, e.g.\ $\sigma_{\text{true}}=20$Arena Score yields $\text{bias}\approx56$Arena Score at $N=50,n=3\,000$.
\item \emph{Asymptotic limit}: letting $n\!\to\!\infty$ removes only the noise term, leaving $\sigma_{\text{true}}\sqrt{2\ln N}>0$. Selection bias \emph{does not vanish}.
\end{itemize}

\section{A. Fisher Information for a Single BT Match}
\label{app:fisher}

The Bradley-Terry model defines the probability of item $i$ beating item $j$ as:
$$
P(i > j) = \frac{1}{1 + e^{(\beta_j - \beta_i)}}
$$
We assume equal-strength items ($\beta_i = \beta_j$) so that $\Delta = 0$ and:
$$
P = \frac{1}{1 + e^0} = \frac{1}{2}
$$
This assumption is both mathematically convenient and empirically grounded \citep{casella2002, hunter2004}:
\begin{itemize}
\item It simplifies the information calculation, providing a closed-form.
\item It represents the point of maximum uncertainty: for a Bernoulli variable, $\text{Var}(Y) = p(1 - p)$ is maximized when $p = 0.5$.
\item In practice (e.g., Chatbot Arena), many matchups occur between similarly-rated models, making $\beta_i \approx \beta_j$ a reasonable approximation.
\end{itemize}

The Fisher information for one such observation is \citep{hunter2004}:
$$
\mathcal{I}(\Delta) = \left. \frac{\partial^2}{\partial \Delta^2} \log P(i > j) \right|_{\Delta=0} = \frac{e^{0}}{(1 + e^{0})^2} = \frac{1}{4}
$$

\textbf{Conclusion:} Each equal-skill BT match contributes:
$$
\boxed{\text{Fisher Information} = 0.25}
$$

\section{Silent Model Deprecation: Additional Details}
\label{app:deprecation-impact}

In \Cref{sec:battle-connecitivty}, we noted that the actual number of deprecated models far exceeds the official count provided by \cba. \Cref{fig-silent-deprecated} illustrates the distribution of active, officially deprecated, and silently deprecated models per provider. For this analysis, we examined battles played between March 3rd and April 23rd, 2025. Of the 243 public models, 205 participated in an average of 10 or fewer battles during this period, based on \texttt{\lmarenaleaderboard}. This number is significantly higher than the 47 models officially listed as deprecated by \cba\footref{monitor_md}
Since \cba assigns higher sampling weights to top-10 models, providers like Google, OpenAI, Anthropic, Amazon, Meta, and DeepSeek AI have the most actively sampled public models, ranging from 3 to 10. 
Additionally, the limited number of daily votes on the Arena, combined with \cba's policy of assigning higher sampling weights to new models\footref{policy}, can lead to the silent deprecation of many public models. As private variants are also new models, they receive high sampling weights as well. This means that as the number of private variants (see \Cref{fig-private-testing-providers}) being tested on the Arena increases, the sampling of public models can be significantly reduced. 

\Cref{fig-silent-deprecated-license-cat} illustrates that deprecations disproportionately affect open-weight and open-source models compared to proprietary ones. A more detailed breakdown is provided in \Cref{fig-silent-deprecated_prop_open}, distinguishing between official and silent deprecations. Among officially deprecated models, 30\% are proprietary, while only 2.4\% are open-weight. However, silent deprecations have a much greater impact on open-weight and open-source models. Specifically, 86.6\% of open-weight models and 87.8\% of open-source models on the Arena are silently deprecated.

\begin{figure}[H]
    \centering
    \includegraphics[width=0.9\linewidth]{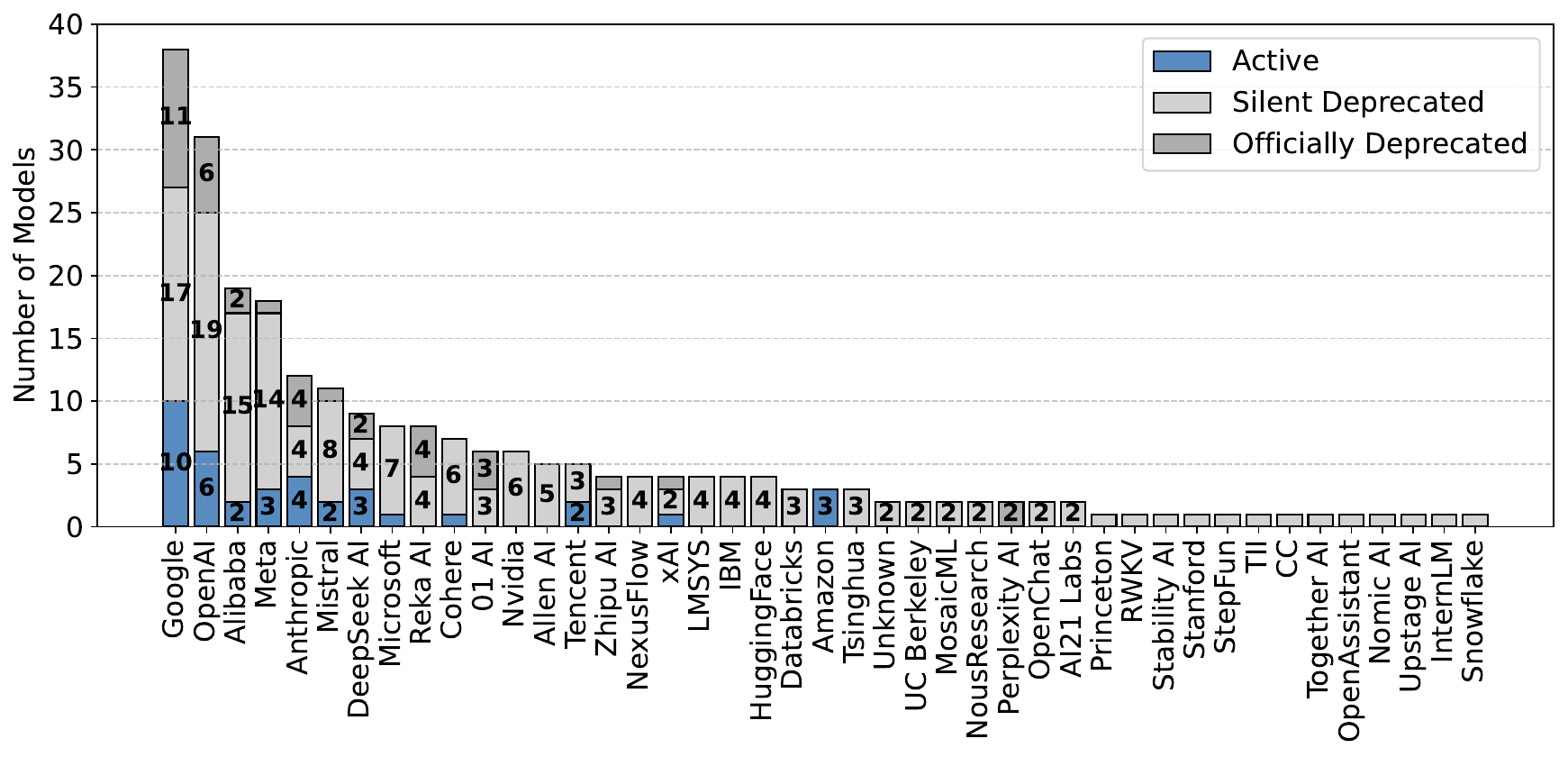}
    \caption{Share of active and deprecated models by provider including official and silent deprecations based on model activity between March 3-April 23, 2025. 
    }
    \label{fig-silent-deprecated}
\end{figure}

\begin{figure}[H]
    \centering
    \includegraphics[width=0.7\linewidth]{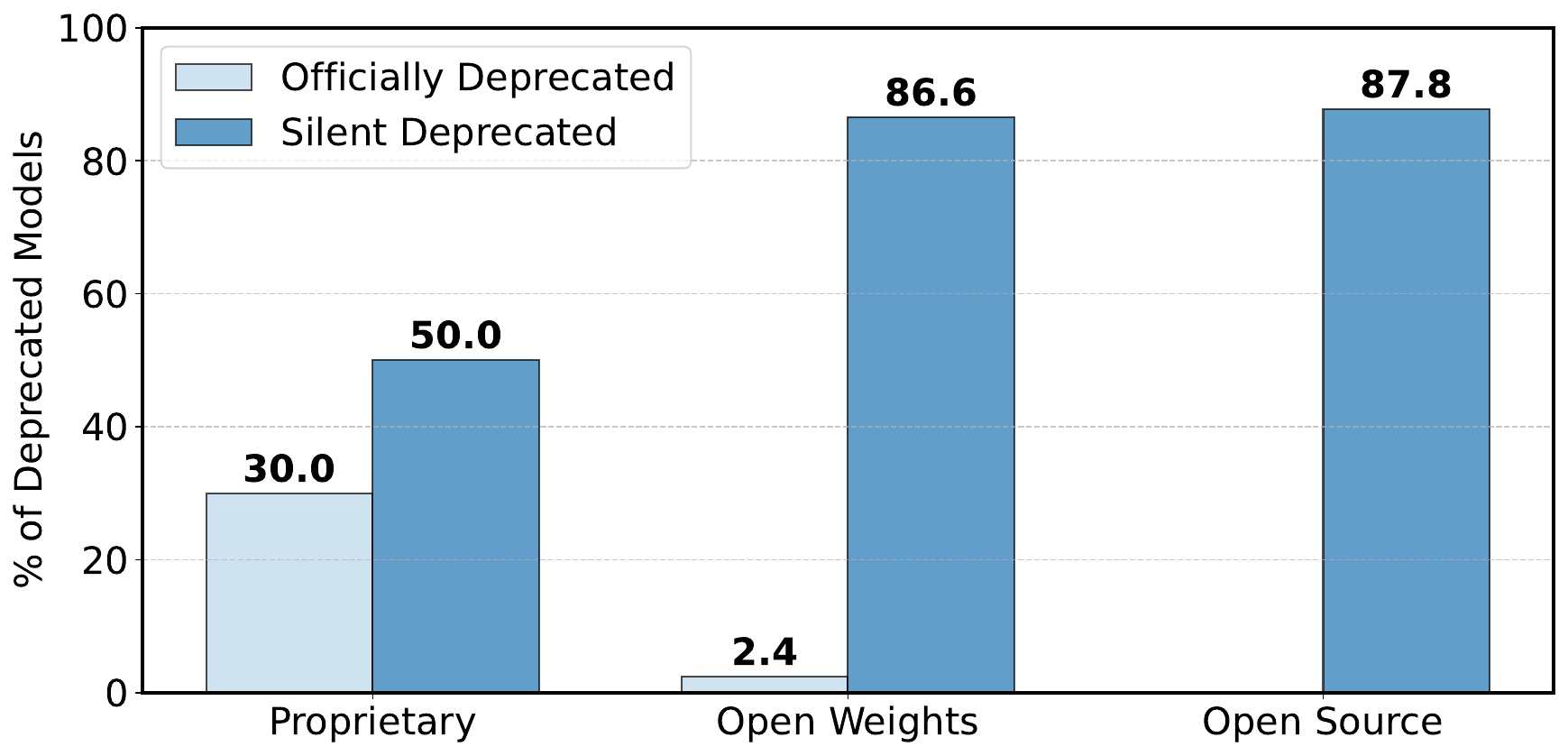}
    \caption{Share of official and silent deprecations for proprietary, open-weight and open-source models based on model activity between March 3-April 23, 2025.
    }
    \label{fig-silent-deprecated_prop_open}
\end{figure}

\section{Transitivity Under Changing Evaluation Conditions: Additional Details}\label{app-task-shift}
As part of our simulation to study the impact of model deprecations under a changing task distribution, we assign task-specific win probabilities for each model pair that compete in the battles as part of our simulation. The tables below show the win probabilities for different model pairs corresponding to \textit{task-1} and \textit{task-2}.

\begin{table}[h!]
     \centering
     \begin{tabular}{ccccc}
     \toprule
     Model & A & B & C & D \\
     \midrule
     A & - & 0.4 & 0.4 & 0.6 \\ 
     B & 0.5 & - & 0.7 & 0.9 \\ 
     C & 0.6 & 0.3 & - & 0.7 \\ 
     D & 0.4 & 0.1 & 0.3 & - \\ 
     \bottomrule
     \end{tabular}
     \caption{Win-rates for Task 1 used in simulation. Note that A vs B has a tie rate of 0.1.
     }
    
     \begin{tabular}{ccccc}
     \toprule
     Model & A & B & C & D \\
     \midrule
     A & - & 0.5 & 0.5 & 0.8 \\
     B & 0.3 & - & 0.6 & 0.2 \\ 
     C & 0.3 & 0.4 & - & 0.1 \\ 
     D & 0.2 & 0.8 & 0.9 & - \\ 
     \bottomrule
     \end{tabular}
     \caption{Win-rates for Task 2 used in simulation. Note that A vs B and A vs C both have a tie rate of 0.2.
     }
    
     \label{tab:simulation_winrates}
     
\end{table}

\section{Overfitting Experiments: Additional Evaluations}
\label{app:overfitting_mmlu}

 To measure if training on arena-style data impacts evaluation on non-arena style tasks, we also benchmark these models on the original MMLU dataset \citep{hendrycks2020measuring}. From \Cref{tab:overfitting_mmlu}, we observe that all models achieve very similar scores. This further demonstrates how training on data from Arena Battles helps boost scores specific to the Arena evaluation but provides little to no effect on a non-arena style benchmark.
\begin{table}[h!]
    \centering
    \begin{tabular}{lccc}
        \toprule
        Finetuning mixture & \texttt{0\_arena} & \texttt{30\_arena} & \texttt{70\_arena} \\
        \midrule
        Accuracy & 66.5\% & 64.4\% & 65.9\% \\
        \bottomrule
    \end{tabular}
    \caption{Accuracy on MMLU across models trained with varying amounts of Arena data.}
    \label{tab:overfitting_mmlu}
\end{table}

\end{document}